\documentclass[letter, 10pt]{IEEEtran}
\usepackage{graphicx}
\usepackage{multicol}
\usepackage{booktabs}
\usepackage{amsmath,amssymb}
\usepackage{amsthm}
\usepackage[utf8]{inputenc}
\usepackage[english]{babel}
\usepackage{multirow}
\usepackage[font=small]{caption}
\usepackage{float}

\usepackage{lettrine}
\usepackage[]{algorithm2e}
\usepackage{algpseudocode}
\usepackage{cite}
\usepackage{filecontents}
\usepackage{lipsum}
\usepackage{color}
\usepackage{esdiff}
\usepackage{epstopdf}
\usepackage[normalem]{ulem}
\usepackage{soul}

\newcommand{\tabincell}[2]{\begin{tabular}{@{}#1@{}}#2\end{tabular}}
\usepackage{subcaption}
\usepackage{xcolor}
\usepackage{colortbl}
\usepackage{balance}
\newtheorem{thm}{Theorem}[section]

\newtheorem{lem}{Lemma}[section]

\newtheorem{defn}{Definition}[section]

\newtheorem{rem}{Remark}[section]

\newcommand{\blue}{\color{blue}}

\definecolor{pink}{rgb}{1, 0, 1}
\definecolor{orange}{rgb}{1, 0.7529, 0}
\definecolor{darkgreen}{rgb}{0, 0.8, 0}
\usepackage[hidelinks]{hyperref}

\begin{document}

\title{
{\small{{\blue This paper is published in IEEE Transactions on Robotics}}}\\ \vspace{0pt}
\Huge{C$^*$: A Coverage Path Planning Algorithm for Unknown Environments using Rapidly\\ Covering Graphs}
\vspace{6pt}
\thanks{Digital Object Identifier 10.1109/TRO.2026.3661719}
\thanks{© 2025 IEEE.  Personal use of this material is permitted.  Permission from IEEE must be obtained for all other uses, in any current or future media, including reprinting/republishing this material for advertising or promotional purposes, creating new collective works, for resale or redistribution to servers or lists, or reuse of any copyrighted component of this work in other works.}
}

\author{\begin{tabular}{cccccccccc}
 {Zongyuan Shen$^\dag$} & {James P. Wilson$^\dag$} & {Shalabh Gupta$^\dag$$^\star$}
\end{tabular}\vspace{-6pt}
\thanks {$^\dag$Dept. of Electrical and Computer Engineering, University of Connecticut, Storrs, CT 06269, USA.}
\thanks {$^\star$Corresponding author (email: shalabh.gupta@uconn.edu)}\vspace{-3pt}
}

\maketitle 

\begin{abstract}
The paper presents a novel sample-based algorithm, called C$^*$, for real-time coverage path planning (CPP) of unknown environments. C$^*$ is built upon the concept of a Rapidly Covering Graph (RCG),  which is incrementally constructed during robot navigation via progressive sampling of the search space. By using efficient sampling and pruning techniques, the RCG is constructed to be a minimum-sufficient graph, where its nodes and edges form the potential waypoints and segments of the coverage trajectory, respectively. The RCG tracks the coverage progress, generates the coverage trajectory and helps the robot to escape from the dead-end situations. 
To minimize coverage time, C$^*$ produces the desired back-and-forth coverage pattern, while adapting to the TSP-based optimal coverage of local isolated regions, called coverage holes, which are surrounded by obstacles and covered regions.
It is analytically proven that C$^*$ provides complete coverage of unknown environments. The algorithmic simplicity and low computational complexity of C$^*$ make it easy to implement and suitable for real-time on-board applications. The performance of C$^*$ is validated by 1) extensive high-fidelity simulations and 2) laboratory experiments using an autonomous robot. C$^*$ yields near optimal trajectories, and a comparative evaluation with seven existing CPP methods demonstrates significant improvements in performance in terms of coverage time, number of turns, trajectory length, and overlap ratio, while preventing the formation of coverage holes. Finally, C$^*$ is comparatively evaluated on two different CPP applications using 1) energy-constrained robots and 2) multi-robot teams. 
\end{abstract}

\begin{IEEEkeywords}
Coverage path planning, Autonomous robots, Unknown environments, Rapidly covering graph
\end{IEEEkeywords}

\IEEEpeerreviewmaketitle

\thispagestyle{empty}

\vspace{-0pt}
\section{Introduction}
\label{sec:introduction}

\IEEEPARstart{R}{ecent} advancements in sensing, computing and communication technologies have spurred a rapid growth of autonomous service robots that support humans by performing a variety of tasks to improve their quality of life. Specifically, some of these tasks require planning a path to cover all points in the search space, this problem is called CPP~\cite{acar2006sensor,acar2002sensor, song2018, galceran2013survey}. A wide range of application domains exist for CPP including exploration (e.g., mapping of underwater terrains~\cite{palomeras2018,shen2022ct}); industry (e.g., 
structural inspection~\cite{englot2013three, vidal2017,bircher2018receding,song2020online}, spray-painting~\cite{vempati2018,atkar2005} and surface cleaning of vehicles~\cite{hassan2019ppcpp}); agriculture (e.g., robotic weeding~\cite{maini2022online} and arable farming~\cite{jin2011}); household (e.g., floor cleaning~\cite{palacin2004building} and lawn mowing~\cite{weiss2008lawn,Song_sose2015}); hazard removal (e.g., oil spill cleaning~\cite{SGH13}); and defense (e.g., demining~\cite{mukherjee2011symbolic}).

Often, the environments where coverage operations are performed are unknown. Therefore, it is essential to develop sensor-based CPP methods that can replan and adapt the coverage trajectory in real-time based on dynamic discovery of the environment. It is critical that the CPP method provides complete coverage and does not get stuck in a dead-end. Furthermore, it is desired that the CPP method is efficient, i.e., it minimizes the coverage time by reducing a) the overlaps which in turn reduces the trajectory length, and b) the number of turns. Moreover, it is desired that the coverage method does not create intermediate coverage holes during navigation, which are local uncovered regions surrounded by obstacles and covered regions. This is required to prevent unnecessary inconvenience to cover such holes later via long return trajectories. Finally, it is important that the CPP method is simple to implement on real robot platforms and also computationally efficient to generate adaptive coverage trajectories in real-time.

\begin{figure}[t]
    \centering
    \includegraphics[width=1.0\columnwidth]{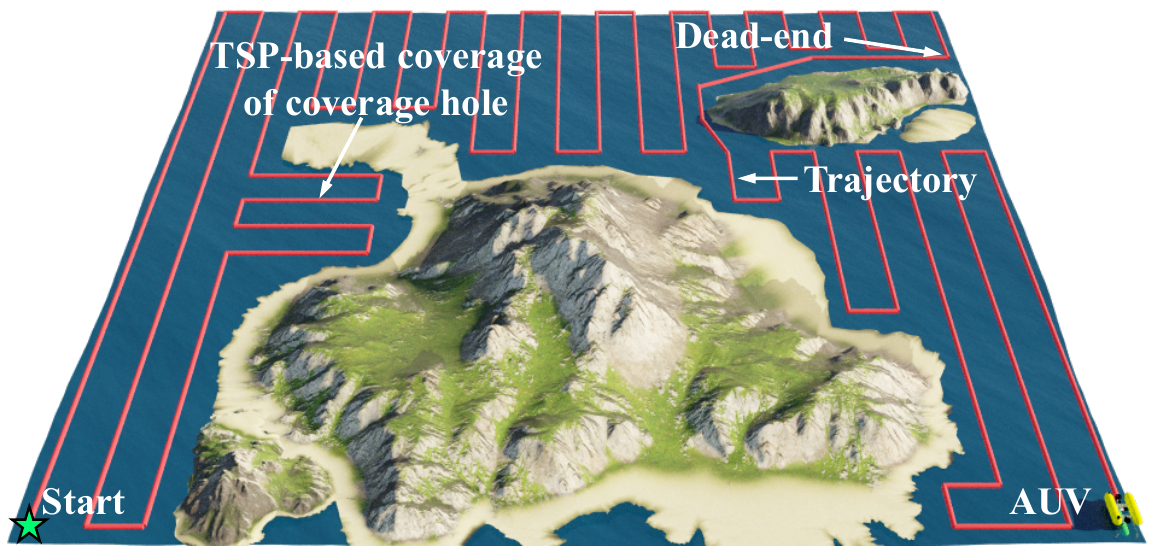}
    \caption{C$^*$ coverage of an island scenario by an AUV.} 
   \label{fig:CPP_application} 
   \vspace{-10pt}
\end{figure}

In this regard, the paper presents a novel algorithm, called C$^*$, for real-time CPP of unknown environments. Fig.~\ref{fig:CPP_application} shows an example of C$^*$-coverage of an island scenario in an ocean environment by an autonomous underwater vehicle (AUV). While the existing CPP methods rely on different techniques such as cellular decomposition~\cite{acar2002sensor}, area tiling~\cite{song2018}, spanning trees~\cite{gabriely2001spanning}, and neural networks~\cite{luo2008bioinspired}, C$^*$ is a sample-based algorithm built upon a novel concept of RCGs. An RCG is a minimum sufficient graphical representation of the obstacle-free space, which is incrementally constructed via progressive sampling during robot navigation. The nodes and edges of an RCG form the potential waypoints and segments of the robot's coverage trajectory, respectively. The objectives of RCG are to track the coverage progress, select non-myopic waypoints for the coverage trajectory in uncovered areas, and guide the robot to escape from dead-end situations. 

C$^*$ provides a complete coverage guarantee of unknown environments using the desired back-and-forth coverage pattern. During navigation, if the robot detects the formation of any potential coverage holes, then C$^*$ synergistically adapts to the locally optimal trajectories generated \textit{in situ} using the traveling salesman problem (TSP) to cover such regions. This prevents the formation of undesired coverage holes in the map during the coverage process, thus eliminating the need to cover them later by return trajectories from distant regions. C$^*$ is validated by high-fidelity simulations and real experiments, which demonstrate that C$^*$ is highly efficient and produces near optimal trajectories. A comparative evaluation with the existing CPP methods shows that C$^*$ yields significant performance improvements in terms of coverage time, number of turns, trajectory length and overlap ratio. Furthermore, C$^*$ is simple to implement and computationally efficient to run on real robotic platforms for real-time operations.

\vspace{-6pt}
\subsection{Related Work}
\label{sec:related_work}
A review of existing CPP methods is presented in~\cite{galceran2013survey}. CPP methods can be classified into various categories.

\vspace{6pt}
\subsubsection{\textbf{Random vs Systematic Methods}} In general, the CPP methods are classified as random or systematic. While the random methods are simple to implement, they can generate strongly overlapping trajectories and do not guarantee complete coverage in finite time. In contrast, the systematic methods use some basic underlying rules to generate well-defined coverage patterns (e.g., back \& forth and spiral) and provide complete coverage in finite time. 

\vspace{6pt}
\subsubsection{\textbf{2D vs 3D Methods}} The CPP methods are also classified based on the dimension of coverage space. While the 2D methods plan coverage paths to cover the 2D surfaces (e.g., floor cleaning), the 3D methods work in the 3D spaces (e.g., mapping of high-rise buildings by UAVs and underwater terrains by UUVs). Hert et. al.~\cite{hert1996} presented a CPP method for underwater terrain mapping using an AUV by sweeping across the terrain while moving at a fixed distance to the uneven surface. Shen and Gupta ~\cite{shen2022ct} 
presented an algorithm, called CT-CPP, that performs layered sensing of the environment using coverage trees (CT) for terrain mapping. 

\vspace{6pt}
\subsubsection{\textbf{Offline vs Online Methods}} The CPP methods are also classified as offline or online depending on the availability of the environmental information. The offline methods rely on prior information known about the environment to design the coverage paths before the robotic system is deployed for operation. Although these algorithms perform well in known scenarios, their performance can degrade if the prior information is incomplete and/or incorrect. In contrast, the online methods are adaptive~\cite{Shen_SMART2023}, i.e., they update the coverage paths in real-time based on the environmental information collected by the onboard sensors.

\vspace{6pt}
$\bullet$ \textit{\textbf{Known environments}}: Several CPP methods have been developed for known environments. Zelinsky et al.~\cite{zelinsky1993planning} proposed a grid-based method, which assigns each grid cell a distance-to-goal value. The coverage path is then generated along the steepest ascent by selecting the unvisited cell in the local neighborhood with the largest distance-to-goal value. Huang~\cite{huang2001optimal} developed a CPP method that divided the coverage area into multiple subareas which are each covered by back-and-forth motion with different sweep directions to minimize the total number of turns. Xie et al.~\cite{xie2020path} proposed a TSP-based method to cover multiple disjoint areas for UAVs. Ramesh et al.~\cite{ramesh2022optimal} computed the optimal partitioning of a non-convex environment for minimizing the number of turns by solving a linear programming problem. Shen et al.~\cite{shen2021cppnet} presented a neural network based method, called CPP Network (CPPNet), which takes the occupancy grid map of the environment as input and generates a near-optimal coverage path offline.

 \vspace{6pt}
$\bullet$ \textit{\textbf{Unknown environments}}: Several CPP methods have been developed for unknown environments. Lee et al.~\cite{LBCO11} proposed a grid-based method using a path linking strategy called coarse-to-fine constrained inverse distance transform (CFCIDT) to generate a smooth coverage path, so as to avoid sharp turns. 
Gabriely and Rimon~\cite{gabriely2001spanning} developed the Spanning Tree Covering (STC) algorithm, which discretizes the workspace into coarse cells, achieving complete coverage with zero overlap but at a coarser resolution than the robot’s footprint, which may leave gaps. Full spiral STC (FS-STC)~\cite{gabriely2003competitive} extends STC by incorporating partially occupied cells into an augmented graph, enabling fine-cell level coverage at the resolution of the robot’s footprint. 
Gonzalez et al.~\cite{gonzalez2005bsa} presented the Backtracking Spiral Algorithm (BSA), which generates a spiral path online by following the covered area or obstacles on a specific lateral side until no uncovered free cell can be found in the local neighborhood. Then, a backtracking mechanism is used to escape from a dead-end.

Acar and Choset~\cite{acar2002sensor} developed a CPP method which detects the critical points on obstacles to decompose the coverage area into multiple subareas. Then, these subareas and their connections are represented by a graph and coverage is achieved by covering each subarea individually.
A limitation of this method is that it can not handle rectilinear environments. This method was further improved by Garcia and Santos ~\cite{garcia2004mobile}. Ferranti et al.~\cite{ferranti2007brick} presented the Brick-and-Mortar (BM) algorithm, which gradually enlarges the blocks of inaccessible cells (i.e., visited or obstacle cells) by selecting the target cell as the unexplored cell in the local neighborhood with the most surrounding inaccessible cells. The target cell is marked as visited if it does not block the path between any two accessible cells (i.e., explored or unexplored
cells) in the neighborhood; otherwise, it is marked as explored. If there is no unexplored cell in the neighborhood, then an explored cell is selected to traverse through and escape from the dead-end situation.

Luo and Yang~\cite{luo2008bioinspired} developed the Neural Network CPP (NNCPP) algorithm, which places a neuron at the center of every grid cell. The robot path is planned online based on the dynamic activity landscape of the neural network. The activity landscape is updated such that the obstacle cells have negative potential values for collision avoidance and uncovered cells have positive potential values for coverage. The algorithm can escape from the dead-ends  by means of neural activity propagation from the unclean areas to attract the robot. Gupta et al.~\cite{GRP09} proposed an Ising model approach for dynamic CPP. Viet et al.~\cite{viet2013ba} presented the 
Boustrophedon motions and A$^*$ (BA$^*$) algorithm, which performs the back-and-forth motion until no uncovered cell is found in the local neighborhood. Then, to escape from a dead-end situation, the algorithm searches for and computes the shortest path to the nearest backtracking cell using the A$^*$ algorithm~\cite{hart1968formal}.

Hassan and Liu~\cite{hassan2019ppcpp} proposed the Predator–Prey CPP (PPCPP) algorithm, which is based on a reward function consisting of three factors: 1) predation avoidance reward that maximizes the distance from the predator to prey (predator is a virtual stationary point and prey is the coverage spot of the robot), 2) smoothness reward that minimizes the turn angle for continuing motion in a straight line, and 3) boundary reward for covering boundary targets. Song and Gupta~\cite{song2018} developed the $\varepsilon^*$ algorithm, which utilizes an Exploratory Turing Machine (ETM) as a supervisor to guide the robot online with navigation and tasking commands for coverage via back-and-forth motion. To escape from a dead-end, the algorithm uses multi-resolution adaptive potential surfaces (MAPS) to find the unexplored regions. 

\vspace{6pt}
\subsubsection{\textbf{Energy Constrained Methods}} Recently, some CPP methods have been proposed considering the energy constraints of robots. Shen et al.~\cite{shen2020} proposed the $\varepsilon^*+$ algorithm which is an extension of the $\varepsilon^*$ algorithm~\cite{song2018} for energy-constrained vehicles. It generates the back-and-forth coverage path while continuously monitoring the remaining energy of the robot if it is sufficient to return back to the charging station. Upon battery depletion, the robot retreats back to the charging station using the A$^*$ algorithm and after recharging it advances to the nearest unexplored cell to restart the coverage of the remaining uncovered area. Shnaps and Rimon~\cite{shnaps2016} developed the battery powered coverage (BPC) algorithm, where the robot follows circular arcs formed by equidistant points to the source. When the battery discharges to a level just sufficient to return to the base, the robot returns to the charging station along the shortest path. Dogru and Marques~\cite{dogru2022eco} proposed energy constrained online path planning (ECOCPP) algorithm, which generates coverage and retreat paths by following the contours of covered area. 

\vspace{6pt}
\subsubsection{\textbf{Multi-robot Methods}} Some online CPP methods have also been proposed for multi-robot systems~\cite{palacios2019equitable,song2020care,hassan2020dec,wang2024apf} and curvature-constrained robots~\cite{shen2019online,kan2020online,maini2022online,GMDM_TRO}.

\begin{figure}[t]
    \centering  
    \includegraphics[width=0.50\textwidth]{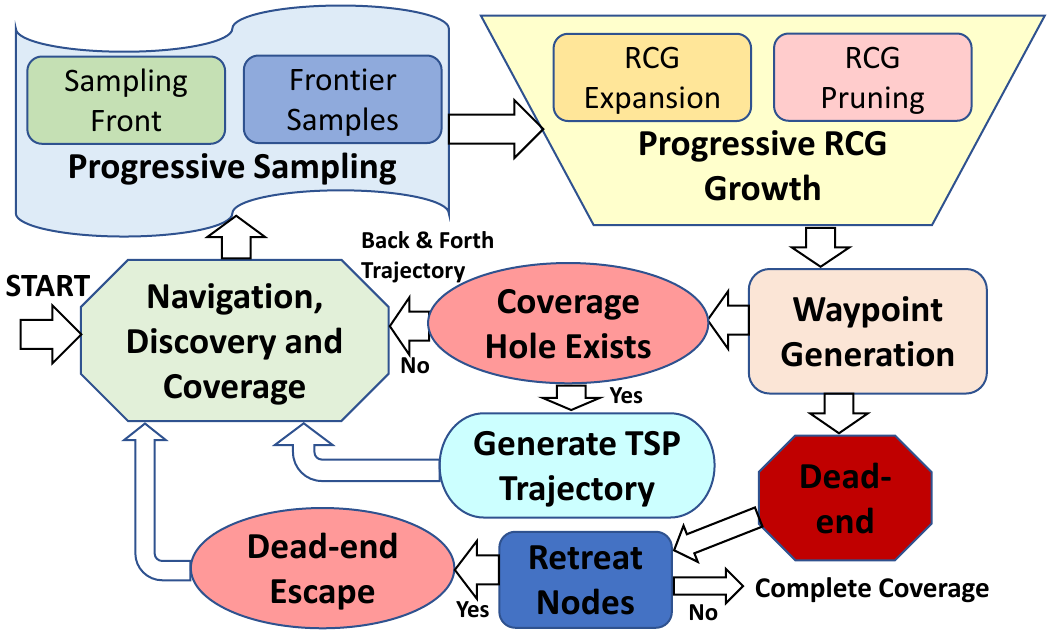}\\ \vspace{6pt}
    \caption{C$^*$ Operation.}\label{fig:CStarOperation}
    \vspace{6pt}
\end{figure}

\vspace{-0pt}
\subsection{Summary of the C$^*$ Algorithm}
This section provides a brief summary of the C$^*$ algorithm. C$^*$ provides complete coverage of unknown environments with dynamic sensing and adaptive coverage trajectory generation. It produces the desired back-and-forth coverage pattern with \textit{in situ} adaptation to the TSP-optimized coverage of coverage holes that are detected along the path during navigation. While most CPP methods are grid-based, C$^*$ is sample-based. It is built upon a novel concept of RCG which is incrementally built as the robot discovers the environment. 

Fig.~\ref{fig:CStarOperation} illustrates the C$^*$ operation. While traversing towards a waypoint on the coverage trajectory, the robot senses and discovers the environment. It then incrementally grows the RCG by progressive sampling in the obstacle-free region detected by onboard sensors (details in Section~\ref{Step1:Sensing}). The progressive sampling is done in a systematic manner within the sampling front using the frontier samples that are adjacent to the unknown area or obstacles  (details in Section~\ref{Step2:Sampling}). Then, the nodes of RCG are formed by these samples and serve as the waypoints for the robot. On the other hand, the edges are formed by connecting the nodes and serve as the potential traversal paths for the robot to form the coverage trajectory. The RCG forms a minimum sufficient graphical representation of the observed obstacle-free space (details in Section~\ref{Step3:RCGformation}). It has a sparse structure built by efficient sampling and pruning strategies, which provides computational and memory-efficient operation and non-myopic waypoint generation. 

To track the coverage progress, the nodes of RCG maintain the state as closed (visited) or open (unvisited). The next waypoint of the robot is selected as an unvisited neighboring node on the RCG. The waypoint is selected in a manner to generate the desired back-and-forth coverage pattern (details in Section~\ref{Step4:waypointgeneration}). Once the waypoint is selected, the robot moves towards it, maps the environment, and covers the region around the trajectory. The above process is then repeated.

\begin{table*}[t]{}
\scriptsize
\caption {Comparison of salient features of C$^*$ with the baseline algorithms.}\label{tab:feature}\vspace{-0pt}
\centering
\setlength\tabcolsep{2.0pt}
\begin{tabular}{l l l l l l l l l} 
 \toprule
 
\specialrule{0em}{1pt}{1pt}\vspace{3pt}

& \textbf{C$^*$} 

& \tabincell{l}{\textbf{PPCPP}\cite{hassan2019ppcpp}  \\ (2019)} 

& \tabincell{l}{\textbf{$\varepsilon^*$}\cite{song2018}  \\ (2018)} 

& \tabincell{l}{\textbf{BA$^*$}\cite{viet2013ba}  \\ (2013)}

& \tabincell{l}{\textbf{NNCPP}\cite{luo2008bioinspired}  \\ (2008)}

& \tabincell{l}{\textbf{BM}\cite{ferranti2007brick}  \\ (2007)}

& \tabincell{l}{\textbf{BSA}\cite{gonzalez2005bsa}  \\ (2005)}

& \tabincell{l}{\textbf{FSSTC}\cite{gabriely2003competitive}  \\ (2003)}\\ 
\toprule 

\specialrule{0em}{3pt}{-3pt}
\tabincell{l}{\textbf{Environment}\\{\textbf{Representation}}} 
& \tabincell{l}{RCG}
& \tabincell{l}{Circular\\disks} 
& \tabincell{l}{Grid map} 
& \tabincell{l}{Grid map} 
& \tabincell{l}{Grid map} 
& \tabincell{l}{Grid map} 
& \tabincell{l}{Grid map}
& \tabincell{l}{Grid map}\\

\specialrule{0em}{3pt}{3pt} 
\tabincell{l}{\textbf{Coverage} \\ \textbf{Pattern}} 
& \tabincell{l}{Back-and-forth \\ with adaptive TSP-\\ optimized coverage \\ of coverage holes} 
& \tabincell{l}{Spiral} 
& \tabincell{l}{Back-and-forth} 
& \tabincell{l}{Back-and-forth}
& \tabincell{l}{Back-and-forth} 
& \tabincell{l}{No obvious \\ pattern observed} 
& \tabincell{l}{Spiral}
& \tabincell{l}{Spiral}\\

\specialrule{0em}{3pt}{3pt}
\tabincell{l}{\textbf {Coverage}\\ \textbf{Method}} 
& \tabincell{l}{Creates RCG \\ incrementally by \\ progressive samp-\\ling, RCG tracks\\ coverage and \\ generates waypoints\\ for navigation}
& \tabincell{l}{A reward function \\ guides waypoint \\ generation in the \\ neighborhood} 
& \tabincell{l}{ETM generates the \\ navigation and \\ tasking commands\\ using the tracking \\ information stored\\  in MAPS} 
& \tabincell{l}{Follows a priority\\ of north-south-\\east-west to pick\\ waypoints} 
& \tabincell{l}{Dynamic neural\\ activity landscape \\ guides waypoint \\ generation} 
& \tabincell{l}{Picks unexplored\\ cell in the neigh-\\borhood as way-\\point to gradually \\ thicken the blocks \\  of inaccessible cells} 
& \tabincell{l}{Follows covered \\ area or obstacles \\ on a specific \\ lateral side }
& \tabincell{l}{Spanning tree \\ guides waypoint \\ generation in the \\ neighborhood}\\

\specialrule{0em}{3pt}{3pt}
\tabincell{l}{\textbf {Waypoint}\\ \textbf{Selection}} 
& \tabincell{l}{Non-myopic \\ due to sparse \\ structure of RCG}
& \tabincell{l}{Myopic} 
& \tabincell{l}{Myopic} 
& \tabincell{l}{Myopic} 
& \tabincell{l}{Myopic} 
& \tabincell{l}{Myopic} 
& \tabincell{l}{Myopic}
& \tabincell{l}{Myopic}\\

\specialrule{0em}{3pt}{3pt}
\tabincell{l}{\textbf{Proactive}\\ \textbf{Coverage Holes}\\ \textbf{Prevention}} 
& \tabincell{l}{Detects coverage\\  holes during back\\  \& forth motion\\ and covers them \textit{in}\\ \textit{situ} by a TSP path}
& \tabincell{l}{None} 
& \tabincell{l}{None} 
& \tabincell{l}{None} 
& \tabincell{l}{None}
& \tabincell{l}{None} 
& \tabincell{l}{None}
& \tabincell{l}{None}\\

\specialrule{0em}{3pt}{3pt}
\tabincell{l}{\textbf{Escape from}\\{\textbf{Dead-end}}} 
& \tabincell{l}{Moves to the \\ nearest retreat\\ node} 
& \tabincell{l}{Moves to the\\ nearest un-\\covered point} 
& \tabincell{l}{Searches MAPS\\ for high-potential\\ uncovered regions} 
& \tabincell{l}{Moves to the\\ nearest back-\\ tracking point} 
& \tabincell{l}{Moves to unco-\\vered area using\\ neural activity   \\ landscape} 
& \tabincell{l}{Picks an explored \\ cell repeatedly\\ until an unexp-\\lored  cell is found} 
& \tabincell{l}{Moves to the \\ nearest  back-\\ tracking point} 
& \tabincell{l}{Moves to parent\\ cell until a new\\ uncovered \\ neighbor is found}\\
\toprule
 \end{tabular}
 \vspace{-6pt}
 \end{table*}

C$^*$ has two important features, as described below.

\vspace{6pt}
\subsubsection{\textbf{Escaping Dead-ends}} It is possible that the robot reaches a point where there are no unvisited nodes on the RCG within the local neighborhood, known as a dead-end. To escape from the dead-ends, C$^*$ utilizes a dead-end escape strategy, which is built upon a concept of retreat nodes. The retreat nodes are all unvisited nodes on the RCG which are adjacent to the visited nodes. The dead-end escape strategy searches for the nearest retreat node using the A* algorithm~\cite{hart1968formal} and sets it as the next waypoint. Upon reaching this node, the robot resumes the back-and-forth coverage of the remaining area. 

\vspace{12pt}
\subsubsection{\textbf{Prevention of Coverage Holes}} 
A coverage hole is an isolated, obstacle-free, and uncovered region that is surrounded by obstacles and covered regions on all sides. While navigating a certain coverage pattern (e.g., back-and-forth or spiral), it is possible that the robot sometimes comes across such uncovered isolated regions, which if bypassed form coverage holes, as shown in Fig.~\ref{fig:bypass_part1}. Most CPP methods cover such holes later by return trajectories from distant regions, thus causing unnecessary overlaps, longer trajectory lengths and higher coverage times. C$^*$ provides a proactive strategy to prevent the formation of such coverage holes.  It detects such isolated regions \textit{in situ} and covers them using TSP-based locally optimal coverage trajectories before resuming the back-and-forth search to cover the remaining area, as shown in Fig.~\ref{fig:bypass_part2}. This proactive strategy prevents the formation of coverage holes, thereby improving the coverage performance and providing user satisfaction and convenience.

\begin{figure}[t]
    \centering
    \subfloat[Coverage holes formed by a typical back-and-forth trajectory.]{
        \includegraphics[width=0.20\textwidth]{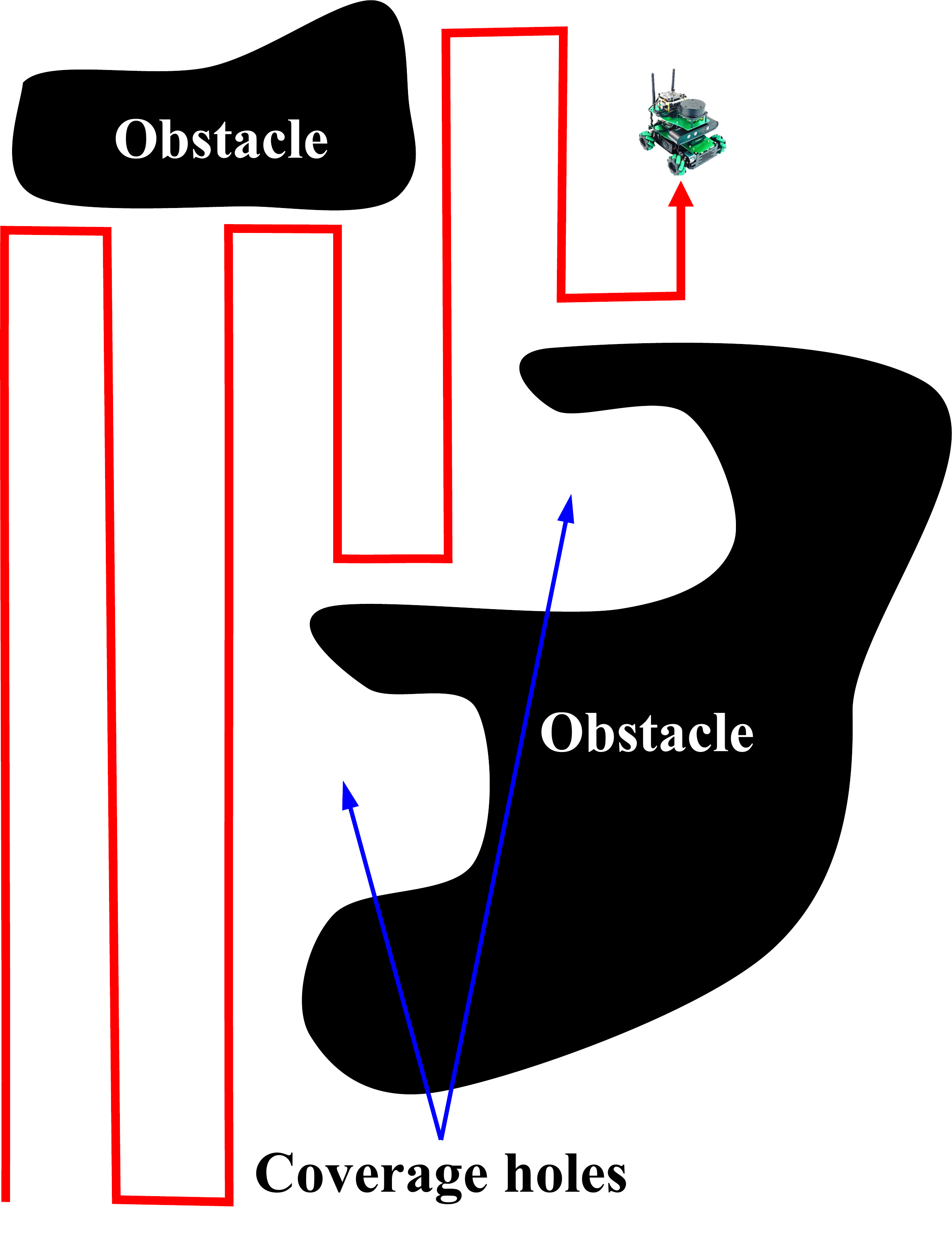}\label{fig:bypass_part1}}\hspace{-2pt}\quad
    \centering
    \subfloat[C$^*$ coverage using the back-and-forth motion with adaptation to TSP-trajectories to prevent the formation of coverage holes.]{
        \includegraphics[width=0.20\textwidth]{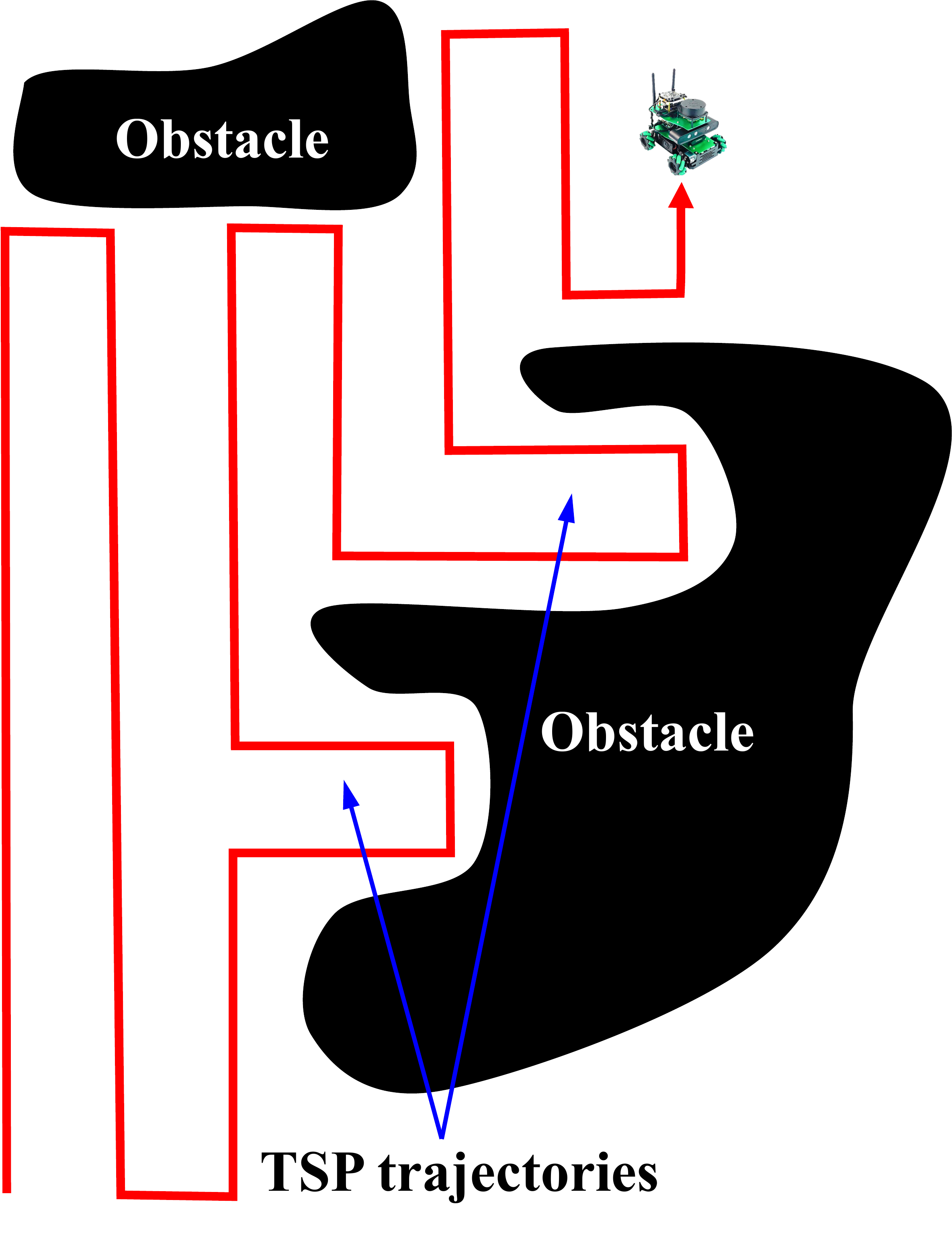}\label{fig:bypass_part2}}\\
          \caption{Coverage holes prevention by  C$^*$ using TSP-trajectories.}\label{fig:coveragehole}
          \vspace{-6pt}
\end{figure}

\vspace{0pt}
\subsection{Comparison of C$^*$ with the Baseline Algorithms}
Table~\ref{tab:feature} presents the salient features of C$^*$ in comparison to seven different baseline algorithms. 
\begin{itemize}
\item The baseline algorithms are grid-based and maintain a fine grid map of the environment. On the other hand, C$^*$ is sample-based and represents the environment using the RCG, which is incrementally built by progressive sampling of the observed area. The RCG is a minimum sufficient graph generated by efficient sampling and pruning techniques to provide high computational and memory efficiency.
\item  The baseline algorithms generate myopic waypoints, i.e., within the local neighborhood on a grid. On the other hand, C$^*$  generates non-myopic waypoints by means of its sparse RCG structure to reduce the total number of algorithm iterations for complete coverage. 
\item The baseline algorithms could produce coverage holes during the coverage progress, which are covered later by return trajectories. This can result in strongly overlapped trajectories and high coverage times. On the other hand, C$^*$ actively detects potential coverage holes during navigation, performs TSP-based optimal coverage of such regions, and then resumes the back-and-forth motion. This not only prevents the formation of undesired coverage holes but also saves coverage time. 
\item The baseline algorithms using spiral paths are limited in their applications because turning is regarded as expensive. In this regard, C$^*$ produces the desired back-and-forth coverage pattern. 
\item The baseline algorithms using the cellular decomposition methods~\cite{acar2002sensor} require the detection of critical points on the obstacles. C$^*$ does not have any such requirement. 
\item Finally, the algorithmic simplicity and computational efficiency of C$^*$ make it easy to implement on onboard processors for real-time applications.
\end{itemize}

\vspace{-6pt}
\subsection{Contributions}
The paper developed a novel sample-based CPP algorithm, called C$^*$, for real-time adaptive coverage of unknown environments. C$^*$ is built upon the concept of an RCG, which is incrementally constructed during robot navigation by progressive sampling along the boundary of discovered area. By virtue of efficient sampling and pruning techniques, the RCG is constructed to be a minimum-sufficient graph, which i) provides computational and memory efficiency for real-time implementation, and ii) generates non-myopic waypoints to reduce the number of iterations for complete coverage. C$^*$ features a dead-end escape strategy to escape from dead-ends using the concept of retreat nodes. In addition, C$^*$ features a coverage hole prevention strategy, which adapts the back-and-forth coverage trajectory to a TSP-based locally optimal trajectory to prevent the formation of coverage holes and thereby minimize the overall coverage time. The above features make C$^*$ highly efficient in generating near optimal trajectories, which is successfully demonstrated by i) extensive simulations on diverse scenarios yielding significantly better performance over several baseline methods, and ii) real-experiments.

\vspace{-0pt}
\subsection{Organization}
The remaining paper is organized as follows. Section~\ref{sec:problem_description} presents the CPP problem and the objectives of C$^*$. Section~\ref{sec:cstarmethodology} describes the details of the C$^*$ algorithm with its iterative steps. Section~\ref{sec:analysis} provides the algorithm complexity and complete coverage proof. Section~\ref{sec:results} shows the comparative evaluation results of C$^*$ with the baseline algorithms along with the experimental validation. This section also presents the application and comparative evaluation of C$^*$ for energy-constrained CPP and multi-robot CPP problems. Finally, Section~\ref{sec:conclusions} concludes the paper with suggested recommendations for future work. Appendices~\ref{appendixA} and \ref{appendixB} show additional results and present scalability analysis for different map sizes, respectively.

\vspace{6pt}
\section{Problem Description}\label{sec:problem_description}
Let $\mathcal{R}$ be a coverage robot, as shown in Fig.~\ref{fig:robotsensing}, which is equipped with the following:
\begin{itemize}
\item A localization device (e.g., GPS, IMU, wheel encoder, and an Indoor Localization System) or SLAM~\cite{Song_sose2015} to locate the position of the robot,
\item A range detector and mapping sensor (e.g., lidar and ultrasonics) of range $r_d \in \mathbb{R}^+$, to detect obstacles within its field of view (FOV) and map the environment, and 
\item  A coverage device/sensor (e.g., cleaning brush) of range $r_c \in \mathbb{R}^+$, to carry out the main coverage task.
\end{itemize} 

Let $\mathcal{A}\subset\mathbb{R}^2$ be the unknown area populated by obstacles of arbitrary shapes that is tasked for coverage by robot $\mathcal{R}$. Let $\mathcal{A}^o\subset\mathcal{A}$ be the space occupied by  obstacles. Then the obstacle-free space $\mathcal{A}^f=\mathcal{A}\setminus\mathcal{A}^o$ is the coverage space that is desired to be covered by the coverage device of robot $\mathcal{R}$ to perform a certain task (e.g., cleaning and demining). It is assumed that the coverage space is connected, i.e., any point in the coverage space is reachable by the robot from any other point, to form the complete coverage trajectory.

\begin{figure}[t]
    \centering
    \subfloat[Sensing and tasking areas.]{
        \includegraphics[width=0.20\textwidth]{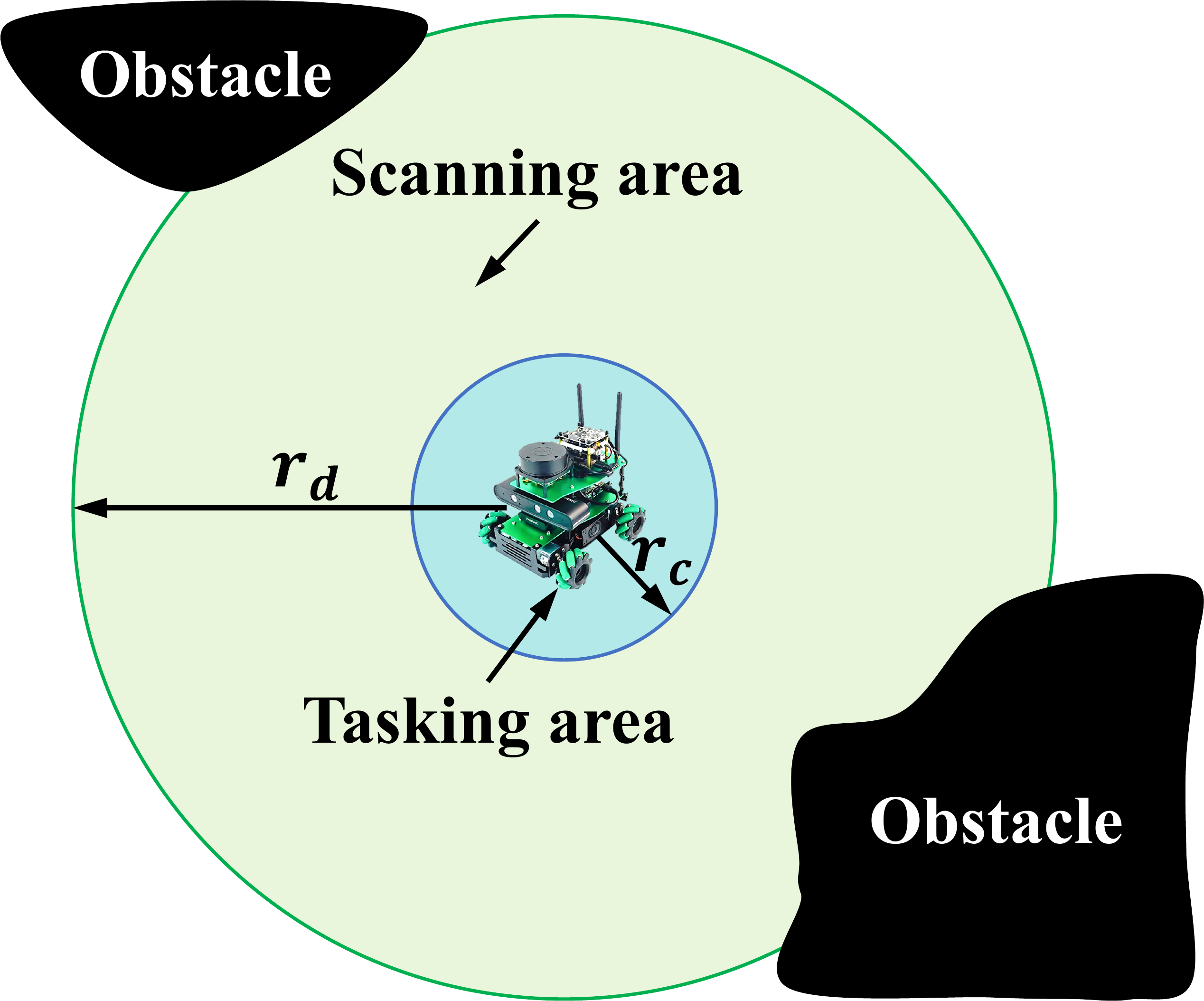}\label{fig:robotsensing}}\hspace{-15pt}\quad
    \subfloat[Coverage trajectory.]{
        \includegraphics[width=0.28\textwidth]{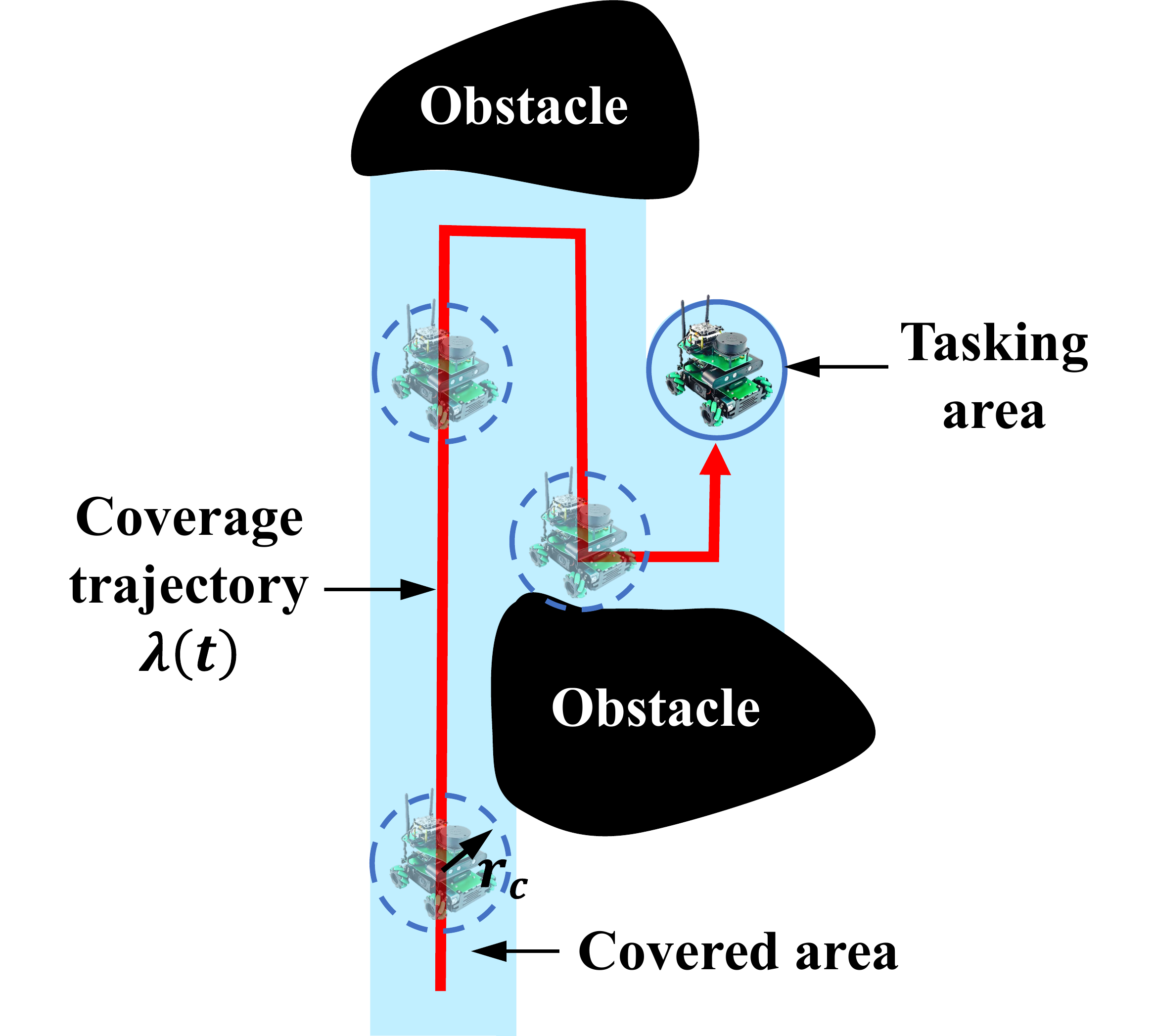}\label{fig:trajectory}}\hspace{-8pt}\quad
    \caption{Robot: a) sensing and tasking areas and b) coverage trajectory.}
    \vspace{6pt}
\end{figure}

\vspace{3pt}
\begin{defn}[\textbf{Coverage Trajectory}] Let $\lambda: [0,T]\rightarrow \mathcal{A}^f$, where $\lambda(t)\in\mathcal{A}^f$ is the point visited by robot $\mathcal{R}$ at time $t\in[0,T]$ and $T\in \mathbb{R}^+$ is the total coverage time. Then, the coverage trajectory is defined as $\Lambda_T\triangleq[\lambda(t)]_{t=0}^T$, with $\lambda(0)$ and $\lambda(T)$ as its start and end points, respectively.  
\end{defn}
Fig.~\ref{fig:trajectory} shows an example of the coverage trajectory and the area covered by the robot's coverage device  during navigation.  Let $\mathcal{A}^{c}(\lambda(t))\subset \mathcal{A}^f$ be the area covered by the robot's coverage device at point $\lambda(t) \in \Lambda_T$ on the coverage trajectory. Then, let $\mathcal{A}^c=\bigcup_{t'\in[0,t]} \mathcal{A}^{c}(\lambda(t'))$ denote the total area covered until time $t$ and $\mathcal{A}^{uc}=\mathcal{A}^f\setminus{\mathcal{A}}^c$ denote the uncovered area. The coverage trajectory has to be generated such that the robot's coverage device covers the entire free space $\mathcal{A}^f$.

\vspace{3pt}
\begin{defn}[\textbf{Complete Coverage}]  The coverage of $\mathcal{A}^f$ is said to be complete in time $T$ if  
\begin{equation}
\mathcal{A}^f \subseteq \bigcup_{t\in[0,T]} \mathcal{A}^{c}(\lambda(t)).
\end{equation}

\vspace{3pt}
\end{defn}
$\bullet$ \textit{\textbf{Objective}}: The objective of C$^*$ is to generate the coverage trajectory $\Lambda_T$ of robot $\mathcal{R}$ to provide complete coverage of $\mathcal{A}^f$ in minimum time. It is essential that C$^*$ does not get stuck at dead-ends and provide uninterrupted coverage. Furthermore, it is envisioned that C$^*$ produces the desired back-and-forth coverage pattern with synergistic adaptations to the TSP-based locally optimal coverage trajectories to prevent the formation of coverage holes. Note that significant coverage time could be saved by preventing coverage holes to avoid return trajectories from distant regions to cover such holes later. Moreover, it is desired that C$^*$ minimizes the coverage time by minimizing the trajectory length, overlaps, and the number of turns. Finally, it is desired  that C$^*$ is scalable to multi-robot CPP problems.

\vspace{-0pt}
\section{C$^*$ Algorithm}
\label{sec:cstarmethodology}
C$^*$ is an adaptive CPP algorithm for unknown environments. It generates the coverage trajectory incrementally during robot navigation based on the environment information collected by sensors. C$^*$ generates the coverage trajectory via the following iterative steps until complete coverage is achieved. 

\begin{itemize}
\item [1.] \textit{Navigation, discovery and coverage}: In this step, the robot moves towards the current waypoint of the coverage  trajectory, discovers and maps the environment using its onboard sensors (e.g., the range detector), and in parallel performs the coverage task using its coverage device.
\item [2.] \textit{Progressive sampling}: Once the robot reaches the current waypoint, this step creates a sampling front in the newly discovered area and systematically generates new samples for the purpose of growing the RCG.
\item [3.] \textit{Progressive RCG construction}: Once the sampling is done, this step grows the current RCG using the new samples via efficient expansion and pruning strategies. 
\item [4.] \textit{Coverage trajectory generation}: Once the RCG is updated, this step determines the next waypoint of the coverage trajectory using the updated RCG.
\end{itemize}

\vspace{-0pt}
\subsection{Overview of the Algorithm Steps} \label{AlgOverview}
Let the iterations of C$^*$ be denoted by $i \in \{0,1...m\}$, where $m \in \mathbb{N}^+$. The total number of iterations is  $m+1$ in which complete coverage is achieved. In each iteration $i$ the above four steps are performed. Let $p_0 \in \mathcal{A}^f$ be the starting point. 

$\bullet$ \textit{Iteration $i=0$}: C$^*$ initializes when the robot arrives at $p_0$ at time $t_0=0$, i.e., $\lambda(t_0)=p_0$. Then, it stays at $p_0$ and senses the environment within the FOV of its range detector. Next, it samples the discovered area and generates the first installment of samples. (Details of sampling are in Section~\ref{Step2:Sampling}). Using these samples, it creates the first portion of RCG by setting samples as nodes and connecting nodes to form edges. (Details of RCG construction are in Section~\ref{Step3:RCGformation}). Then, this first portion of RCG, in turn, selects a node as an intermediate goal and sets its position as the first waypoint $p_1 \in \mathcal{A}^f$. (Details of waypoint generation are in Section~\ref{Step4:waypointgeneration}). This finishes iteration $i=0$ with the output to the robot as waypoint $p_1$. 

$\bullet$ \textit{Iteration $i=1$}: The robot moves from $p_0$ to $p_1$, thus forming the first portion of coverage trajectory. During this motion, the robot continuously discovers the environment using its range detector to incrementally build the map. While navigating the robot also performs the coverage task along the trajectory using its coverage device/sensor. The robot reaches $p_1$ at time $t_1\in(0,T]$, i.e., $\lambda(t_1)=p_1$.  Upon reaching $p_1$, it repeats Steps $2-4$ of sampling, RCG growth, and trajectory generation. This finishes iteration $i=1$ with the output to the robot as the next waypoint $p_2 \in \mathcal{A}^f$.

$\bullet$ \textit{Iterations $i \in \{1,...m-1\}$}: In general, the robot starts at point $p_{i-1} \in \mathcal{A}^f$. The input to  iteration $i$ is the next waypoint $p_i \in \mathcal{A}^f$, which was generated as the output of iteration $i-1$. Thus, at iteration $i$, the robot moves from $p_{i-1}$ to $p_i$ and forms the $i^{th}$ portion of the coverage trajectory. During this motion it discovers the environment using its range detector. It also performs the coverage task along the coverage trajectory using its coverage device/sensor. The robot reaches $p_i$ at time $t_i\in(0,T]$, i.e., $\lambda(t_i)=p_i$. Upon reaching $p_i$, it generates samples in the newly discovered area. Note that the newly discovered area is all the area that has been sensed during motion from $p_{i-1}$ to $p_i$ and that has not been previously observed. Once the samples are created, it grows the existing RCG using these samples. Finally, the updated RCG generates the next waypoint $p_{i+1}\in \mathcal{A}^f$ to continue forming the coverage trajectory until complete coverage is achieved.  Note that it is possible that at a certain iteration no new area is discovered depending on the scenario; thus, no new samples are created and RCG is not grown. In this case, the RCG relies on its existing structure and uses its available nodes to generate the next waypoint. 

$\bullet$ \textit{Iteration $i=m-1$}: The RCG generates the final destination point $p_m  \in \mathcal{A}^f$. This means that at the end of iteration $m-1$, the RCG is a) fully grown, i.e., the coverage area is fully discovered, sampled, and no more eligible samples exist that could be used to grow the RCG, and b) fully exploited, i.e., the last available node is selected to generate $p_m$ as the final waypoint on the coverage trajectory. 

$\bullet$ \textit{Iteration $m$}: Finally, the robot starts at $p_{m-1}$ and moves to $p_{m}$. At this stage, no new area is discovered; thus, no new frontier samples are created and the RCG is not grown any further. Thus, upon reaching $p_m$, the coverage is complete. 

In this manner, C$^*$ generates the waypoint sequence to form the coverage trajectory as described below. 

\vspace{6pt}
\begin{defn}[\textbf{Waypoint Sequence}] The waypoint sequence  $\mathcal{P}=\{p_0,p_1,...p_{m}\}$ is the sequence of intermediate goal points $p_i \in \mathcal{A}^f$ for the robot $\mathcal{R}$ to form the coverage trajectory, where $p_0$ and $p_m$ are its start and end points, respectively. 
\end{defn}

By following the waypoint sequence $\mathcal{P}$, the robot forms the coverage trajectory $\Lambda_T\equiv \lambda(t_o)\rightarrow\lambda(t_1)\rightarrow....\lambda(t_m)$, such that $\lambda(t_i)=p_i, i=0,...m$, where $t_0=0$ and $t_m=T$. Note that $p_i\neq p_j, \forall i\neq j,  i,j \in \{0,1,...m\}$ to avoid any loops. 

\vspace{-3pt}
\subsection{Details of the Algorithm Steps}\label{AlgDetails}
The details of the four steps are presented below.

\vspace{6pt}
\subsubsection{\textbf{Navigation, Discovery and Coverage}}\label{Step1:Sensing}
Since C$^*$ operates in unknown environments, it first collects data using onboard sensors to dynamically discover and map the environment in order to find a coverage trajectory. In each iteration it could only construct a partial map of the environment. 

At iteration $i=0$, the robot arrives at the start point $p_0$, where it senses the environment to make an initial map around $p_0$. Subsequently, at iteration $i \in \{1,...m-1\}$, the robot has a waypoint $p_i$ available. As such, it moves from $\lambda(t_{i-1})=p_{i-1}$ at time $t_{i-1}$ and reaches $\lambda(t_{i})=p_i$ at time $t_{i}$, thus forming the $i^{th}$ portion of the coverage trajectory. During this motion from $p_{i-1}$ to $p_{i}$, the robot performs the coverage task using its coverage device. It also continuously discovers the environment and maps the obstacles within the FOV of its range detector. Fig.~\ref{fig:progressivesampling} shows the dynamic discovery process.

\begin{figure}[t]
    \centering  
    \includegraphics[width=0.32\textwidth]{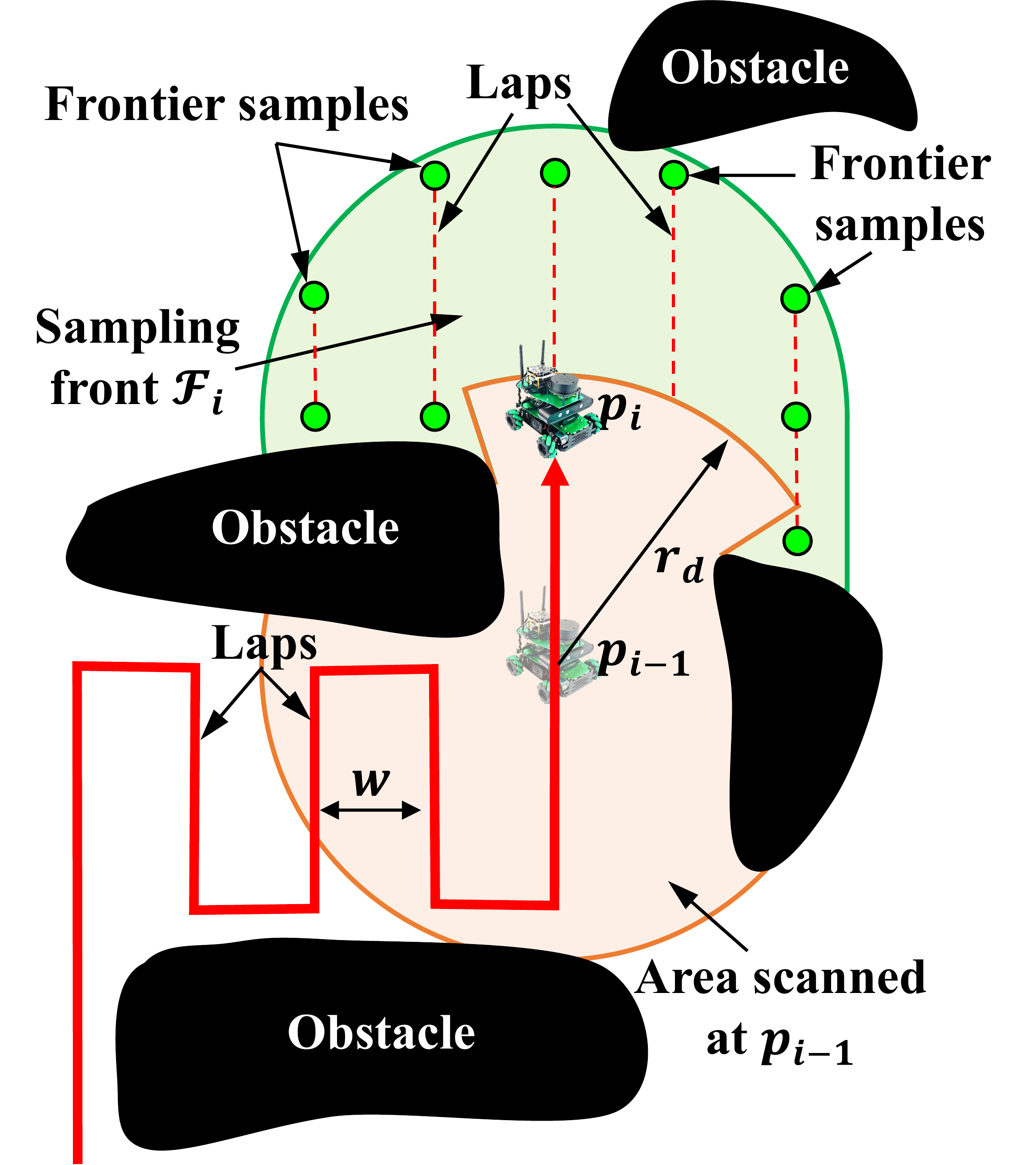}\\ \vspace{-0pt}
    \caption{An example showing i) dynamic discovery of the environment during navigation from $p_{i-1}$ to $p_{i}$ to generate the sampling front $\mathcal{F}_i$, and ii) progressive sampling on $\mathcal{F}_i$ to generate the frontier samples along the boundary of unknown and obstacle regions.}\label{fig:progressivesampling}
    \vspace{-9pt}
\end{figure}

\begin{figure*}[t]
    \centering
    \subfloat[Generation of frontier samples.]{
        \includegraphics[width=0.24\textwidth]{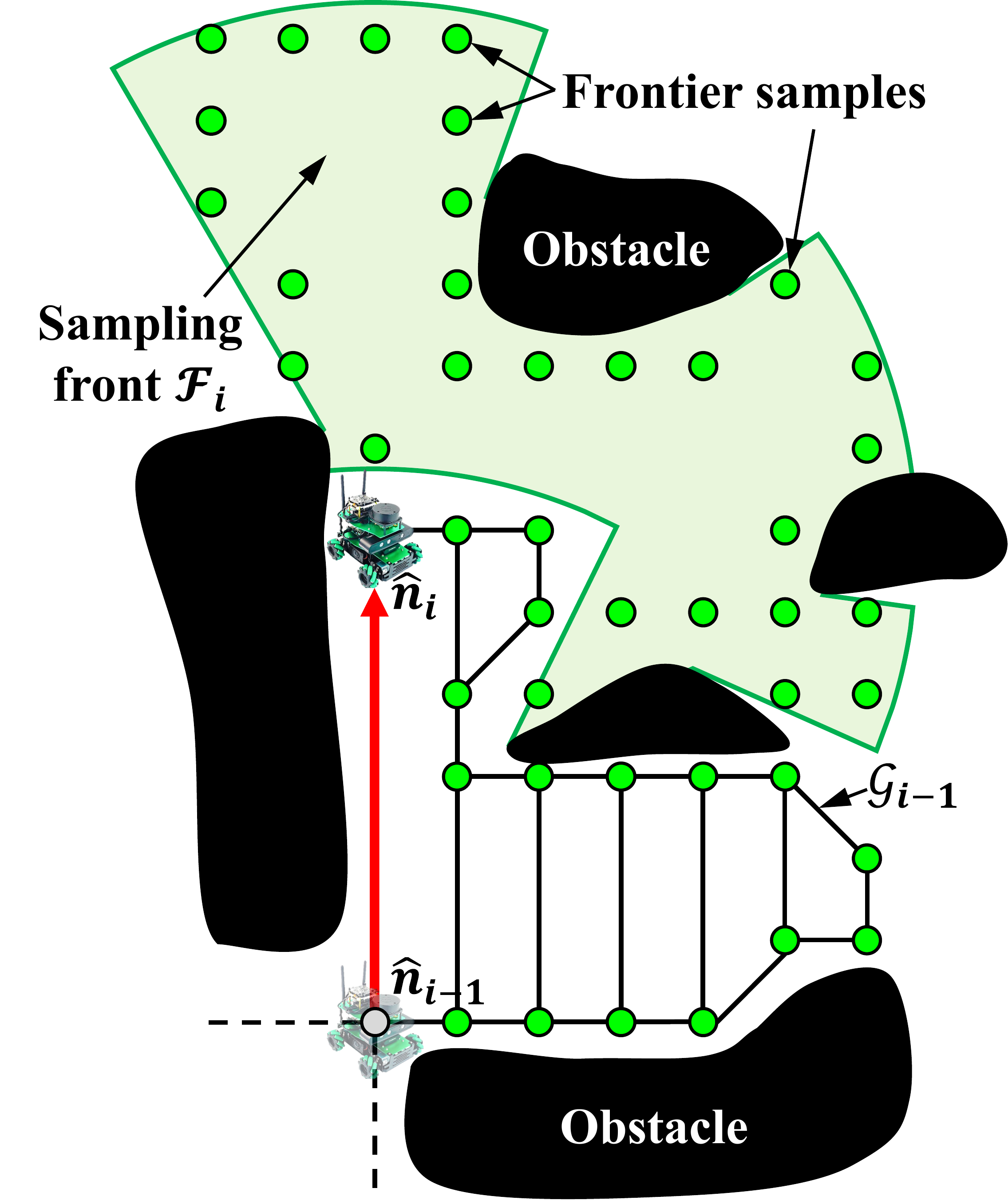}\label{fig:RCG_part1}}\hspace{-8pt}\quad
    \centering
    \subfloat[RCG expansion process.]{
        \includegraphics[width=0.24\textwidth]{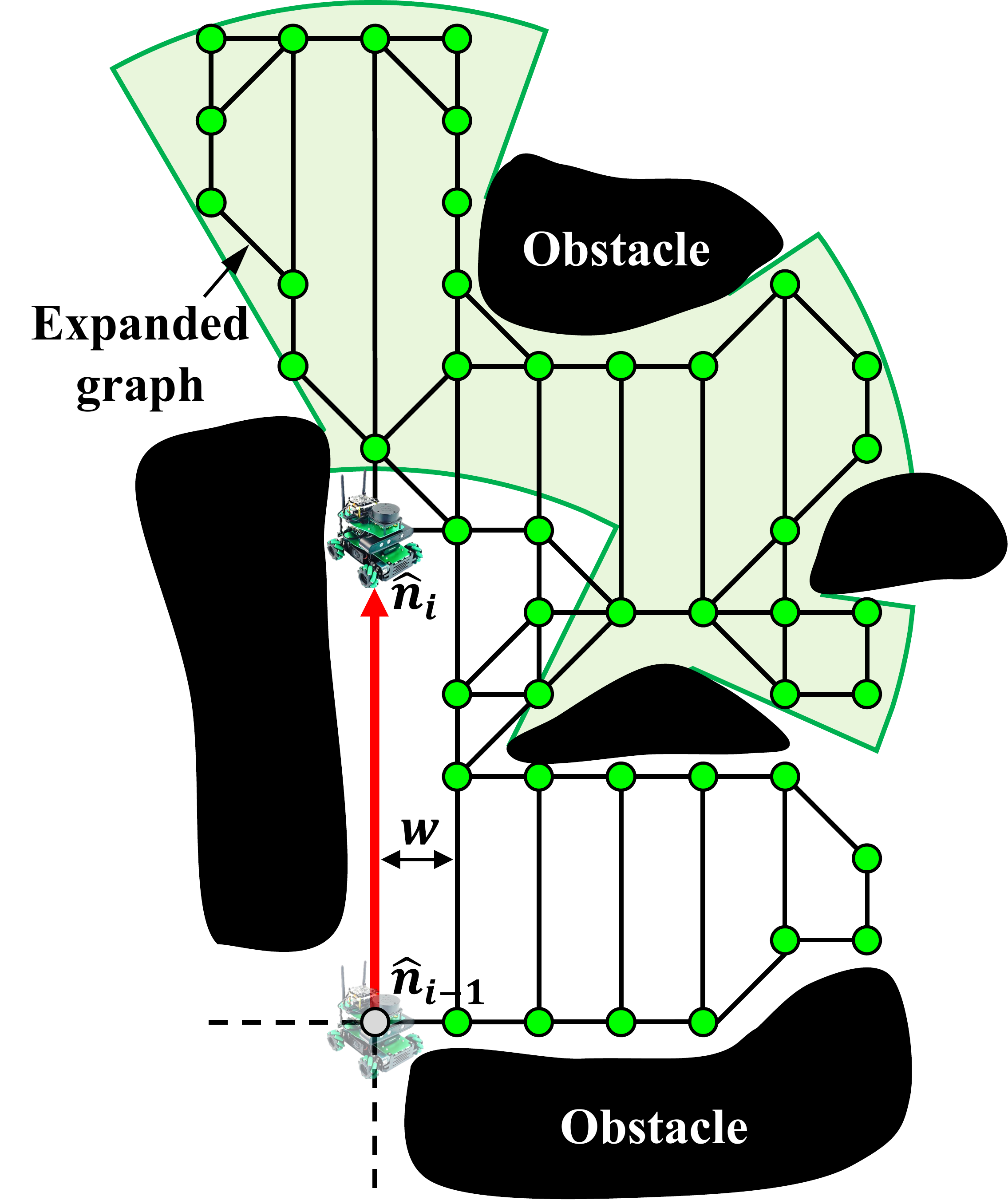}\label{fig:RCG_part2}}\hspace{-8pt}\quad
    \centering
    \subfloat[RCG pruning process.]{
        \includegraphics[width=0.24\textwidth]{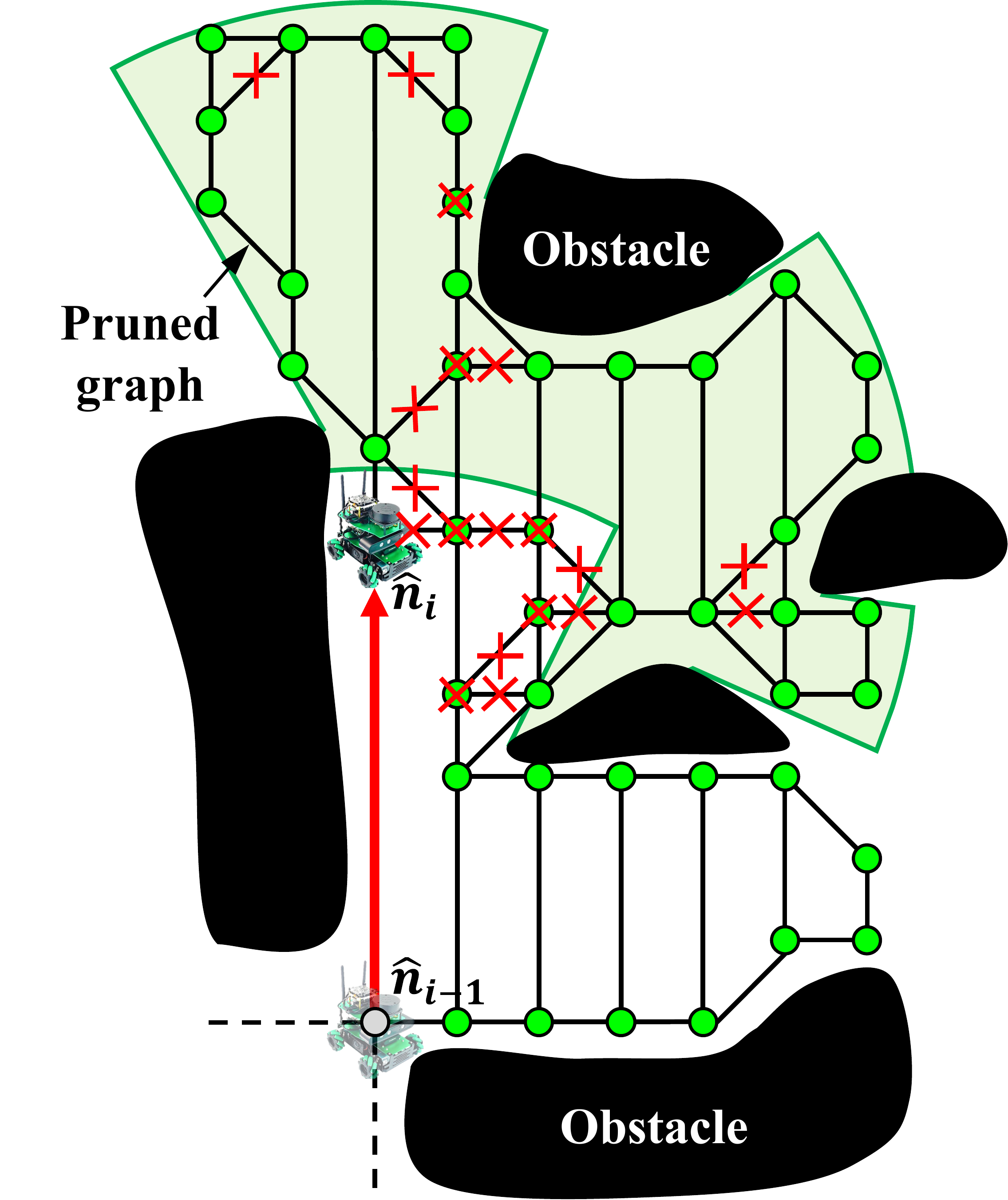}\label{fig:RCG_part3}}\hspace{-8pt}\quad
    \centering
    \subfloat[Updated RCG after pruning.]{
        \includegraphics[width=0.24\textwidth]{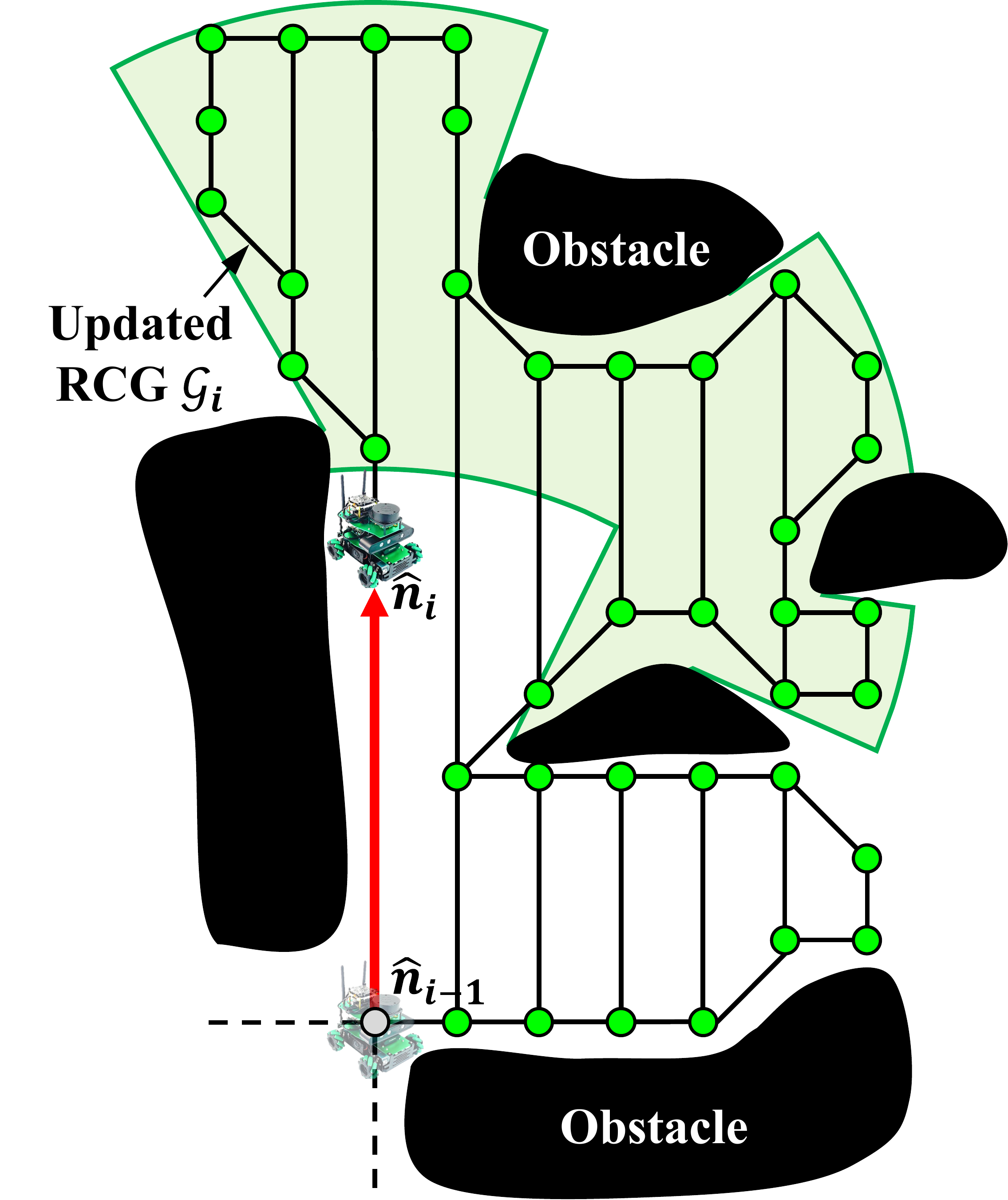}\label{fig:RCG_part4}}\hspace{-8pt}\\
          \caption{RCG growth at iteration $i$ by 1) expansion into $\mathcal{F}_i$ using the frontier samples and 2) pruning the inessential nodes and edges.}\label{fig:RCG_construction}
          \vspace{-6pt}
\end{figure*}

Let $\mathcal{A}^{d}(\lambda(t))\subset \mathcal{A}$ be the area scanned by the range detector when the robot is at a point $\lambda(t) \in \mathcal{A}^f$ at time $t$. Then, the total area discovered and mapped in the $i^{th}$ iteration is the cumulation of the area scanned from time $t_{i-1}$ to $t_i$, given as 

\begin{equation}
\mathcal{A}^d_{i} = \bigcup_{t\in[t_{i-1},t_{i}]}\mathcal{A}^d(\lambda(t)).
\end{equation}
This area $\mathcal{A}^d_{i}$ is used for progressive sampling and growing the RCG in the next steps. The total area discovered from the beginning of the coverage operation is given as $\mathcal{A}^d_{tot}=\bigcup_i\mathcal{A}^{d}_i$. The remaining area $\mathcal{A}^u=\mathcal{A}\setminus \mathcal{A}^d_{tot}$ is still unknown.

At iteration $m$, the robot moves from $p_{m-1}$ to $p_{m}$ and senses the environment; however, since all area has been already mapped until iteration $m-1$, no new area is discovered. 

\vspace{12pt}
\subsubsection{\textbf{Progressive Sampling}}\label{Step2:Sampling}
The progressive sampling is the process of incremental sampling in the newly discovered area $\mathcal{A}^d_i$.  At iteration $i \in \{0,...m-1\}$, once the robot reaches $p_i$ and completes the discovery step, C$^*$ starts the progressive sampling step. The objective of sampling is to grow the RCG in $\mathcal{A}^d_i$. The RCG, in turn, generates the new waypoint $p_{i+1}$ to form the future trajectory of the robot. Note that at the final iteration $m$ there is no need for progressive sampling since the robot reaches the end point $p_m$ of the coverage trajectory.

\vspace{6pt}
a) \textit{\textbf{Creation of the Sampling Front}}: For progressive sampling,  C$^*$ first creates a sampling front. 

\begin{defn}[\textbf{Sampling Front}] The sampling front $\mathcal{F}_{i} \subseteq \mathcal{A}^d_i$, $i\in\{0,...m-1\}$, is the obstacle-free and unsampled portion of the area $\mathcal{A}^d_i$ discovered in the $i^{th}$ iteration.
\end{defn}

Figs.~\ref{fig:progressivesampling} and~\ref{fig:RCG_part1} show examples of the sampling front $\mathcal{F}_{i}$ generated by the robot during navigation from $p_{i-1}$ to $p_{i}$. Note that it is possible that some portion of $\mathcal{A}^d_i$ has been already discovered before and sampled in a previous iteration. Thus, $\mathcal{F}_{i}$ includes only the newly discovered and unsampled area.

\vspace{6pt}
b) \textit{\textbf{Sampling Strategy}}: The sampling strategy is described next. Let $\mathcal{S}$ be the set of all feasible samples that lie in the coverage space $\mathcal{A}^f$. Once $\mathcal{F}_i$ is generated, it is sampled with the frontier samples $\mathcal{S}_i \subseteq \mathcal{S}$ as defined below.   

\begin{defn}[\textbf{Frontier Sample}]\label{define:frontier_node}
A sample $s \in \mathcal{S}$ is said to be a frontier sample if it is adjacent to   
\begin{itemize}
    \item the unknown area, and/or
    \item an obstacle or boundary. 
\end{itemize} 
\end{defn}

Thus, a sample $s \in \mathcal{S}$ is determined to be a frontier sample if a ball $\mathcal{B}(s,w)$ of radius $w$ centered at $s$ contains unknown and/or obstacle area.
The set of frontier samples $\mathcal{S}_i \subseteq \mathcal{S}$ is systematically generated in the sampling front $\mathcal{F}_i$ by sampling on laps as described below. 

\begin{defn}[\textbf{Lap}]
A lap is a virtual straight line segment in the coverage space whose two ends terminate either at an obstacle or the coverage space boundary. 
\end{defn}
The laps are created in the sampling front $\mathcal{F}_i$ for the purpose of efficient, nonuniform and sparse sampling. At iteration $i\in\{0,...m-1\}$, the laps are drawn in the sampling front $\mathcal{F}_i$, such that they are parallel to each other and the distance between any two adjacent laps is equal to the sampling resolution $w\in \mathbb{R}^+$.  
The laps are created along the direction of desired back-and-forth motion. The first lap is drawn in the sampling front $\mathcal{F}_0$ starting from $p_0$. Note that it is possible that only portions of laps are created within $\mathcal{F}_i$, that is their ends may not reach an obstacle or the boundary. The laps are also drawn to be consistent, which means that any two laps separated by an obstacle in between, such that one lap is above and the other below the obstacle, are aligned on the same line whenever possible. Fig.~\ref{fig:progressivesampling} shows the examples of laps on the coverage trajectory and the new ones created in $\mathcal{F}_i$ for sampling.

Once the laps are created in $\mathcal{F}_i$, they are sampled with the frontier samples $\mathcal{S}_i \subseteq \mathcal{S}$. The first sample is generated at $p_0$ in  $\mathcal{F}_0$ at iteration $i=0$. Thereafter, all laps in $\mathcal{F}_0$ are populated with samples. The sampling procedure is as follows. At iteration $i\in\{0,...m-1\}$, the samples are placed on each lap in $\mathcal{F}_i$, such that the spacing between any two adjacent samples on a lap is $\delta w, \ \delta\in\mathbb{N}^+$, where $\delta \geq 1$ is the minimum multiplier that allows the placement of a frontier sample. Note that multiple frontier samples could be placed on a lap within a sampling front depending on the obstacle geometries.   

The progressive sampling produces nonuniform and sparse samples 
which provide computational and memory efficiency and enable non-myopic waypoint selection (Section~\ref{Step4:waypointgeneration}).
Figs.~\ref{fig:progressivesampling} and~\ref{fig:RCG_part1} show examples of progressive sampling.

\vspace{9pt}
\subsubsection{\textbf{Progressive RCG Construction}}
\label{Step3:RCGformation}
The objectives of a RCG are to track the coverage progress, generate adaptive waypoints to produce the coverage trajectory through unexplored regions, and provide escape routes from the dead-ends. First, we define an RCG as follows.

\begin{defn}[\textbf{RCG}]\label{graph}
An RCG $\mathcal{G}=\left(\mathcal{N},\mathcal{E}\right)$ is defined as a simple, connected, and planar graph, where 
\begin{itemize}
\item $\mathcal{N} = \left\{n_j: j = 1,\ldots|\mathcal{N}|\right\}$ is the node set, such that each node corresponds to a frontier sample and represents a waypoint on the coverage trajectory of the robot, and 

\item $\mathcal{E}=\left\{e(n_j,n_k): j = 1,\ldots|\mathcal{N}|, k = 1,\ldots|\mathcal{N}|\right\}$ is the edge set, such that each edge connects two nodes and represents a collision-free traversal path for the robot.
\end{itemize}
\end{defn}

\begin{defn}[\textbf{Neighbor}]
A node $n'\in\mathcal{N}$ is said to be a neighbor of a node $n\in\mathcal{N}$ on RCG and vice versa if there exists an edge $e(n,n')\in \mathcal{E}$ connecting $n$ and $n'$.
\end{defn} 

\begin{defn}[\textbf{Neighborhood}]
The neighborhood of a node $n\in\mathcal{N}$ is the set $\mathcal{N}_b(n)\subset\mathcal{N}$ of all its neighbors.
\end{defn} 

At iteration $i \in \{0,...m-1\}$, the RCG $\mathcal{G}_{i}=\left(\mathcal{N}_{i},\mathcal{E}_{i}\right)$ is grown in the sampling front $\mathcal{F}_i$ using the RCG a) expansion and b) pruning strategies described below.  Fig.~\ref{fig:RCG_construction} shows an example of the RCG expansion and pruning processes. 

\vspace{12pt}
a) \textit{\textbf{RCG Expansion Strategy}}:
The RCG expansion strategy describes the procedure to add new nodes and edges to the existing RCG. Initially, the RCG has empty node and edge sets. Subsequently, at each iteration $i \in \{0,...m-1\}$, once the robot reaches the waypoint $p_i$ and performs the progressive sampling, the RCG is expanded into the sampling front $\mathcal{F}_i$ using the frontier samples. Fig.~\ref{fig:RCG_part2} shows an example of RCG expansion into the sampling front $\mathcal{F}_i$. 

The RCG is expanded as follows. First, all the frontier samples $\mathcal{S}_i$ in $\mathcal{F}_i$ are assigned to be the new nodes $\mathcal{N}_{new}$ of RCG, where $|\mathcal{N}_{new}|=|\mathcal{S}_i|$. The new nodes are added to the existing node set $\mathcal{N}_{i-1}$ to get $\mathcal{N}'_{i-1}= \mathcal{N}_{i-1}\cup \mathcal{N}_{new}$. Then, the new edges $\mathcal{E}_{new}$ are created by making valid connections between new nodes and to the nodes of existing RCG. To add the edges all laps in the coverage space $\mathcal{A}^f$ are considered. Note that a lap in $\mathcal{A}^f$ may contain both new and old nodes of the existing RCG. Thus, for each new node $n \in \mathcal{N}_{new}$ on each lap, new edges are drawn to connect it with 

\begin{itemize}
\item [i.] the adjacent nodes on the same lap. 
\item [ii.] all nodes within a distance of $\sqrt{2}w$ on each adjacent lap.  
\end{itemize}

Note that it is possible that a node has more than two adjacent laps on one side separated by an obstacle. 
Each edge is checked for feasibility before connecting, i.e., if it lies in $\mathcal{A}^f$. In this manner, all possible edges are drawn 
from each new node that connect it to: 1) the nodes above and below on the same lap, and 2) the nodes on the left and right adjacent lap(s). The new edges are added to the existing edge set $\mathcal{E}_{i-1}$ to get $\mathcal{E}'_{i-1}=\mathcal{E}_{i-1}\cup \mathcal{E}_{new}$. In this fashion, by adding new nodes and edges, the RCG is expanded at iteration $i$ into the sampling front $\mathcal{F}_i$ to get $\mathcal{G}'_{i-1}=\{\mathcal{N}'_{i-1},\mathcal{E}'_{i-1}\}$.

\vspace{12pt}
b) \textit{\textbf{RCG Pruning Strategy}}:
Once the RCG is expanded into $\mathcal{F}_i$, it is pruned to maintain a sparse graph structure for computational and memory efficiency. This is done by pruning redundant nodes and edges in the expanded graph $\mathcal{G}'_{i-1}$.

\begin{defn}[\textbf{Essential Node}]\label{define:valid_node}
A node $n \in \mathcal{N}$ is said to be essential if it is 
\begin{itemize}
\item [1.] adjacent to the unknown area, or 
\item [2.] an end node of a lap, or
\item [3.] not an end-node but connected to an end node, say $n_x\in\mathcal{N}$, of an adjacent lap, i.e., $n\in \mathcal{N}_b(n_x)$, and 
\begin{itemize}
\item [a.] $n_x$ has no other neighbor on $n's$ lap, or 
\item [b.] $n_x$ has other neighbors on $n's$ lap which are all non-end nodes but the edge $e(n,n_x)$ connecting $n$ and $n_x$ is closest to an obstacle or the unknown area. 
\end{itemize}
\end{itemize} 
\end{defn}

The first condition above implies that a node is essential if it can enable future connections to the unknown area. The second condition is important to reach the end of a lap and complete its coverage. The third condition is required to enable transition between laps for connectivity and to make sure that the transition edge is closest to the obstacle for maximum coverage of the obstacle-free area. Fig.~\ref{fig:flowchart} shows the flowchart to determine the essential and inessential nodes.

\begin{figure}[t]
    \centering
\includegraphics[width=0.28\textwidth]{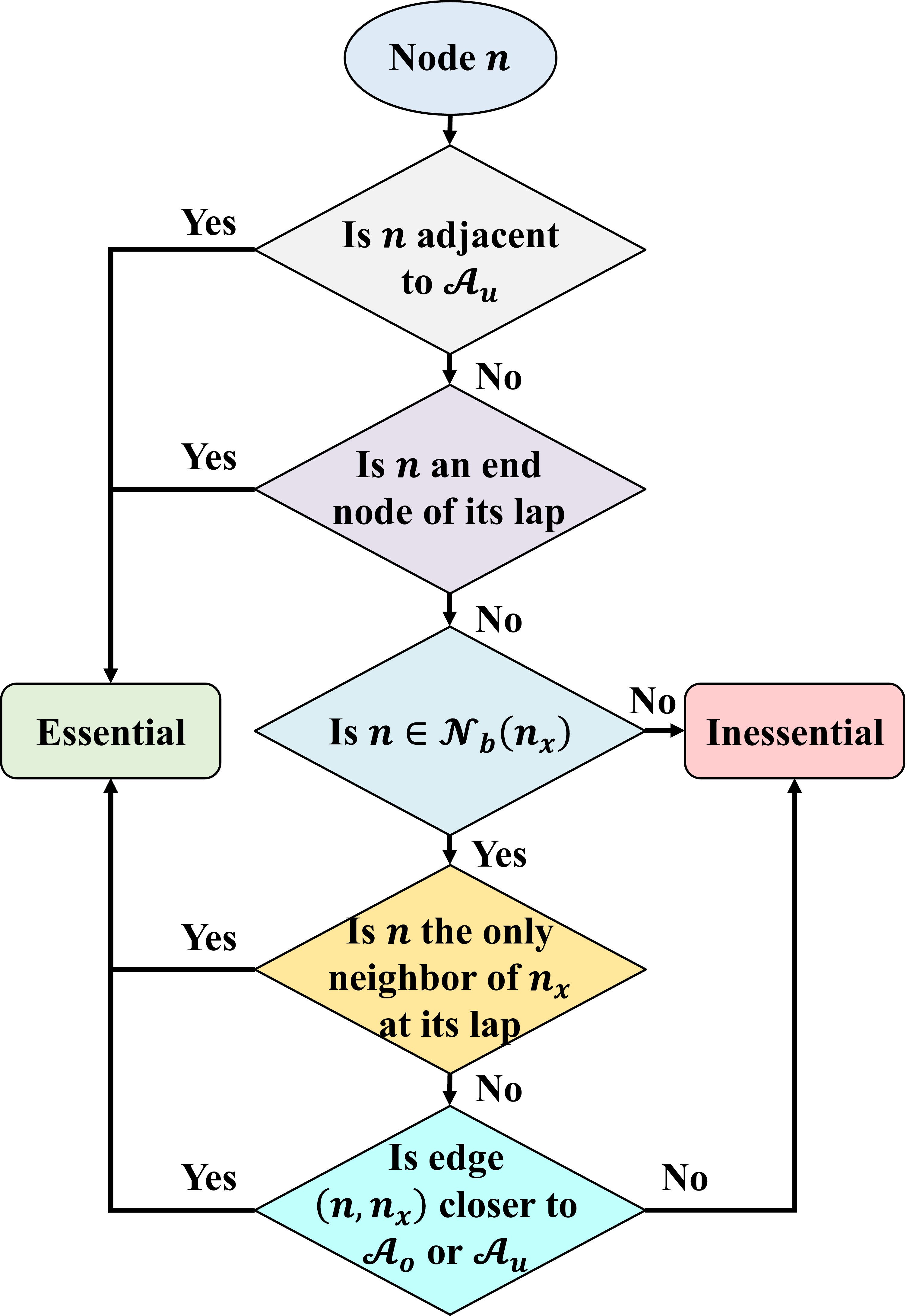}\\
\caption{Flowchart to determine the essential and inessential nodes.} \label{fig:flowchart}
    \vspace{-12pt}
\end{figure}

The node pruning is done as follows. Let $\mathcal{N}_{\partial F_i} \subseteq \mathcal{N}_{i-1}$ be the set of previous nodes that are adjacent to the boundary $\partial F_i$ of the sampling front $\mathcal{F}_i$. For node pruning, each node $n \in \{\mathcal{N}_{new} \cup \mathcal{N}_{\partial F_i}\}$ is checked for the conditions of Def.~\ref{define:valid_node} to find if it is essential or inessential. Then, the set of inessential nodes $\mathcal{N}_{iness}$ is pruned and the remaining nodes are kept in RCG. The next step is to prune the inessential edges. 

\begin{defn}[\textbf{Essential Edge}]\label{define:valid_edge}
An edge $e(n_x,n_y) \in \mathcal{E}$ connecting nodes $n_x$ and $n_y$ is said to be essential if $n_x$ and $n_y$ are both essential nodes and
\begin{itemize}
\item [1.] $n_x$ and $n_y$ are on the same lap, or 
\item [2.] $n_x$ and $n_y$ are on the adjacent  laps, and
\begin{itemize}
\item [a.] $n_x$ and $n_y$ are the end nodes of laps, or
\item [b.] $n_x$ is an end-node and $n_y$ is a non-end node, and 
\begin{itemize}
\item [i] $n_x$ has no other neighbor on $n'_{y}s$ lap or 
\item [ii] $n_x$ has other neighbors on $n'_{y}s$ lap which are all non-end nodes, but the edge $e(n_x,n_y)$ is closest to an obstacle or the unknown area.
\end{itemize}
\end{itemize}
\end{itemize} 
\end{defn}

The edge pruning is done as follows. Let $\mathcal{E}_{\partial F_i} \subseteq \mathcal{E}_{i-1}$ be the set of previous edges connected to $\mathcal{N}_{\partial F_i}$. For edge pruning, each edge $e \in \{\mathcal{E}_{new} \cup \mathcal{E}_{\partial F_i}\}$ is checked for the conditions of Def.~\ref{define:valid_edge} to find if it is essential or inessential. Then, the set of inessential edges $\mathcal{E}_{iness}$ is pruned and the remaining edges are kept in RCG. This is further explained here. First, for each pruned node $n \in \mathcal{N}_{iness}$ the edges connecting it to 
\begin{itemize}
\item  [i.] the two adjacent nodes on its lap are merged to form a single edge, and 
\item [ii.] all nodes on the adjacent laps are pruned. 
\end{itemize}
After the inessential nodes and the corresponding edges are pruned, it is possible that there still exist some inessential edges between essential nodes. This is possible when  some essential nodes have edges connected to more than one essential nodes on an adjacent lap. 
Fig.~\ref{fig:RCG_part2} shows several examples of such cases. In the top-left corner, the top end node of the second lap from the left is connected to two essential nodes on the first lap, which are adjacent to the unknown area. Similar case arises with the top end node of the third lap. Another case arises with the bottom end node of the third lap from the right. This node is connected to three essential nodes in the right adjacent lap. Thus, further pruning of inessential edges is done.
Fig.~\ref{fig:RCG_part3} shows examples of node and edge prunings. Fig.~\ref{fig:RCG_part4} shows the updated RCG after pruning. 

The graph pruning results in the updated RCG, $\mathcal{G}_{i}=\left(\mathcal{N}_{i},\mathcal{E}_{i}\right)$, with a sparse structure that reduces unnecessary computations and unwanted memory usage. It also enables distant (i.e., non-myopic) node selection as waypoints for smooth trajectory generation.

\vspace{6pt}
c) \textit{\textbf{RCG Properties}}:
The RCG has the following properties:

\begin{itemize}

\item [P1.] \textit{Simple}: there is no more than one edge between any two nodes of RCG and no edge loops back on the same node.

\item [P2.] \textit{Planar}: RCG can be embedded on a plane such that no edges cross each other.

\item [P3.] \textit{Connected}: any node of RCG is reachable from any other node by traversing on the  edges. 

\item [P4.] \textit{Essential}: RCG is a minimum sufficient graph consisting of only essential nodes and edges.  

\item [P5.] \textit{Sparse}: the nodes of RCG are distributed only close to the obstacles and unknown area.

\item [P6.] \textit{Scalable}: the size of RCG is $|\mathcal{N}|+|\mathcal{E}| \leq 4|\mathcal{N}|-6.$

\end{itemize}

P1 and P2 directly result from the RCG expansion strategy, where the edges are only drawn from a node to connect the adjacent nodes on the same and side laps. P3 arises from the RCG expansion and pruning strategies, which build and maintain connections between all adjacent nodes on the same lap as well as connections to the ends nodes of adjacent laps, thus making every node reachable from every other node. P4 emerges from the fact that the nodes of RCG are formed from the frontier samples in the sampling front. After each pruning step only essential nodes and edges are left, and it is not possible to add any more frontier samples, thus making the RCG a minimum sufficient graph consisting of essential nodes and edges. Eventually the progressive sampling reaches the free space boundary by exploration until no more sampling front is found and no more frontier samples are created. Therefore, a fully formed RCG is a minimum sufficient graph for complete coverage. P5 appears because the RCG is created by using the frontier samples followed by the pruning process. P6 results from the Euler's Formula~\cite{west1996introduction}, which states that for a simple, connected, and planar graph $|\mathcal{E}| \leq 3|\mathcal{N}|-6$.

\vspace{6pt}
\subsubsection{\textbf{Coverage Trajectory Generation}}\label{Step4:waypointgeneration}
At each iteration $i\in\{0,...m-1\}$, once the robot reaches the waypoint $p_i$, generates the sampling front $\mathcal{F}_i$, and updates the RCG to $\mathcal{G}_{i}=\left(\mathcal{N}_{i},\mathcal{E}_{i}\right)$, the next step is to generate the next waypoint $p_{i+1} \in \mathcal{A}^f$. In order to generate $p_{i+1}$, one of the nodes $\hat{n}_{i+1} \in \mathcal{N}_{i}$ is selected as an intermediate goal node and its position is set as the next waypoint, such that $p_{i+1}=pos({\hat{n}_{i+1}})$, where $pos$ denotes the position of a node. The node selection procedure is described later. First, the node sequence is defined below.

\begin{defn}[\textbf{Node Sequence}] The node sequence $\hat{\mathcal{N}}=\{\hat{n}_0, \hat{n}_1,... \hat{n}_m\}$ is the sequence of intermediate goal nodes $\hat{n}_i \in \mathcal{N}$ for the robot $\mathcal{R}$ to form the coverage trajectory, where $\hat{n}_0$ and $\hat{n}_m$ are its start and end nodes, respectively. 
\end{defn}

The node sequence $\hat{\mathcal{N}}=\{\hat{n}_0, \hat{n}_1,... \hat{n}_m\}$ has one-to-one correspondence to the waypoint sequence $P=\{p_0, p_1...p_m\}$, such that $p_{i}=pos(\hat{n}_i)$, $\forall i\in\{0,...m\}$. The robot follows the waypoint sequence to form the coverage trajectory. In order to select the next goal node, the RCG needs to track the coverage progress by recording which nodes have been visited. Thus, the RCG assigns a state to each node, as follows.

\RestyleAlgo{ruled}
\LinesNumbered
\begin{algorithm}[t]
\footnotesize
Function \textbf{SelectGoalNode}\\
\KwIn{Iteration $i$: current RCG $\mathcal{G}_{i}$; current node $\hat{n}_i$; $\mathcal{N}^{retreat}$.}

\uIf{\ $q(\hat{n}^L_i)= Op$} 
    {
        $\hat{n}_{i+1} = \hat{n}^L_i$; \hspace{0pt} \tcp*[h]{if left node is open, go left}\\
    }
  \uElseIf{\ $q(\hat{n}^U_i)= Op$}
    {
        $\hat{n}_{i+1} = \hat{n}^U_i$; \hspace{0pt} \tcp*[h]{if node above is open, go up}\\
    }
  \uElseIf{\ $q(\hat{n}^D_i)= Op$}
    {
        $\hat{n}_{i+1} = \hat{n}^D_i$; \hspace{0pt} \tcp*[h]{if node down is open, go down} 
    }
  \uElseIf{\ $q(\hat{n}^R_i)= Op$}
    {
        $\hat{n}_{i+1} = \hat{n}^R_i$; \hspace{0pt} \tcp*[h]{if right node is open, go right} 
    }
  \Else(\hspace{30pt} \tcp*[h]{there exists a dead-end at $\hat{n}_{i}$})  
    {   \uIf{$\mathcal{N}_{retreat}\neq \emptyset$ }  
   { $\hat{n}_{i+1} \leftarrow \textbf{EscapeDeadEnd}(\hat{n}_{i},\mathcal{N}_{retreat})$;\\
    }
    \Else {Coverage is complete;}
    }  
     $p_{i+1}$=pos($\hat{n}_{i+1}$);\\
    return $\hat{n}_{i+1}$ and $p_{i+1}$.
\caption{Goal node selection at the current node}
\label{alg:GoalNodeCompute} 
\end{algorithm}
\setlength{\textfloatsep}{12pt}

\RestyleAlgo{ruled}
\LinesNumbered
\begin{algorithm}[t]
\footnotesize
Function \textbf{UpdateState}\\
\KwIn{Iteration $i$: current node $\hat{n}_i$; goal node $\hat{n}_{i+1}$; current node state $q(\hat{n}_i)=Op$; $\mathcal{N}^{Op}_i$; $\mathcal{N}^{Cl}_i$.}
    \If {$q(\hat{n}^U_i)= Cl$ or $q(\hat{n}^D_i)= Cl$}
    {
        $q(\hat{n}_i)=Cl$; \\
       $\mathcal{N}^{Op}_i= \mathcal{N}^{Op}_i\setminus\hat{n}_i$; $\mathcal{N}^{Cl}_i=\mathcal{N}^{Cl}_i\cup\hat{n}_i$;\\
         \If {$\hat{n}_{i+1} = \hat{n}^L_{i}$}
        {
            \If{\ $q(\hat{n}^U_i)=Op$}
            {
                Create a link node above $\hat{n}_i$ at a distance $w$;\\
            }
            \If{\ $q(\hat{n}^D_i)=Op$}
            {
                Create a link node below $\hat{n}_i$ at a distance $w$;\\
            }
        }
    }
    return $q_i(\hat{n}_{i})$, $\mathcal{N}^{Op}_i$, $\mathcal{N}^{Cl}_i$.
\caption{Updating the state of the current node}
\label{alg:StateUpdate} 
\end{algorithm}
\setlength{\textfloatsep}{12pt}

\begin{defn}[\textbf{State Encoding}]
Let $q: \mathcal{N} \rightarrow\left\{Cl,Op\right\}$ be the encoding that assigns each node $n \in \mathcal{N}$ of RCG, a state $q(n) \in \{Cl,Op\}$, where $Cl\; (Op) \equiv Closed\; (Open)$ implies that the node is visited (unvisited) by the robot.
\end{defn}

At iteration $i\in\{0,...m-1\}$, once the robot reaches $\hat{n}_i$ (i.e., waypoint $p_i$) and the RCG is updated to $\mathcal{G}_{i}$, the states of all new nodes in $\mathcal{G}_{i}$ are assigned to be $Op$. Let $\mathcal{N}^{Op}_i \subseteq \mathcal{N}_i$ and $\mathcal{N}^{Cl}_i \subseteq \mathcal{N}_i$ be the sets of \textit{Open} and \textit{Closed} nodes of $\mathcal{G}_i$, respectively, such that $\mathcal{N}^{Op}_i \bigcup \mathcal{N}^{Cl}_i$$=$$\mathcal{N}_i$. Once the node states are updated, the goal node $\hat{n}_{i+1}$ is selected as follows.

\vspace{12pt}
a) \textit{\textbf{Goal Node Selection Strategy}:} \label{goalnodeselect}
Let $\hat{n}_i^L$, $\hat{n}_i^U$, $\hat{n}_i^D$ and $\hat{n}_i^R$ be the neighbors of the current node $\hat{n}_{i}$ on the left lap, the current lap upwards, the current lap downwards, and the right lap, respectively. Note: these directions are defined with respect to a fixed coordinate frame whose vertical axis is parallel to the laps. To obtain $\hat{n}_{i+1}$ (\textbf{Alg.~\ref{alg:GoalNodeCompute} lines 2-9}) an \textit{Open} node is searched from the neighbors of $\hat{n}_{i}$ and selected based on the priority order as follows: $\hat{n}_i^L$ $\rightarrow$ $\hat{n}_i^U$ $\rightarrow$ $\hat{n}_i^D$ $\rightarrow$ $\hat{n}_i^R$. In case there are more than one \textit{Open} neighbors of $\hat{n}_{i}$ on the left or the right adjacent lap, then a random pick is done. 

This simple node selection strategy generates the back-and-forth coverage pattern as follows. At any node, the robot first goes to the leftmost lap where an open node neighbor is available, then it moves on that lap until it finds another open node neighbor on the left or it reaches the end of the lap. If none of the above condition is true then it switches to the right lap and continues the coverage. In this manner, the robot performs the back-and-forth coverage lap by lap while shifting laps from left to right as they are covered.   

\begin{rem} \label{lapcoverage} While moving on a lap, the robot's tasking device covers the region on both sides of the lap, such that the region between any two adjacent laps is covered if 
$w \leq 2r_c$. Thus, the sampling resolution $w$ can be chosen  smaller than $2r_c$ to generate slightly overlapping coverage pattern.
\end{rem}

\vspace{3pt}
b) \textit{\textbf{State Update Strategy}:} 
As soon as $\hat{n}_{i+1}$ is selected, the state of $\hat{n}_i$ is updated as $q(\hat{n}_i)=Cl$ (\textbf{Alg.~\ref{alg:StateUpdate} lines 2-4}). Only in case when both $\hat{n}^U_i \in\mathcal{N}^{Op}_i$ and $\hat{n}^D_i \in\mathcal{N}^{Op}_i$, the state of $\hat{n}_i$ is kept as $q(\hat{n}_i)=Op$. This is to avoid distortion of the back-and-forth trajectory by marking a node as \textit{Closed} in the middle of \textit{Open} nodes on a lap.

If the goal is selected on the left lap, i.e., $\hat{n}_{i+1}=\hat{n}^L_{i}$, and $q(\hat{n}_{i})=Cl$, then it is further checked if either $\hat{n}^U_i \in\mathcal{N}^{Op}_i$ or $\hat{n}^D_i \in\mathcal{N}^{Op}_i$. If this is true then it means that the lap between $\hat{n}_{i}$ and the \textit{Open} neighboring node above or below, if at a distance greater than $w$, is still uncovered. Since $q(\hat{n}_{i})=Cl$, the robot would never travel on this lap. Thus, to cover this lap in the future, a new node, called link node, is created above or below $\hat{n}_{i}$ at a distance of $w$ (\textbf{Alg.~\ref{alg:StateUpdate} lines 5-12}). The link nodes are pruned after they are covered and \textit{Closed}.

\vspace{6pt}
c) \textit{\textbf{Dead-end Escape Strategy}:}
\label{dead-end}
During navigation it is possible that the robot reaches a point where it is unable to find the next goal node because it is surrounded by only obstacles and covered regions. This is known as a dead-end situation. Fig.~\ref{fig:deadEnd_part1} shows an example of a dead-end.

\begin{figure}[t]
    \centering
    \subfloat[Robot reaches a dead-end.]{
        \includegraphics[width=0.42\columnwidth]{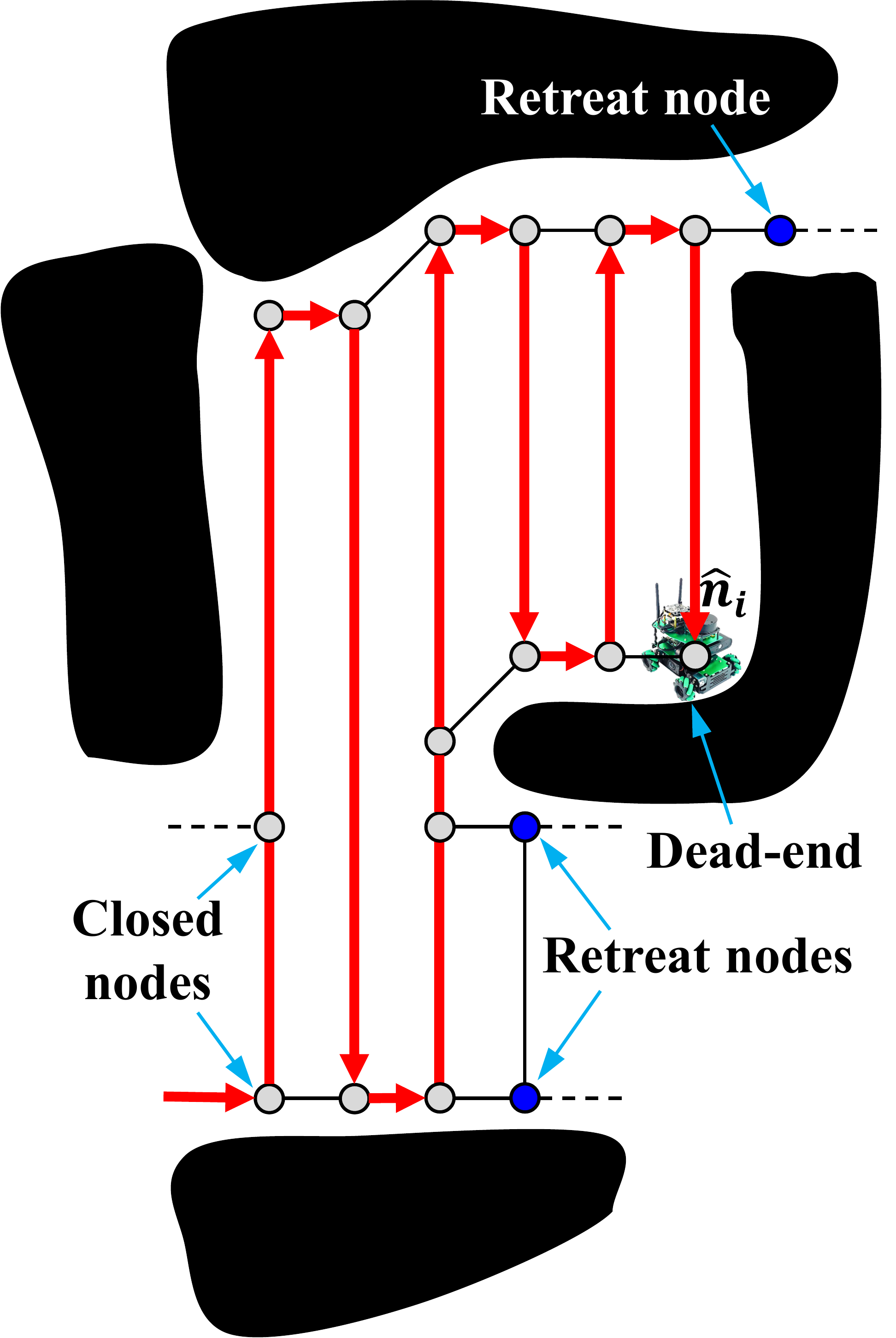}\label{fig:deadEnd_part1}}\hspace{-3pt}\quad
    \centering
    \subfloat[Robot escapes from the dead-end to the nearest retreat node.]{
        \includegraphics[width=0.42\columnwidth]{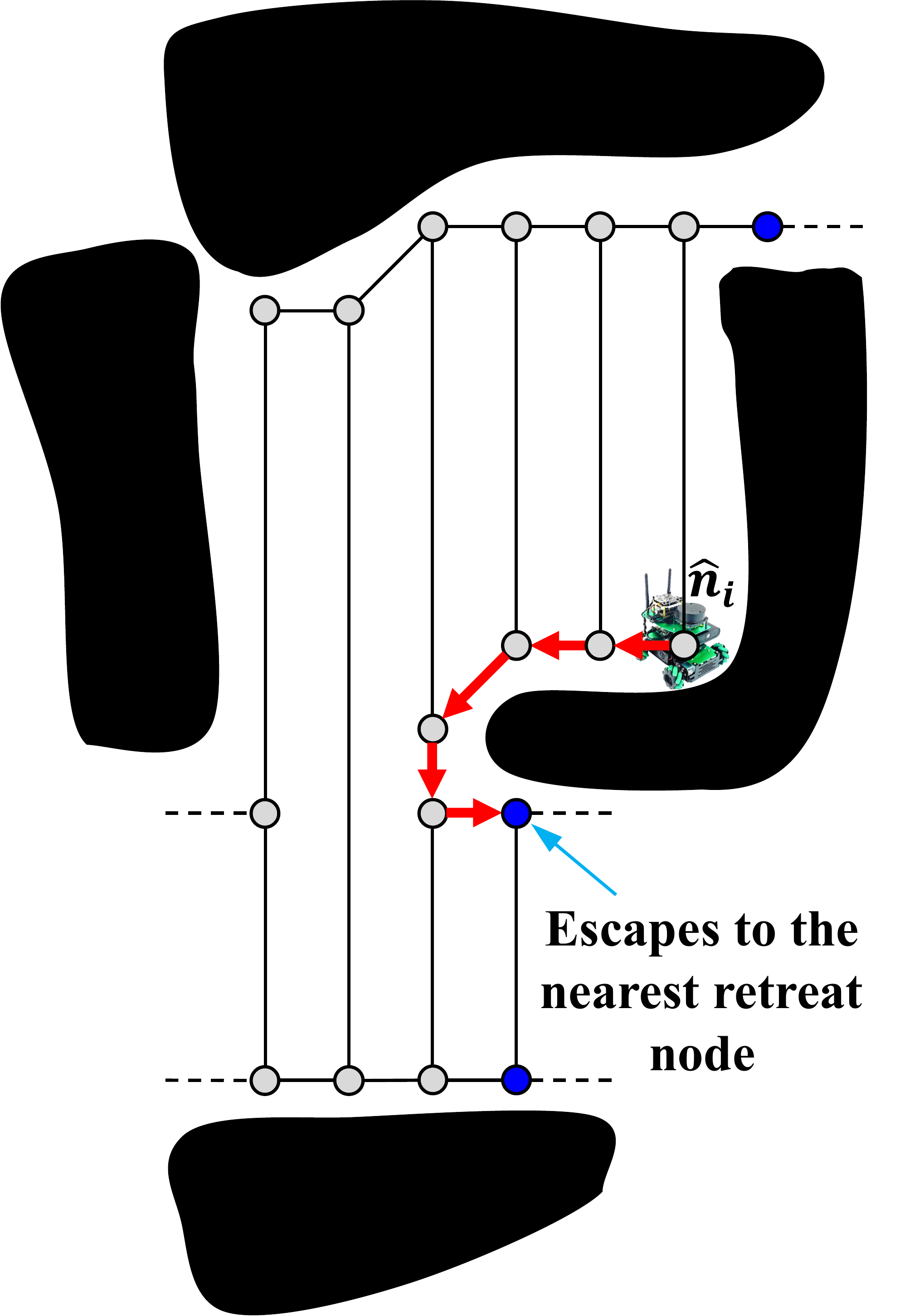}\label{fig:deadEnd_part2}}\\
          \caption{Example of a robot escaping from a dead-end.} \label{fig:deadEnd}
          \vspace{0pt}
\end{figure}

\begin{defn}[\textbf{Dead-end}]
A dead-end is said to occur at node $\hat{n}_i$ if $q(\hat{n}^L_i)=q(\hat{n}^U_i)=q(\hat{n}^D_i)=q(\hat{n}^R_i)=Cl$, which means that there is no Open node in the neighborhood $\mathcal{N}_b(\hat{n}_i)$ that could be selected as the next goal $\hat{n}_{i+1}$. 
\end{defn}

To escape from a dead-end and resume the back-and-forth coverage, an \textit{Open} node needs to be found (\textbf{Alg.~\ref{alg:GoalNodeCompute} line 12}). To select an open node 1) searching the entire graph is computationally inefficient and 2) randomly picking a node in uncovered region could distort the coverage trajectory. As such, C$^*$ uses the concept of retreat nodes as defined below. 

\begin{defn}[\textbf{Retreat Node}]
\label{retreat_node}
A node $n\in\mathcal{N}^{Op}_i$ is defined as a retreat node if it is adjacent to the robot's trajectory.
\end{defn}

\begin{figure*}[t]
    \centering
    \subfloat[The next goal node $\hat{n}_{i+1}$ is determined at the current node $\hat{n}_{i}$.]{
    \includegraphics[width=0.42\columnwidth]{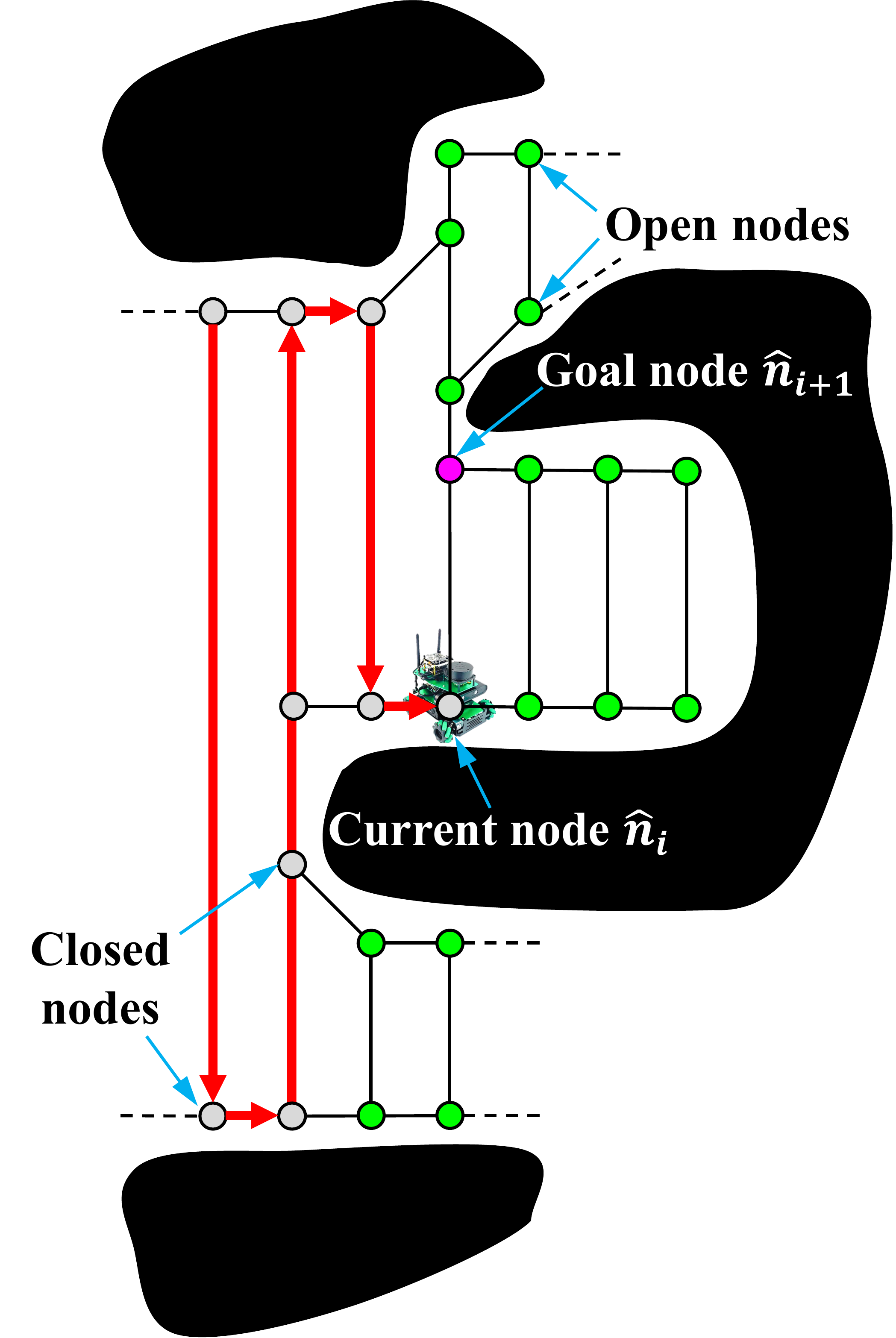}\label{fig:coverageHole_part1}}\hspace{-3pt}\quad
    \centering
    \subfloat[A potential coverage hole is detected near $\hat{n}_{i}$.]{
    \includegraphics[width=0.42\columnwidth]{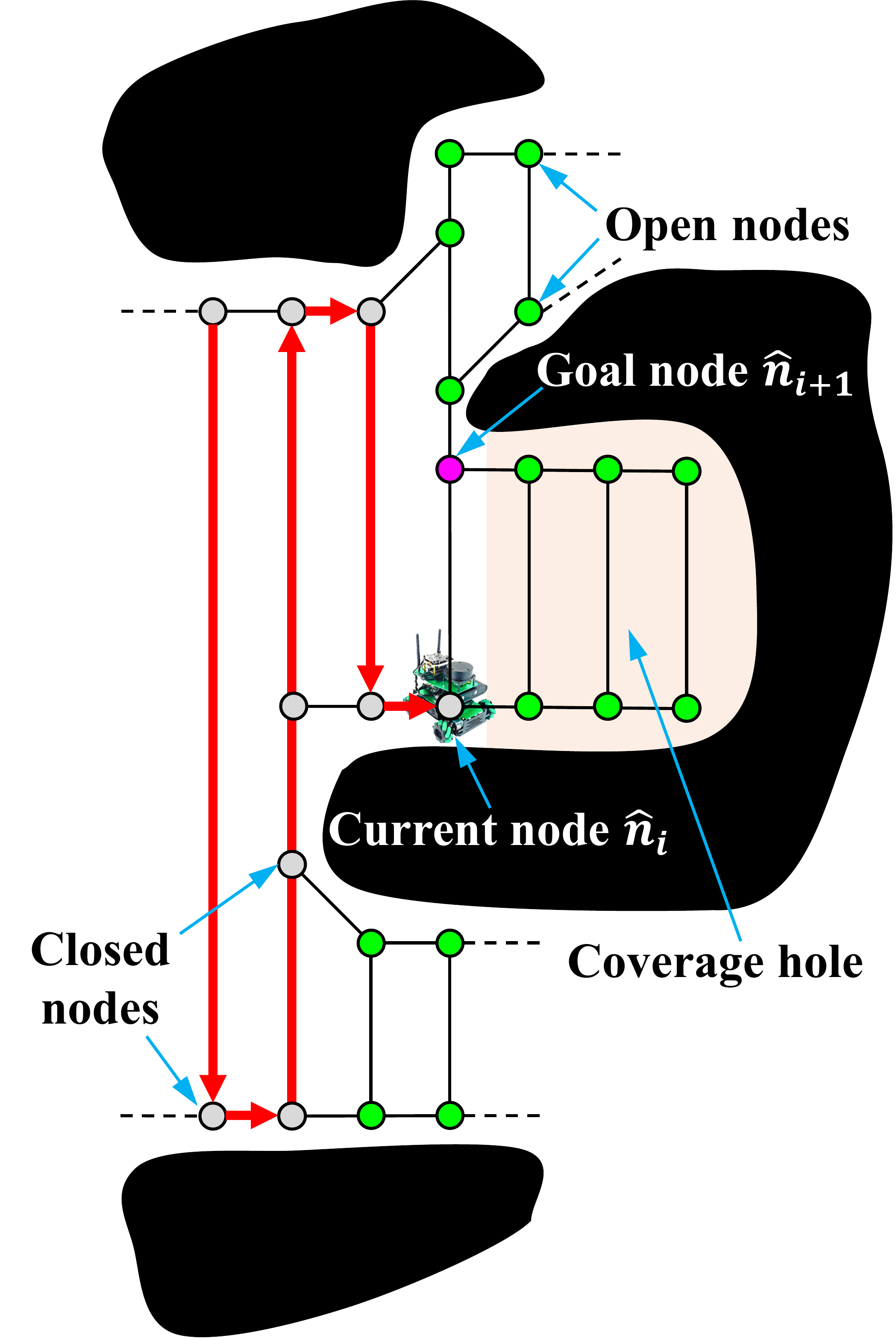}\label{fig:coverageHole_part2}}\hspace{-3pt}\quad
    \centering
    \subfloat[Additional nodes are created in the coverage hole to set up TSP.]{
        \includegraphics[width=0.42\columnwidth]{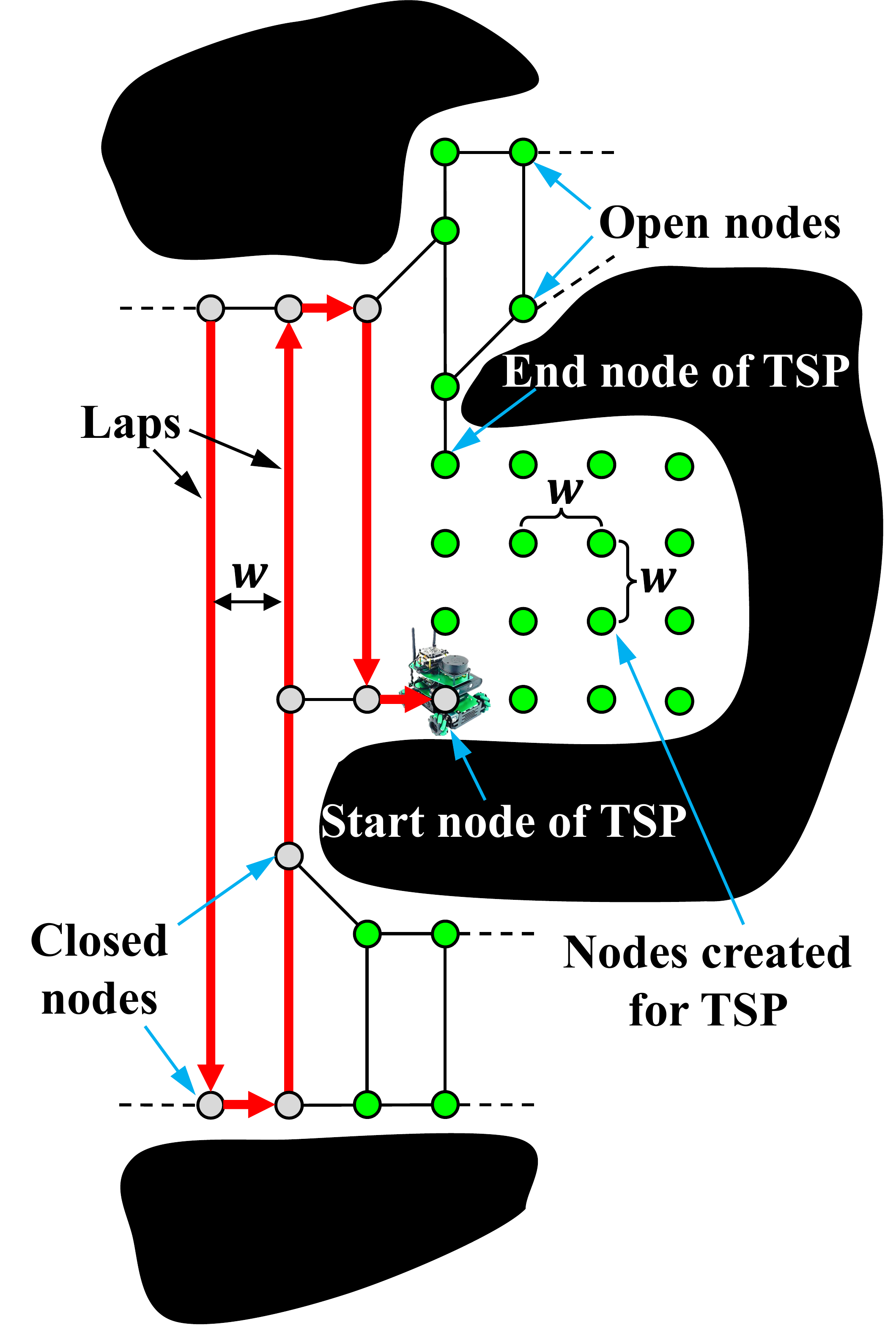}\label{fig:coverageHole_part3}}\hspace{-3pt}\quad
    \centering
    \subfloat[TSP trajectory is executed in the coverage hole.]{
        \includegraphics[width=0.42\columnwidth]{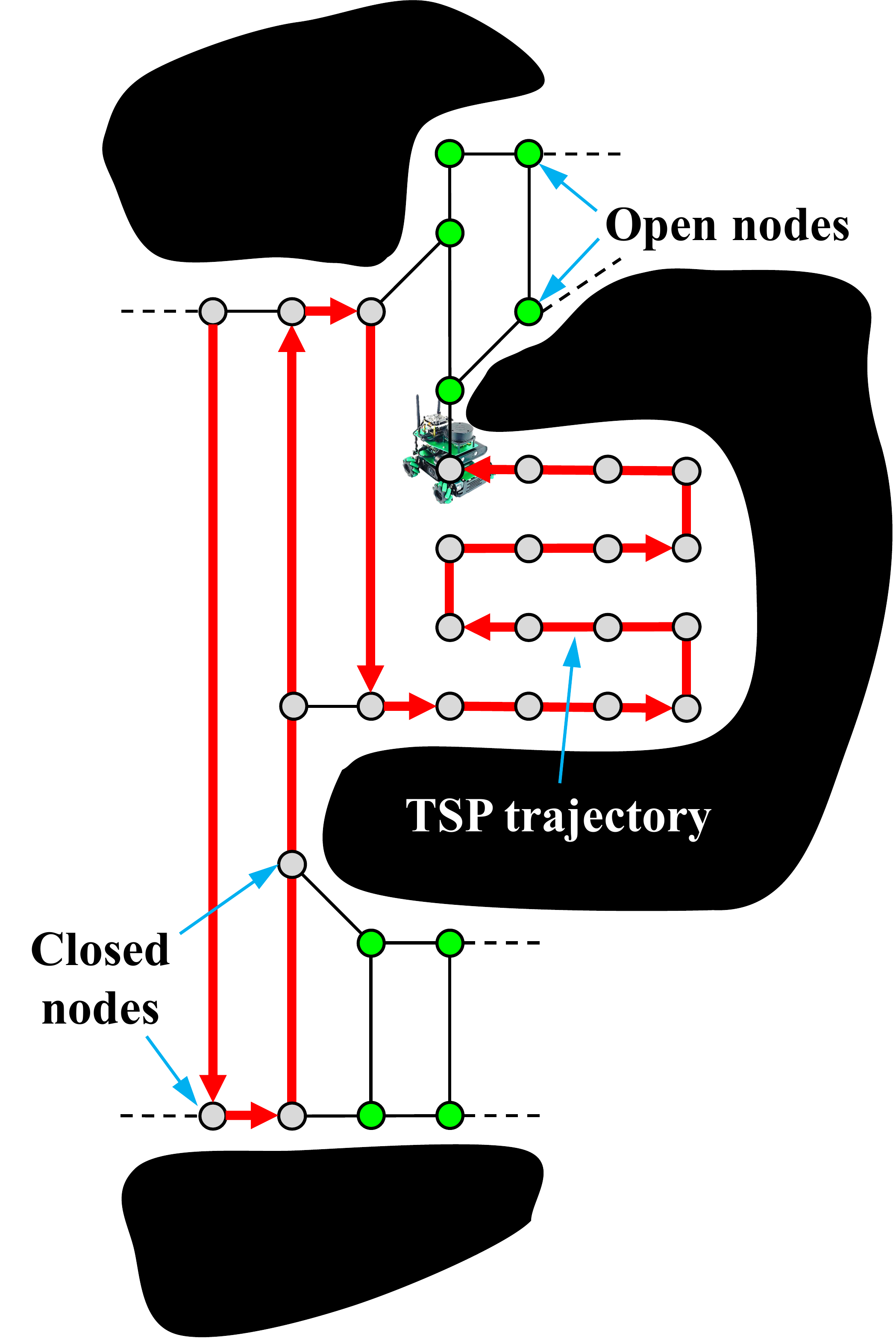}\label{fig:coverageHole_part4}}\\
          \caption{An illustrative example of coverage hole detection and its coverage using the TSP trajectory.}\label{fig:tspcoverage}
          \vspace{3pt}
\end{figure*}

The set $\mathcal{N}_{retreat} \subset \mathcal{N}^{Op}_i$ of retreat nodes is constantly updated during robot navigation. As shown in Fig.~\ref{fig:deadEnd_part1}, this is done by adding any \textit{Open} node to $\mathcal{N}_{retreat}$ that is within a distance of $\sqrt{2}w$ from the robot and removing any \textit{Closed} node. Note that there are always some retreat nodes in the list until coverage is complete because the area is connected. If the robot reaches a dead-end at $\hat{n}_{i}$, then the closest retreat node is chosen as the next goal node $\hat{n}_{i+1}$. Subsequently, the robot moves to $\hat{n}_{i+1}$ on the shortest path obtained by the A$^*$ algorithm~\cite{hart1968formal} and resumes the back-and-forth coverage of the remaining area. Fig.~\ref{fig:deadEnd_part2} shows an example of escaping from a dead-end to the nearest retreat node.

\vspace{9pt}
d) \textit{\textbf{Coverage Holes Prevention Strategy}:}
\label{tsp_coverage}
During navigation it is possible that the robot observes isolated uncovered regions that are surrounded by obstacles and covered regions. If bypassed, these regions form coverage holes, as shown in Fig.~\ref{fig:coverageHole_part1}, which then require long return trajectories to cover them later. This paper presents a proactive coverage hole prevention strategy that consists of: 1) detection of potential coverage holes and 2) coverage of such coverage holes by adaptation to TSP-based locally optimal coverage trajectories.

\begin{defn}[\textbf{Coverage Hole}]
\label{coverage_hole} 
Let $\mathcal{A}^{CH}$$\subseteq$ $\mathcal{A}^{uc}$ be a fully discovered but uncovered region with its boundary denoted as $\partial\mathcal{A}^{CH}$. Then, $\mathcal{A}^{CH}$ forms a coverage hole if $\forall x\in $$\partial\mathcal{A}^{CH}$, $\mathcal{B}(x,w)\setminus \mathcal{A}^{CH}\subseteq \mathcal{A}^o\bigcup\mathcal{A}^{c}$, i.e., it is surrounded by obstacles and covered regions.  
\end{defn}

\vspace{0pt}
$\bullet$ \textit{\textbf{Detection of Potential Coverage Holes}}: In order to detect coverage holes, we use RCG. When the robot reaches node $\hat{n}_i$, it finds the next goal node $\hat{n}_{i+1}$. At this stage, it detects if there are any potential coverage holes near $\hat{n}_i$ that are surrounded by obstacles and covered regions. Let us index a coverage hole around $\hat{n}_i$ with $\alpha \in \mathbb{N}^+$ and denote the set of all nodes within this coverage hole with $\mathcal{N}_{i,\alpha}^{CH}$.

\RestyleAlgo{ruled}
\LinesNumbered
\begin{algorithm}[t]
\footnotesize
Function \textbf{ComputeTSPTrajectory}\\
\KwIn{Iteration $i$: Current node $\hat{n}_i$; Goal node $\hat{n}_{i+1}$; $\mathcal{N}^{CH}_i$.}
 $\overline{\mathcal{N}}^{CH}_i \leftarrow \textbf{AppendNodes}(\mathcal{N}^{CH}_i,\hat{n}_i,\hat{n}_{i+1})$;\\
$n^{TSP}_s=\hat{n}_i$;\\
    \uIf{$(\exists n \in \mathcal{N}_b(\hat{n}_{i+1})$, s.t. $q_i(n)=Op$ and $n \notin \overline{\mathcal{N}}^{CH}_i$) or $\hat{n}_{i+1}$ is adjacent to unknown area}
    {
        $n^{TSP}_e=\hat{n}_{i+1}$; \tcp*[h]{condition 1}\\
    }
    \uElseIf{$\exists n \in \mathcal{N}_b(\hat{n}_{i})$, s.t. $q_i(n)=Op$ and $n \notin \overline{\mathcal{N}}^{CH}_i$} 
    {
         $n^{TSP}_e=\hat{n}_i$; \tcp*[h]{condition 2}\\
    }
    \Else
    {
        $n^{TSP}_e$ is not specified; \tcp*[h]{condition 3}\\
    }
            $\Lambda^{TSP}_i \leftarrow \textbf{SolveTSP}(\overline{\mathcal{N}}^{CH}_i, n_s^{TSP}, n_e^{TSP})$;\\
    return $\Lambda^{TSP}_i$.
\caption{Computation of the TSP trajectory}
\label{alg:ComputeTSP} 
\end{algorithm}
\setlength{\textfloatsep}{12pt}

The process of finding all such coverage holes near $\hat{n}_i$ is as follows. First, an \textit{Open} node is picked up in the neighborhood of $\hat{n}_i$ and assigned an index $\alpha=1$. Then, starting at this node, the neighboring \textit{Open} nodes are recursively searched and labeled with index $\alpha=1$. This is done using the floodfill function. The search continues to expand in all directions until reaching \textit{Closed} nodes or the goal node $\hat{n}_{i+1}$. Once the search is complete then it is further checked if any \textit{Open} node with index $\alpha=1$ is adjacent to the unknown area. If true, then the set of nodes with index $\alpha=1$ is disqualified to form a coverage hole. If not, then it forms a coverage hole. Fig. \ref{fig:coverageHole_part2} shows an example where a coverage hole is detected near $\hat{n}_i$. The above search is repeated for any other unlabeled \textit{Open} neighbor of $\hat{n}_i$ and labeling it and its \textit{Open} neighbors with index $\alpha=2$ until another coverage hole is detected and so on. Note that it is possible to have more than one coverage hole adjacent to $\hat{n}_i$. The search ends when all coverage holes near $\hat{n}_i$ are detected. Then  $\mathcal{N}^{CH}_i=\bigcup_{\alpha}\mathcal{N}_{i,\alpha}^{CH}$ denotes the set of all coverage holes around $\hat{n}_i$.

 \vspace{6pt}
$\bullet$ \textit{\textbf{Set up TSP}}: Once $\mathcal{N}^{CH}_i$ is discovered, it is covered using a locally optimal TSP trajectory $\Lambda^{TSP}_i$. To compute $\Lambda^{TSP}_i$, the TSP is set up as follows. Since RCG is built with sparse nodes, it is important to create additional nodes in $\mathcal{N}^{CH}_i$ to compute $\Lambda^{TSP}_i$. Thus, nodes are added on each lap inside a coverage hole with the inter-node distance equal to the sampling resolution $w$ (\textbf{Alg.~\ref{alg:ComputeTSP} line 2}). Fig. \ref{fig:coverageHole_part3} shows an example of node addition inside a coverage hole to compute $\Lambda^{TSP}_i$. These nodes are then connected to form a fully connected sub-graph $\overline{\mathcal{G}}^{CH}_i=(\overline{\mathcal{N}}^{CH}_i,\overline{\mathcal{E}}^{CH}_i)$. The node set $\overline{\mathcal{N}}^{CH}_i$ includes the current node $\hat{n}_{i}$, the goal node $\hat{n}_{i+1}$, the coverage hole nodes $\mathcal{N}^{CH}_i$ and the appended nodes.   
Each edge in $\overline{\mathcal{E}}^{CH}_i$ is assigned the  A$^*$~\cite{hart1968formal} cost between its end nodes.

The start node of the TSP trajectory is set to $n^{TSP}_s=\hat{n}_{i}$ (\textbf{Alg.~\ref{alg:ComputeTSP} line 3}), while its end node $n^{TSP}_e$ is determined as follows. Condition 1: If $\hat{n}_{i+1}$ is connected to an \textit{Open} node which is not in the coverage hole then $n^{TSP}_e=\hat{n}_{i+1}$ (\textbf{Alg.~\ref{alg:ComputeTSP} lines 4-5}). This allows for synergistic merging of $\Lambda^{TSP}_i$ to the back-and-forth trajectory through $\hat{n}_{i+1}$. The above also applies when $\hat{n}_{i+1}$ is adjacent to unknown area; i.e., 
$\mathcal{B}(\hat{n}_{i+1},w)\cap\mathcal{A}^u\neq \emptyset$. Condition 2: Else, if $\hat{n}_{i}$ is connected to an \textit{Open} node which is not in the coverage hole then $n^{TSP}_e=\hat{n}_{i}$ (\textbf{Alg.~\ref{alg:ComputeTSP} lines 6-7}). This allows for synergistic merging of $\Lambda^{TSP}_i$ to the back-and-forth trajectory through $\hat{n}_{i}$. Note that $\hat{n}_{i}$ cannot be adjacent to unknown area because it is visited by the robot. Condition 3: If none of the above conditions are true, then  $n^{TSP}_e$ is not specified (\textbf{Alg.~\ref{alg:ComputeTSP} lines 8-9}).

\vspace{6pt}
$\bullet$ \textit{\textbf{Solve TSP}}: Once the TSP is set up it is solved (\textbf{Alg.~\ref{alg:ComputeTSP} line 11}) in an efficient manner to get the TSP trajectory $\Lambda^{TSP}_i$. For conditions $1$ and $3$ above (\textbf{Alg.~\ref{alg:ComputeTSP} lines 4-5 and 8-9}, respectively), since $n^{TSP}_e \neq n^{TSP}_s$, a dummy node $\eta$ is added to convert the problem into a closed-loop TSP that starts and ends at $\eta$. The edge cost between $\eta$ and $\hat{n}_i$ is set as  $u_{\eta,\hat{n}_i}=0$. For condition 1, the edge cost between $\eta$ and $\hat{n}_{i+1}$ is set as  $u_{\eta,\hat{n}_{i+1}}=0$, while for condition 3, it is set as $u_{\eta,\hat{n}_{i+1}}=\infty$. The edge costs between $\eta$ and all the remaining nodes is set as $u_{\eta,n}=\infty, \forall n\in \overline{\mathcal{N}}^{CH}_i\setminus \{\hat{n}_i, \hat{n}_{i+1}\}$. Note that $\eta$ is removed once the optimized node sequence $\Lambda^{TSP}_i$ is obtained. For condition 2 (\textbf{Alg.~\ref{alg:ComputeTSP} lines 6-7}), since $n^{TSP}_e = n^{TSP}_s$, the TSP is closed-loop; thus no dummy node is needed.

A heuristic approach is used to solve for TSP as it is NP hard~\cite{aarts2003}. First, the nearest neighbor algorithm~\cite{aarts2003} is used to obtain an initial node sequence. Then, the 2-opt algorithm~\cite{aarts2003} is used to improve this initial solution and obtain the optimized node sequence. Fig. \ref{fig:coverageHole_part4} shows an example of the TSP trajectory executed inside a coverage hole.

\section{Algorithm Analysis}
\label{sec:analysis}

Algorithm \ref{alg:Cstar} shows the iterative steps of C$^*$. 
The complexity of C$^*$ is evaluated for each iteration $i$ when the robot reaches a node $\hat{n}_i$. Since the environment mapping is done continuously during robot navigation, the complexity of C$^*$ is evaluated for the steps of progressive sampling, RCG expansion, RCG pruning, goal node selection and RCG state updating. The complexities of dead-end escape and coverage holes prevention strategies are evaluated separately.

\begin{lem}
    The complexity of sampling is $O(|\mathcal{S}_{\mathcal{F}_i}|)$. 
    \label{lem:sampling}
\end{lem}

\begin{proof}  
Let $\mathcal{S}_{\mathcal{F}_i}$ be the set of all possible samples that could be placed uniformly in the sampling front $\mathcal{F}_i$ with the sampling resolution of $w$. In order to pick the set of frontier samples $\mathcal{S}_i \subseteq \mathcal{S}_{\mathcal{F}_i}$, it is required to check 
the ball $\mathcal{B}(s,w)$ around each potential sample $s \in \mathcal{S}_{\mathcal{F}_i}$ to verify if $s$ is adjacent to unknown area and/or obstacle. Thus, the complexity of progressive sampling is
$O(|\mathcal{S}_{\mathcal{F}_i}|)$.  
\end{proof}

\begin{rem}$\label{rem_sampres} |\mathcal{S}_{\mathcal{F}_i}| < \left\lfloor \frac{2r_d}{w} \right\rfloor \left \lfloor \frac{\ell_i^{max}}{w} \right\rfloor$, where  $\lfloor\rfloor$ is the floor operation, $\left\lfloor \frac{2r_d}{w} \right\rfloor$ is the maximum number of laps in a sampling front, and $\left \lfloor \frac{\ell_i^{max}}{w} \right\rfloor$ is the maximum number of samples on the longest lap in $\mathcal{F}_i$ with length $\ell_i^{max} \in \mathbb{R}^+$. Note that since $\mathcal{F}_i$ is generated by sensing during navigation from $p_{i-1}$ to $p_{i}$, $\ell^{max}_i$ could be greater than $2r_d$.
\end{rem}

\vspace{-6pt}
\begin{lem}
    The complexity of RCG expansion is $O(|\mathcal{S}_i|)$. 
    \label{lem:expansion}
\end{lem}
\begin{proof}  
First, the new nodes $\mathcal{N}_{new}$ of RCG are created from the frontier samples $\mathcal{S}_i$ and added to its node set $\mathcal{N}_{i-1}$. This has the complexity of $O(|\mathcal{N}_{new}|)$. Then, the new edges are created and added to the edge set $\mathcal{E}_{i-1}$ by connecting the new nodes to each other and to the neighboring nodes of existing RCG. For each new node, new edges are drawn to connect it to at most $2$ adjacent nodes on the same lap and at most $3$ nodes on each adjacent lap that are within a distance of $\sqrt{2}w$. This has the complexity of $O(8(|\mathcal{N}_{new}|))$. Thus, the complexity of RCG expansion is $O(9|\mathcal{N}_{new}|)= O(|\mathcal{N}_{new}|)=O(|\mathcal{S}_{i}|)$.
\end{proof}

\RestyleAlgo{ruled}
\LinesNumbered
\begin{algorithm}[t]
\footnotesize
\KwIn{Iteration $i$: starting location $p_{i-1}$; existing RCG $\mathcal{G}_{i-1}$; 
      next waypoint $p_{i}$; goal node $\hat{n}_{i}$.}

\While{$\mathcal{N}^{\text{Op}}_{i-1} \neq \emptyset$}{
    \tcp{Navigate from $p_{i-1}$ to $p_i$ and discover the environment}
    $\mathcal{A}_{d,i} \leftarrow \textbf{NavigateSense}(p_{i-1},p_i)$; \\

    \If{$Reached(p_i)$}{
        \tcp{Create a sampling front at $p_i$}
        $\mathcal{F}_{i} \leftarrow \textbf{CreateSamplingFront}(\mathcal{A}_{d,i})$; \\
    
        \tcp{Generate frontier samples on laps of $\mathcal{F}_{i}$}
        $\mathcal{S}_{i} \leftarrow \textbf{GenerateFrontierSamples}(\mathcal{F}_{i})$; \\
    
        \tcp{Assign $\mathcal{S}_i$ as $\mathcal{N}_{\text{new}}$ and add connecting edges to $\mathcal{G}_{i-1}$}
        $\mathcal{G}'_{i-1} \leftarrow \textbf{ExpandRCG}(\mathcal{G}_{i-1},\mathcal{S}_{i})$; \\
    
        \tcp{Prune inessential nodes and edges}
        $\mathcal{G}_{i} \leftarrow \textbf{PruneRCG}(\mathcal{G}'_{i-1})$; \\
    
        \tcp{Select goal node or escape dead-end (Alg.~\ref{alg:GoalNodeCompute})}
        $(\hat{n}_{i+1},p_{i+1}) \leftarrow \textbf{SelectGoalNode}(\mathcal{G}_{i},\hat{n}_{i})$; \\
    
        \tcp{Update the state of $\hat{n}_{i}$ (Alg.~\ref{alg:StateUpdate})}
        $(q_i(\hat{n}_{i}),\mathcal{N}^{\text{Op}}_{i},\mathcal{N}^{\text{Cl}}_{i}) 
          \leftarrow \textbf{UpdateState}(q_i(\hat{n}_{i}),\mathcal{N}^{\text{Op}}_{i},\mathcal{N}^{\text{Cl}}_{i})$; \\
    
        \tcp{Detect coverage holes}
        $\mathcal{N}^{\text{CH}}_i \leftarrow \textbf{DetectCoverageHoles}(\mathcal{G}_{i}, \hat{n}_{i}, \hat{n}_{i+1})$; \\
    
        \If{$\mathcal{N}^{\text{CH}}_i \neq \emptyset$}{
            \tcp{Plan and execute TSP trajectory}
            $\Lambda^{\text{TSP}}_i \leftarrow \textbf{ComputeTSPTrajectory}(\mathcal{N}^{\text{CH}}_i,\hat{n}_{i},\hat{n}_{i+1})$; \\
            $n^{\text{TSP}}_e \leftarrow \textbf{ExecuteTSPTrajectory}(\Lambda^{\text{TSP}}_i)$ \\
            $\hat{n}_{i} \leftarrow n^{\text{TSP}}_e$; $\hat{n}_{i+1} \leftarrow n^{\text{TSP}}_e$; \\
        }
        $i \leftarrow i+1$; \\
    }
}
\caption{C$^*$}
\label{alg:Cstar}
\end{algorithm}
\setlength{\textfloatsep}{18pt}

\vspace{-3pt}
\begin{lem}
    The complexity of node pruning is $O(|\mathcal{S}_i|)$. 
    \label{lem:nodeprune}
\end{lem}
\begin{proof}  
For each node $n \in \{\mathcal{N}_{new} \cup \mathcal{N}_{\partial F_i}\}$, we follow Def.~\ref{define:valid_node} to check if its essential or inessential. Condition 1 of Def.~\ref{define:valid_node} checks 
around node $n$ to verify if it is adjacent to the unknown area. Condition 2 checks
if $n$ is an end node of its lap. Condition 3 checks at most $6$ neighbors of $n$ on adjacent laps to verify if it is connected to an end node $n_x$ of the adjacent lap. Condition 3a checks if $n_x$ has other neighbor(s) on $n's$ lap. Condition 3b checks if the other neighbors (at most $2$) of $n_x$ on $n's$ lap are non-end nodes. Condition 3b further checks if the edge $e(n,n_x)$ connecting $n$ and $n_x$ is closest to an obstacle or the unknown area. To check this, $k$ intermediate points are created on the edge and their neighborhoods are checked to compute the distance to the nearest obstacle. Based on the above procedure,  the set of inessential
nodes $\mathcal{N}_{iness}$ is determined and pruned from RCG. Thus, a fixed number of computations are done for each node $n \in \{\mathcal{N}_{new} \cup \mathcal{N}_{\partial F_i}\}$. Thus, the node pruning has a complexity of $O(|\mathcal{N}_{new}|+|\mathcal{N}_{\partial F_i}|)$. Typically, $|\mathcal{N}_{\partial F_i}|\leq |\mathcal{N}_{new}|=|\mathcal{S}_{i}|$, thus, the node pruning complexity is $\sim O(|\mathcal{S}_i|)$.
\end{proof}

\begin{lem}
The complexity of edge pruning is $O(|\mathcal{S}_i|)$.
    \label{lem:edgeprune}
\end{lem}
\begin{proof}  
 For each edge $e(n_x,n_y) \in \{\mathcal{E}_{new}\cup\mathcal{E}_{\partial \mathcal{F}_i}\}$, we follow Def.~\ref{define:valid_edge} to check if its essential or inessential. First, for each pruned node $n \in \mathcal{N}_{iness}$, the edges connecting it to the two adjacent nodes on the same lap are merged and the edges connecting it to the nodes on adjacent laps are pruned. This satisfies Conditions 1 and 2a of Def.~\ref{define:valid_edge}. Then, Condition 2b-i checks if $n_x$ has no other neighbor on $n'_{y}s$ lap. Condition 2b-ii checks if all other neighbors on $n'_{y}s$ lap are non-end nodes. It further checks if the edge $e(n_x,n_y)$ is closest to an obstacle or the unknown area. To check this, $k$ intermediate points are created on $e(n_x,n_y)$ and their neighborhoods are checked to compute the distance to the nearest obstacle. Based on the above procedure, the inessential edges are pruned from RCG. Thus, a fixed number of computations are done for each edge $e(n_x,n_y) \in \{\mathcal{E}_{new}\cup\mathcal{E}_{\partial \mathcal{F}_i}\}$. Thus, the edge pruning has a complexity of
$O(|\mathcal{E}_{new}|+|\mathcal{E}_{\partial \mathcal{F}_i}|)$. Typically, $|\mathcal{E}_{\partial F_i}|\leq |\mathcal{E}_{new}| \leq 3(|\mathcal{N}_{new}|+|\mathcal{N}_{\partial \mathcal{F}_i}|)-6$, and $|\mathcal{N}_{\partial F_i}|\leq |\mathcal{N}_{new}|=|\mathcal{S}_{i}|$, thus, the edge pruning complexity is $\sim O(|\mathcal{S}_i|)$.
\end{proof}

\begin{lem}
    The complexity of goal node selection is $O(1)$. 
    \label{lem:goalnodeselect}\vspace{-6pt}
\end{lem}

\begin{proof}  
To select the goal node at $\hat{n}_i$, an open node is searched in the neighborhood of $\hat{n}_{i}$ based on the priority order of 1) $\hat{n}_i^L$, 2) $\hat{n}_i^U$, 3) $\hat{n}_i^D$, and 4) $\hat{n}_i^R$. This has a complexity of $O(1)$.
\end{proof}

\begin{lem}
    The complexity of node state updating is $O(1)$. 
    \label{lem:stateupdate}\vspace{-6pt}
\end{lem}
\begin{proof}  
Once the goal node $\hat{n}_{i+1}$ is selected, the algorithm checks if the nodes adjacent to the current node $\hat{n}_i$ on the same lap are both \textit{Open}. If yes, then the state of  $\hat{n}_i$ is kept as \textit{Open}; otherwise, it is updated to \textit{Closed}. Further, if $\hat{n}_{i+1}$ is selected on the left lap and the state of $\hat{n}_i$ is \textit{Closed}, then it is checked if either $\hat{n}^U_i$ or $\hat{n}^D_i$ if farther than $w$ is \textit{Open}. If this is true, then a link node is created above or below $\hat{n}_i$ at a distance of $w$. The above process requires a fixed number of operations, thus the complexity of state updating is $O(1)$.
\end{proof}

\begin{thm}
The complexity of C$^*$ is $O(|\mathcal{S}_{\mathcal{F}_i}|)$.
    \label{thm:complexity}
\end{thm}
\begin{proof}
From Lemmas~\ref{lem:sampling}-\ref{lem:stateupdate}, the complexity of C$^*$ is  $O(|\mathcal{S}_{\mathcal{F}_i}|+3|\mathcal{S}_i|+1+1)$. Since $\mathcal{S}_i \subseteq \mathcal{S}_{\mathcal{F}_i}$, the overall complexity of C$^*$ is $O(|\mathcal{S}_{\mathcal{F}_i}|)$. Note that from Remark~\ref{rem_sampres}, the complexity increases if the sampling resolution $w\in \mathbb{R}^+$ is decreased. 
\end{proof}

Since the dead-ends and coverage holes do not occur at all iterations, the complexities of dead-end escape and coverage hole prevention strategies are evaluated separately. 

\vspace{6pt}
\begin{lem}
    The complexity of dead-end escape strategy is $O(|\mathcal{N}_i|log|\mathcal{N}_i|)$.
    \label{lem:deadend}
\end{lem}

\begin{proof}  
To escape from a dead-end, the nearest retreat node with the shortest path is selected from  $\mathcal{N}_{retreat}$. Thus, the path length to each retreat node is calculated using A* and then the node with the shortest path length is selected. Since A* has a complexity of $O(|\mathcal{E}_i|log|\mathcal{N}_i|)$, the overall complexity of this process is $O(|\mathcal{N}_{retreat}|(|\mathcal{E}_i|log|\mathcal{N}_i|+1))$. Since $|\mathcal{E}_i| \leq 3|\mathcal{N}_i|-6$ and usually $|\mathcal{N}_{retreat}|<<  |\mathcal{N}_i|$,   
the complexity of dead-end escape strategy is $O(|\mathcal{N}_i|log|\mathcal{N}_i|)$.
\end{proof}

\vspace{6pt}
\begin{lem} The complexity of coverage hole prevention strategy is $O(|\mathcal{N}_{i}|)$.
    \label{lem:coverageholeprev}
\end{lem}
\begin{proof}  
To find any coverage holes around node $\hat{n}_i$, a recursive search is conducted for {\textit Open nodes} in the neighborhood. This has the worst case complexity of $O(|\mathcal{N}_i|)$; however, in practice it is significantly smaller. If a coverage hole is detected, then a fully connected sub-graph $\overline{\mathcal{G}}^{CH}_i=(\overline{\mathcal{N}}^{CH}_i,\overline{\mathcal{E}}^{CH}_i)$ 
is formed in the coverage hole and a TSP 
is formulated. An initial solution of TSP is obtained using a nearest-neighbor based heuristic approach which has a complexity of $O\left(|\overline{\mathcal{N}}^{CH}_i|^2\right)$~\cite{aarts2003}. Then, the 2-Opt greedy algorithm is applied to improve the solution which has a complexity of $O\left(|\overline{\mathcal{N}}^{CH}_i|\right)$~\cite{aarts2003}. Typically, $|\overline{\mathcal{N}^{CH}_{i}}| \ll |\mathcal{N}_i|$, thus the complexity of coverage hole prevention is $O(|\mathcal{N}_{i}|)$. 
\end{proof}

\vspace{6pt}
\begin{thm}[Space Complexity]
\label{thm:space_complexity}
The space complexity of C$^*$ is $O(|\mathcal{N}_i|)$.
\end{thm}

\begin{proof}
The space complexity is the amount of memory used to store the RCG $\mathcal{G}_i\left(\mathcal{N}_i,\mathcal{E}_i\right)$. The size of $\mathcal{G}_i$ is $|\mathcal{N}_i|+|\mathcal{E}_i|$. Since $|\mathcal{E}_i| \leq 3|\mathcal{N}_i|-6$, the space complexity is $O(|\mathcal{N}_i|)$.
\end{proof}

\begin{figure*}[t]
    \centering
    \subfloat[Coverage trajectories generated by different algorithms.]{
        \includegraphics[width=0.9\textwidth]{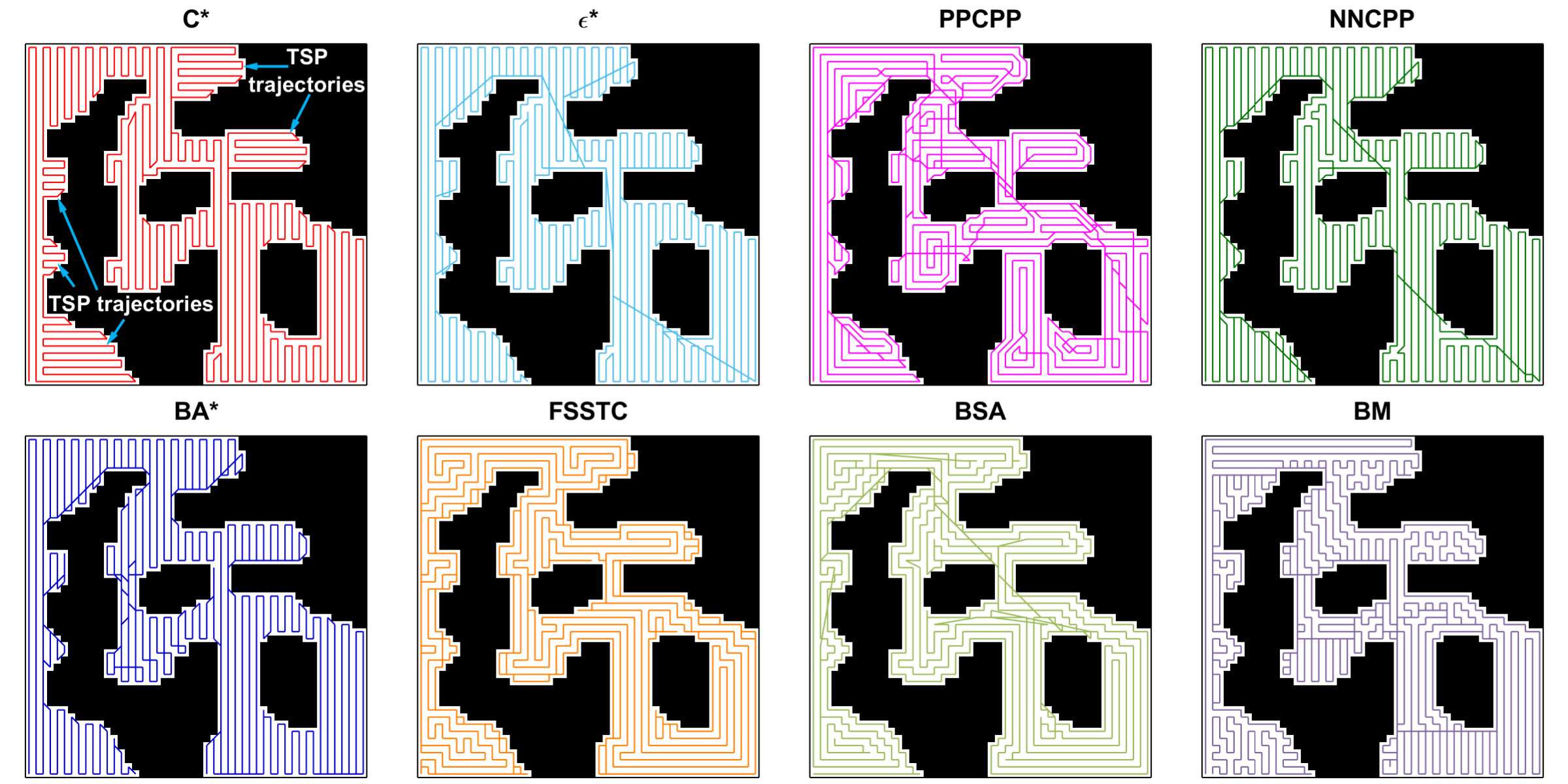}\label{fig:path_scene1}}\vspace{10pt}\\
         \centering
    \subfloat[Comparison of performance metrics of different algorithms.]{
         \includegraphics[width=0.9\textwidth]{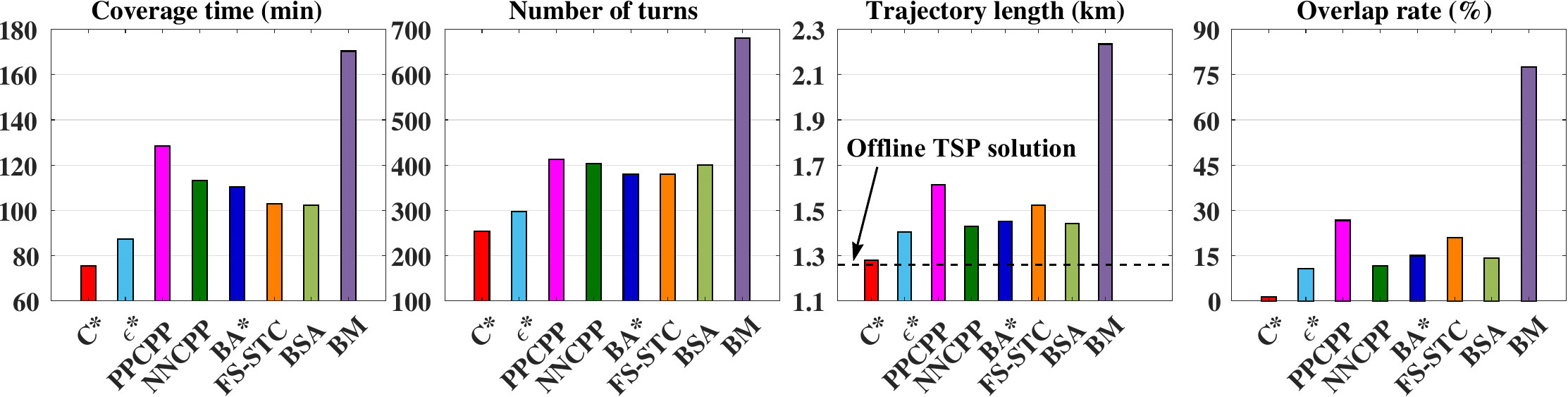}\label{fig:performance_scene1}}\\
          \caption{\textbf{Scenario 1:} Performance comparison of C$^*$ with the baseline algorithms. Video available in supplementary documents.}\label{fig:comparison_scene1}
          \vspace{-0pt}
\end{figure*}

\begin{thm}[Complete Coverage]
\label{complete_cover}
C$^*$ achieves complete coverage of the obstacle-free space.
\end{thm}

\begin{proof}
The proof follows from the RCG construction, goal node selection and dead-end escape strategies. 
\begin{itemize}
\item First, the connectivity of obstacle-free space and progressive sampling at all the frontier nodes adjacent to the unknown area imply that RCG will incrementally grow and eventually spread in all of the free space. 
\item Second, RCG is a minimum sufficient graph (P4) constructed from the frontier samples. This implies that the nodes of a fully grown RCG will lie within a ball of radius $w$ to the obstacles or the boundary of the free space. Thus, a fully grown RCG will wrap tightly around the boundary of the obstacle-free space.   
\item Third, the goal node selection strategy in Section \ref{Step4:waypointgeneration}a ensures that the robot performs the back-and-forth coverage lap by lap while shifting laps from left to right as they are covered. From Remark~\ref{lapcoverage} it is guaranteed that the region between any two adjacent laps is covered if $w \leq 2r_c$. By connectivity of RCG (P3) it is further guaranteed that each lap is covered. The TSP-based coverage of potential coverage holes further optimizes the coverage trajectory in local regions surrounded by obstacles before resuming the lap by lap coverage.
\item Finally, the coverage continues until all nodes of RCG are visited or a dead-end is reached. The robot escapes from a dead-end to the closest retreat node and resumes the back-and-forth coverage. If there is no retreat node available, it implies that complete coverage is achieved. 
\end{itemize}
Note: While C$^*$ provides complete coverage of the interior free regions of a map, it provides resolution completeness with respect to the RCG along obstacle boundaries.
\end{proof}

\begin{figure*}[t]
    \centering
    \subfloat[Coverage trajectories generated by different algorithms.]{
        \includegraphics[width=0.9\textwidth]{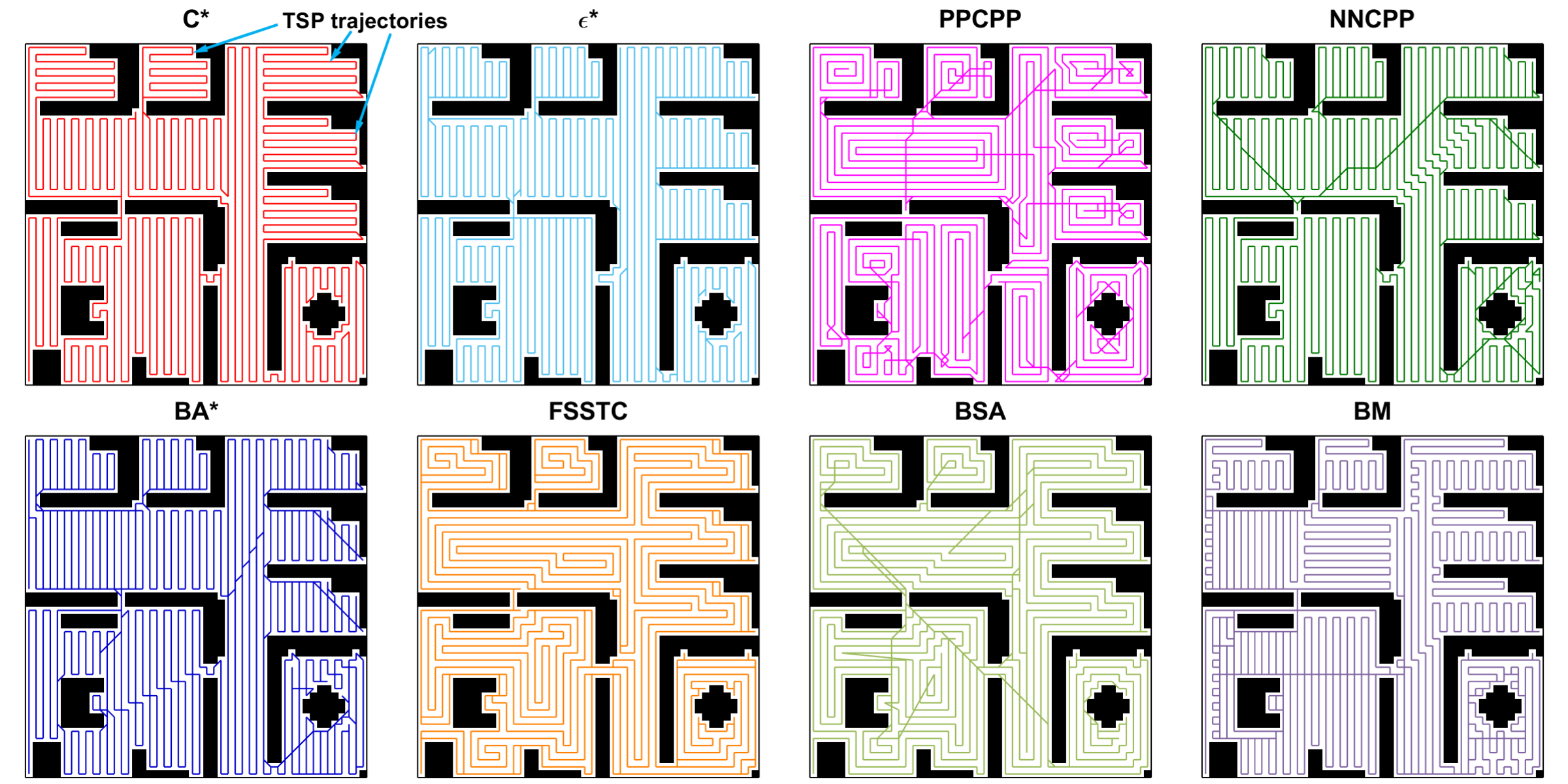}\label{fig:path_scene2}}\vspace{10pt}\\
         \centering
    \subfloat[Comparison of performance metrics of different algorithms.]{
         \includegraphics[width=0.9\textwidth]{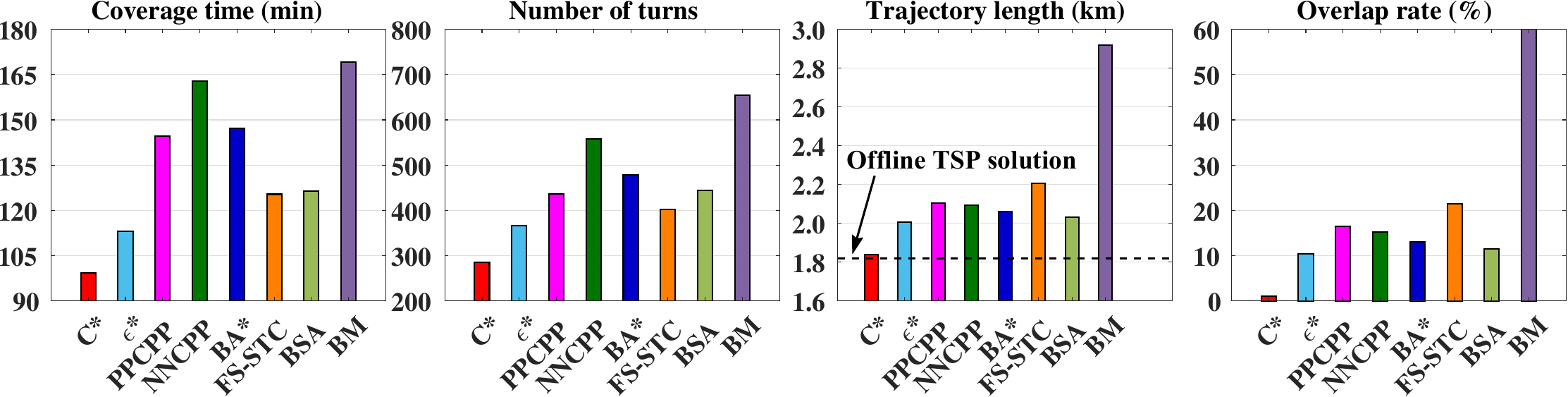}\label{fig:performance_scene2}}\\
          \caption{\textbf{Scenario 2:} Performance comparison of C$^*$ with the baseline algorithms. Video available in supplementary documents.}\label{fig:comparison_scene2}
          \vspace{-0pt}
\end{figure*}

\begin{figure*}[t]
    
        \centering
    \subfloat[Twenty four randomly generated scenarios for performance validation.]{
        \includegraphics[width=0.86\textwidth]{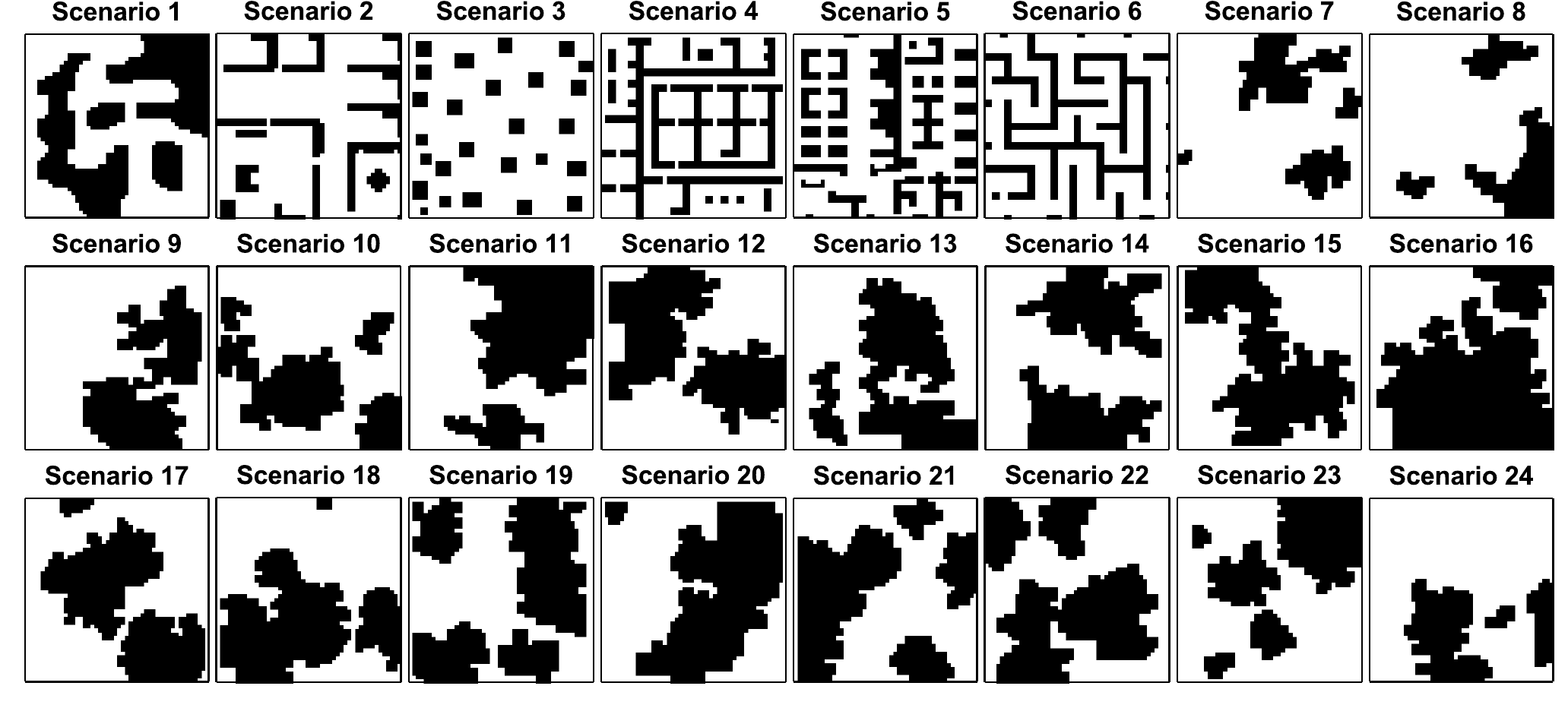}\vspace{-1.5em}\label{fig:scenario}}\vspace{1em}\quad
        \centering
    \subfloat[Comparison of performance metrics of different algorithms.]{
         \includegraphics[width=0.76\textwidth]{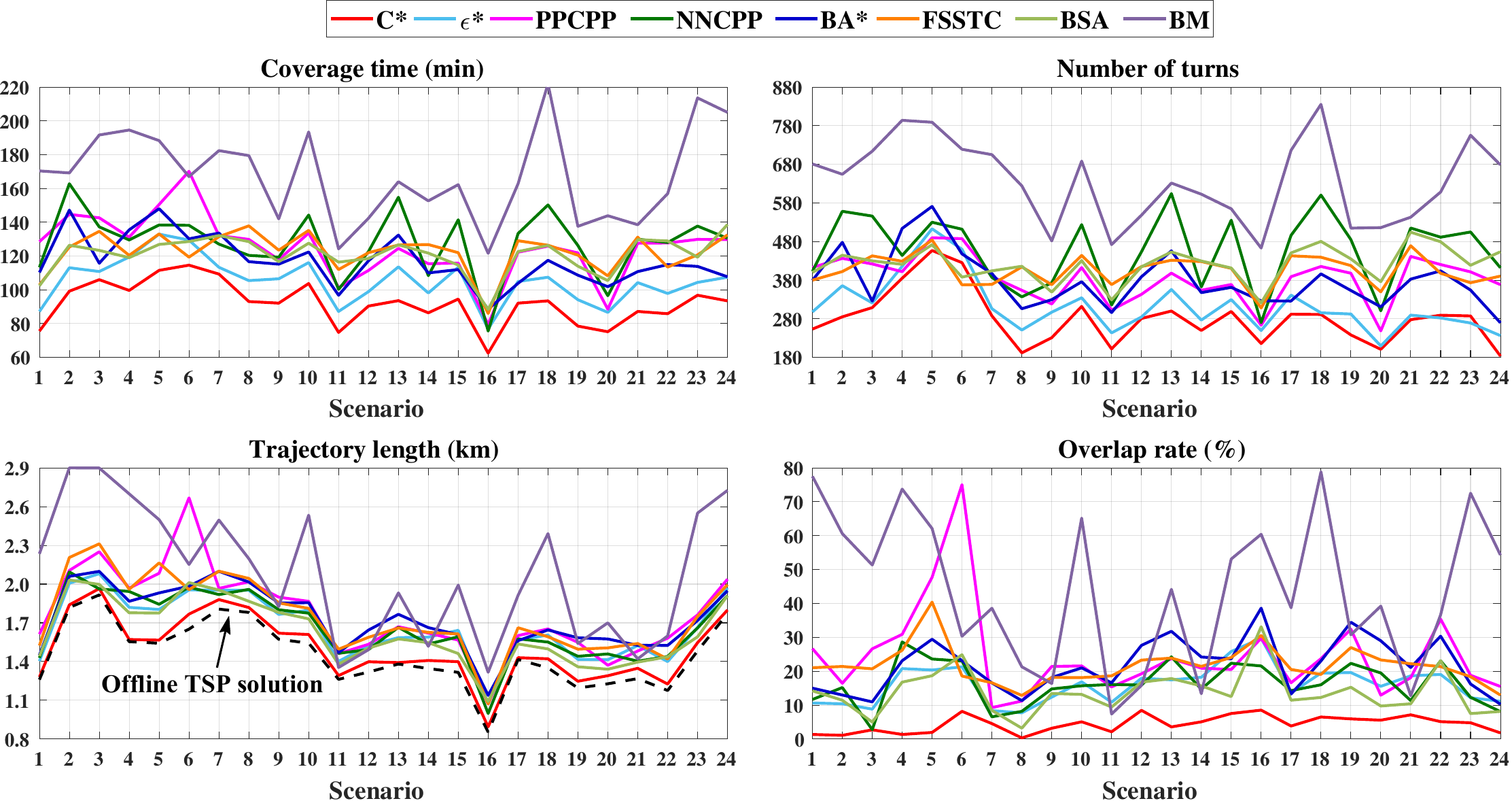}\label{fig:performance_overall}}\\
          \caption{Performance comparison of C$^*$ with the baseline algorithms on different scenarios}.\label{fig:comparison_overall}
          \vspace{-15pt}
\end{figure*}

\vspace{0pt}
\section{Results and Discussion}
\label{sec:results}
This section presents the results of testing and validation of 
C$^*$ by 1) high-fidelity simulations, 2) real-experiments on autonomous robots in a laboratory setting, and 3) applications on two different CPP problems considering a) energy-constrained robots and b) multi-robot teams.

\begin{figure}[!h]
 \centering 
  \subfloat{
 \includegraphics[width=0.86\columnwidth]{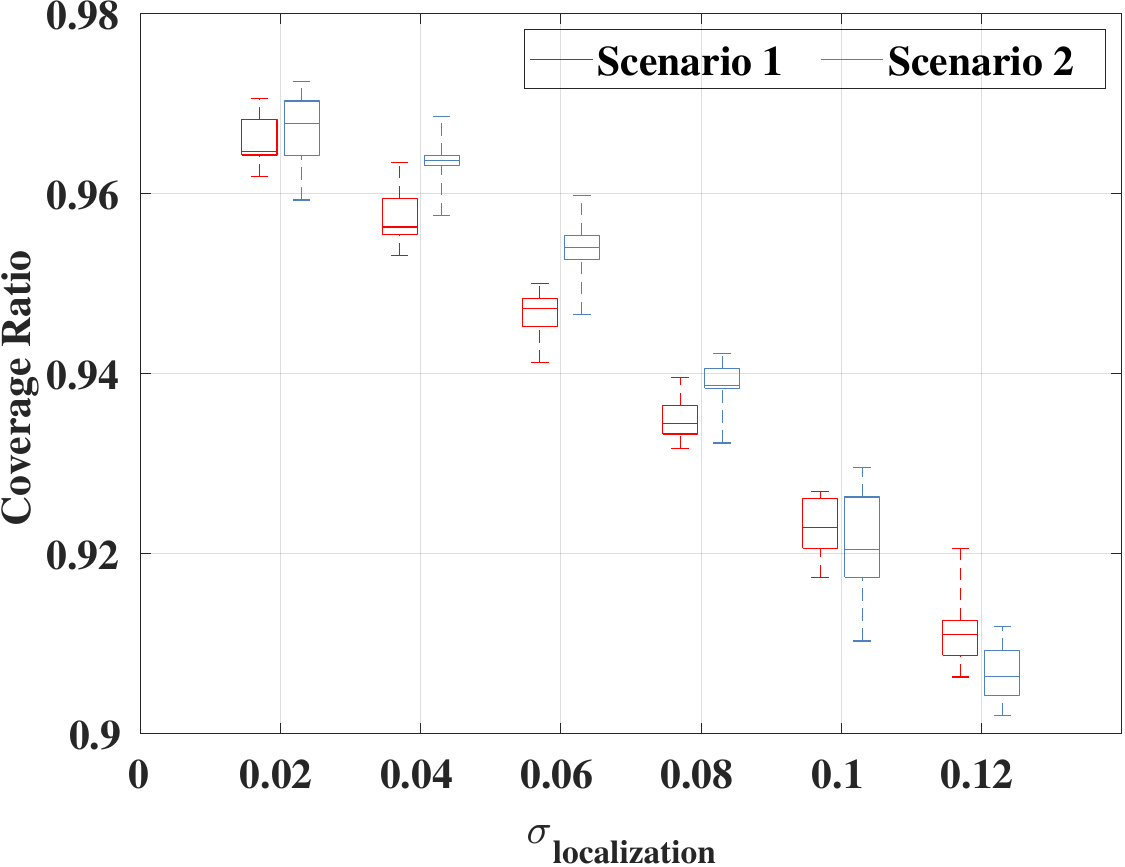}}
\caption{Coverage ratio versus noise.}\label{fig:uncertainty}
 \vspace{-6pt}
\end{figure}

\vspace{0pt}
\subsection{Validation on a Simulation Platform}
The performance of C$^*$ is first evaluated by extensive high-fidelity simulations on a multitude of complex scenarios. A robotic simulation platform, called Player/Stage~\cite{gerkey2003player} is used, where Player is a robot server that contains models of different types of robots and sensors, while Stage is a highly configurable simulator of a variety of robotic operations. Several complex scenarios of dimensions $50$ m $\times$ $50$ m are randomly generated with different obstacle layouts to test and validate the performance of C$^*$. These scenarios are considered as unknown to the robot at the beginning of the CPP operation. In order to detect and map obstacles during navigation, the autonomous robot is equipped with a laser sensor that has the FOV with a span of $360^\circ$ and a detection range of $15$ m. The robot has motion constraints including the top speed of $0.5$ m/s, max acceleration of $0.5$ m/s$^2$, and the minimum turning radius of $0.04$ m. The robot dynamically discovers the environment during navigation and generates adaptive coverage trajectory until complete coverage is achieved.

\vspace{0pt}
\subsubsection{\textbf{Baseline Algorithms}} 
Seven baseline algorithms are selected for performance comparison including: PPCPP~\cite{hassan2019ppcpp}, $\varepsilon^*$~\cite{song2018}, BA$^*$~\cite{viet2013ba}, NNCPP~\cite{luo2008bioinspired}, BM~\cite{ferranti2007brick}, BSA~\cite{gonzalez2005bsa}, and FSSTC~\cite{gabriely2003competitive}. Table~\ref{tab:feature} compares the key features of C$^*$ with the baseline algorithms. All algorithms were deployed in C$++$ and run on a computer with $2.60$ GHz processor and $32$ GB RAM. The sampling resolution for C$^*$ and the grid resolution for the baseline algorithms are set as $1$ m.

\vspace{6pt}
\subsubsection{\textbf{Performance Metrics}} 
The following performance metrics are considered for comparative evaluation:
\begin{itemize}
\item Coverage Time (CT): the total time of the coverage trajectory to achieve complete coverage.
\item Number of Turns (NT): the total number of turns on the coverage trajectory computed as the integral part of the total turning angle of the trajectory divided by $90^\circ$.
\item Trajectory Length (TL): the total length of the coverage trajectory to achieve complete coverage.
\item Overlap Rate (OR): the percentage of area covered repeatedly out of the total area of the coverage space. This is computed by  partitioning the coverage space into a grid with $1$ m resolution, and calculating the total number of obstacle-free cells that are traversed more than once with respect to the total number of free cells. 
\end{itemize}

\vspace{6pt}
\subsubsection{\textbf{Comparative Evaluation Results}}\label{subsec:mc_sim_res}
Figs.~\ref{fig:comparison_scene1} and \ref{fig:comparison_scene2} show the comparative evaluation results on two different scenarios: 1) an island scenario with complex obstacle layout, and 2) an indoor scenario (e.g., office floor or apartment) with several rooms, respectively. Figs.~\ref{fig:path_scene1} and \ref{fig:path_scene2} show the coverage trajectories generated by C$^*$ and the baseline algorithms for the above two scenarios, respectively. As seen, C$^*$ generates the back-and-forth coverage trajectories by default and adapts to the TSP-based optimal trajectories to cover any coverage holes detected on the way. This avoids the need to cover such holes later, thereby providing clean coverage without crossings and reducing the trajectory lengths and overlapping rates. Furthermore, there are no zig-zag and spiral patterns in the C$^*$ trajectories making them more appealing to the users. The baseline algorithms perform different motion patterns including the back-and-forth and spiral for complete coverage.  Figs.~\ref{fig:performance_scene1} and \ref{fig:performance_scene2} provide quantitative comparison results in terms of the performance metrics (i.e., coverage time, number of turns, trajectory length, and overlap rate) for the above two scenarios, respectively. Overall, C$^*$ achieves significant improvements over the baseline algorithms in all metrics. This follows from the fact that C$^*$ produces very efficient coverage pattern with minimal overalappings and number of turns.

 \begin{figure*}[t]
    \centering
    \subfloat[Coverage operation started.]{
        \includegraphics[width=0.47\textwidth]{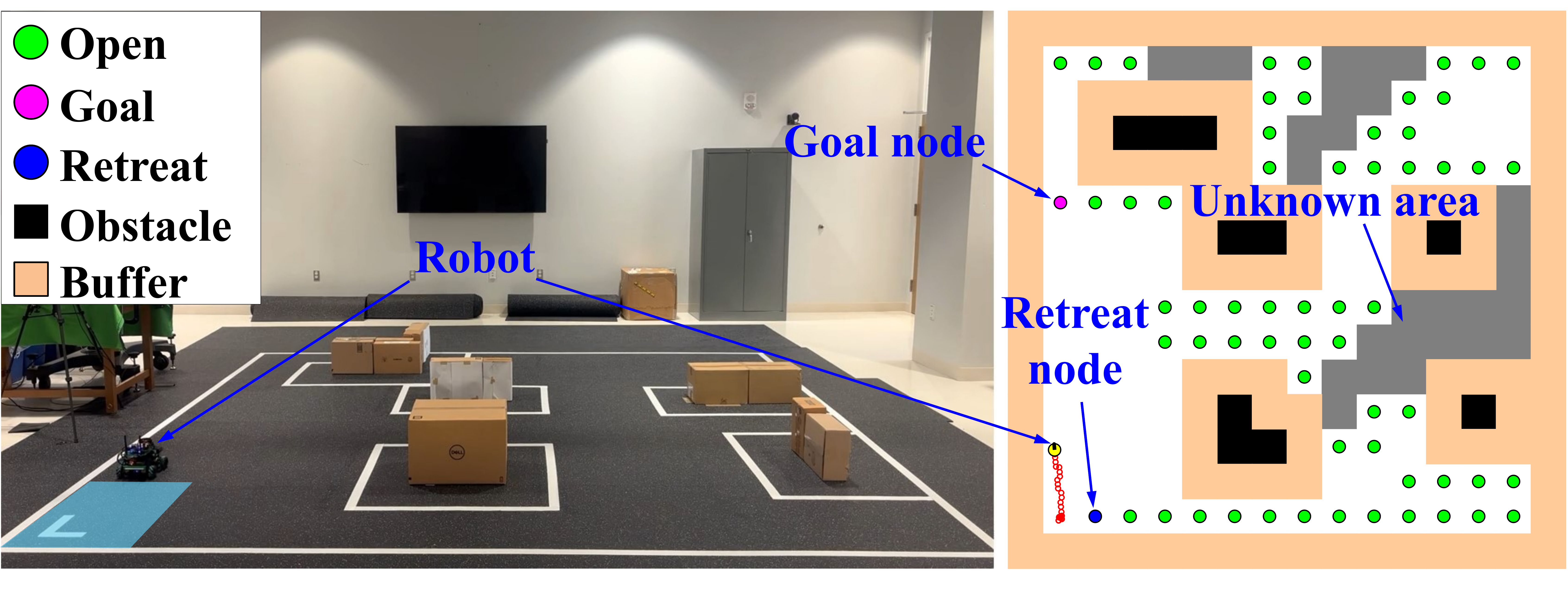}\label{fig:Cstar_experiment_part1}}\hspace{-12pt}\quad
    \centering
    \subfloat[Escaping from a dead-end.]{
        \includegraphics[width=0.47\textwidth]{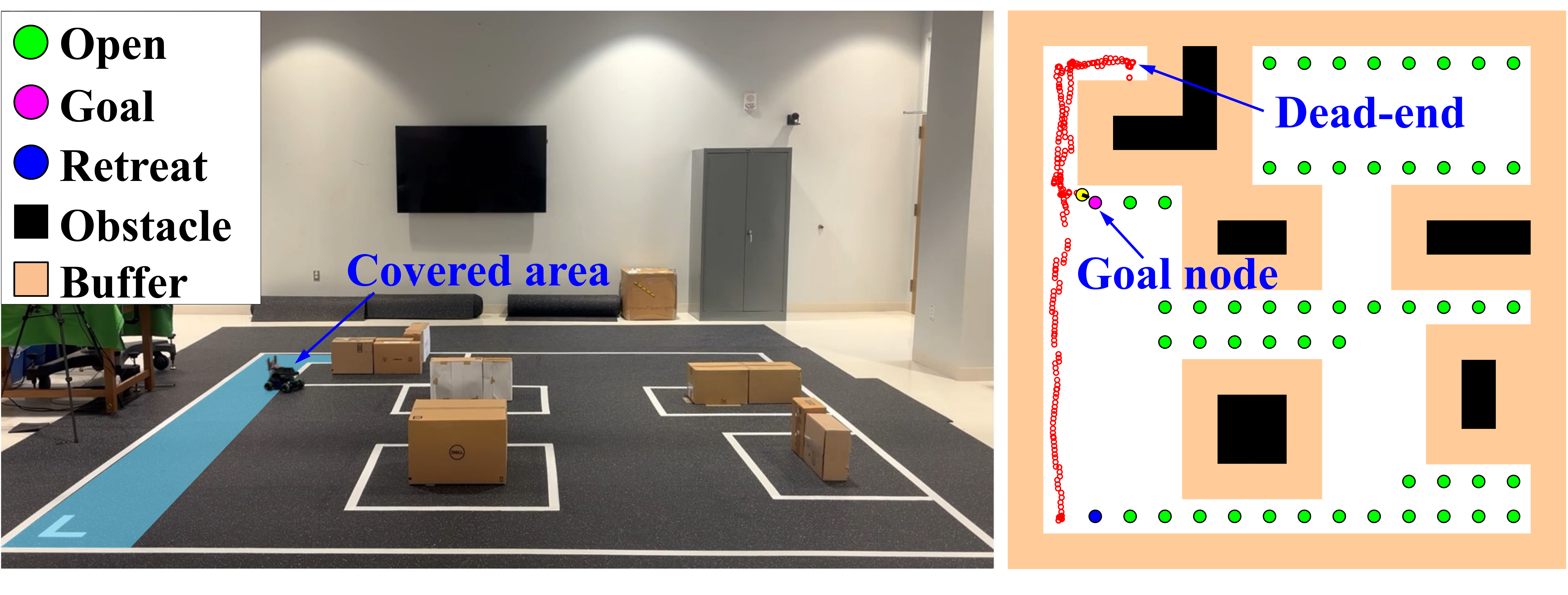}\label{fig:Cstar_experiment_part2}}\hspace{-12pt}\\
        \centering
    \subfloat[Continue back-and-forth coverage.]{
        \includegraphics[width=0.47\textwidth]{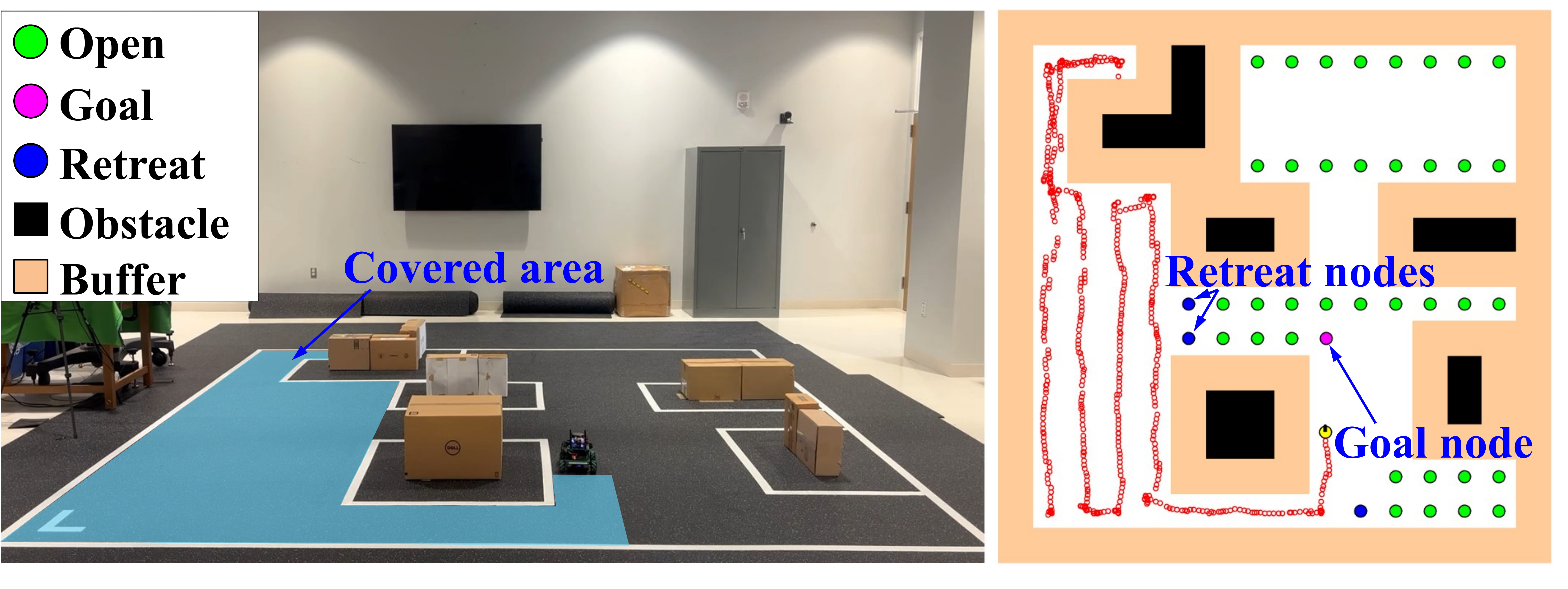}\label{fig:Cstar_experiment_part3}}\hspace{-12pt}\quad
    \centering
    \subfloat[Coverage-hole detected and TSP-based coverage started.]{
        \includegraphics[width=0.47\textwidth]{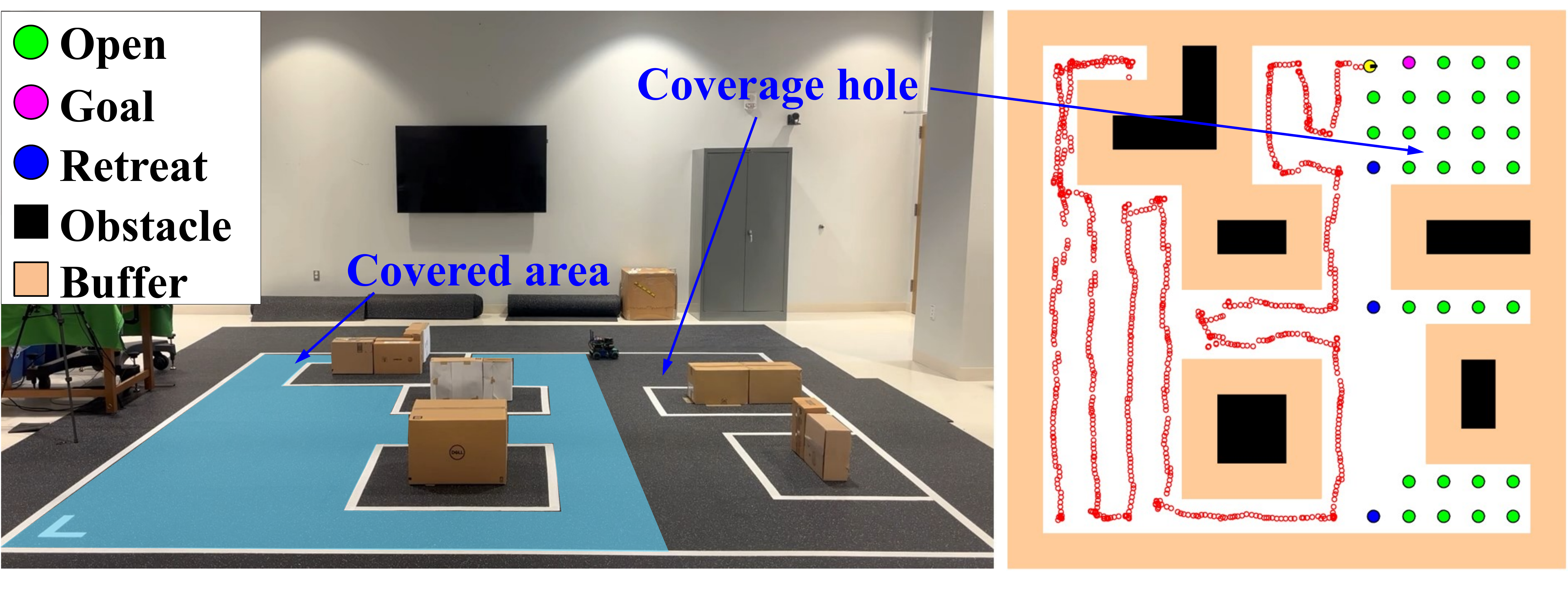}\label{fig:Cstar_experiment_part4}}\hspace{-12pt}\\
    \centering
    \subfloat[Another coverage-hole detected and TSP-based coverage started.]{
        \includegraphics[width=0.47\textwidth]{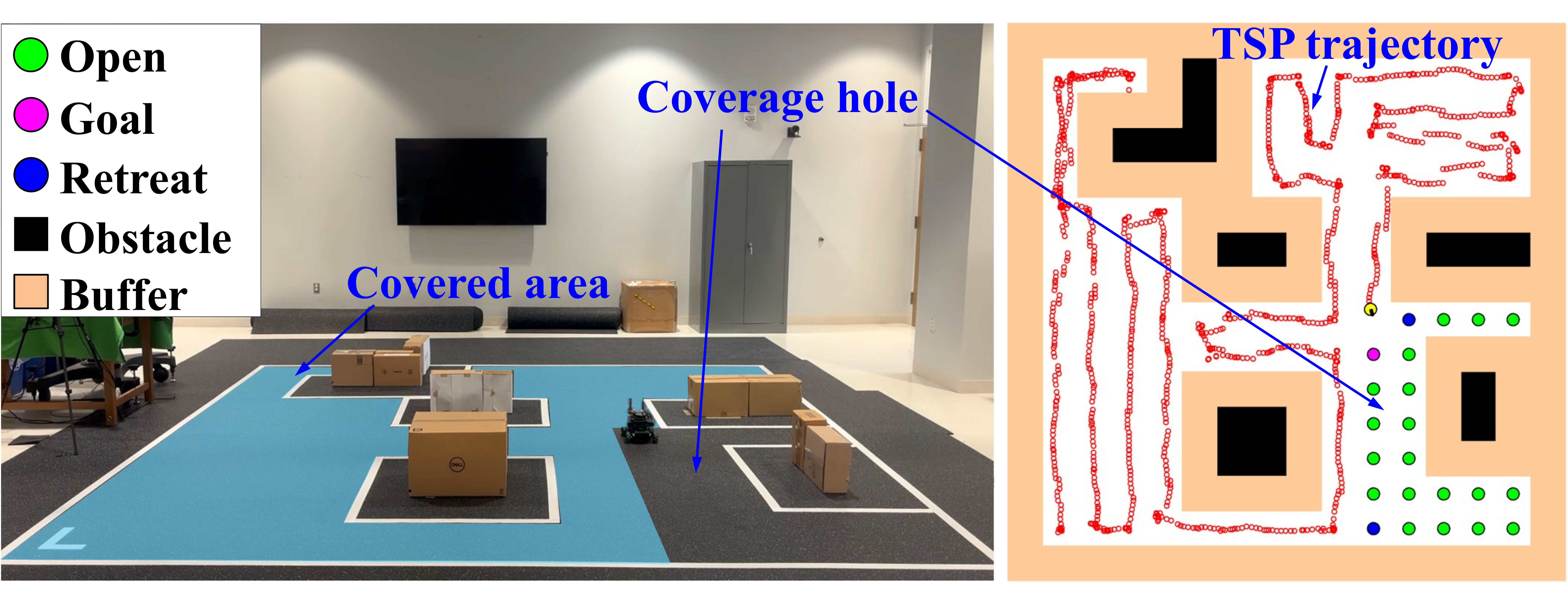}\label{fig:Cstar_experiment_part5}}\hspace{-12pt}\quad
    \centering
    \subfloat[Complete coverage achieved.]{
        \includegraphics[width=0.47\textwidth]{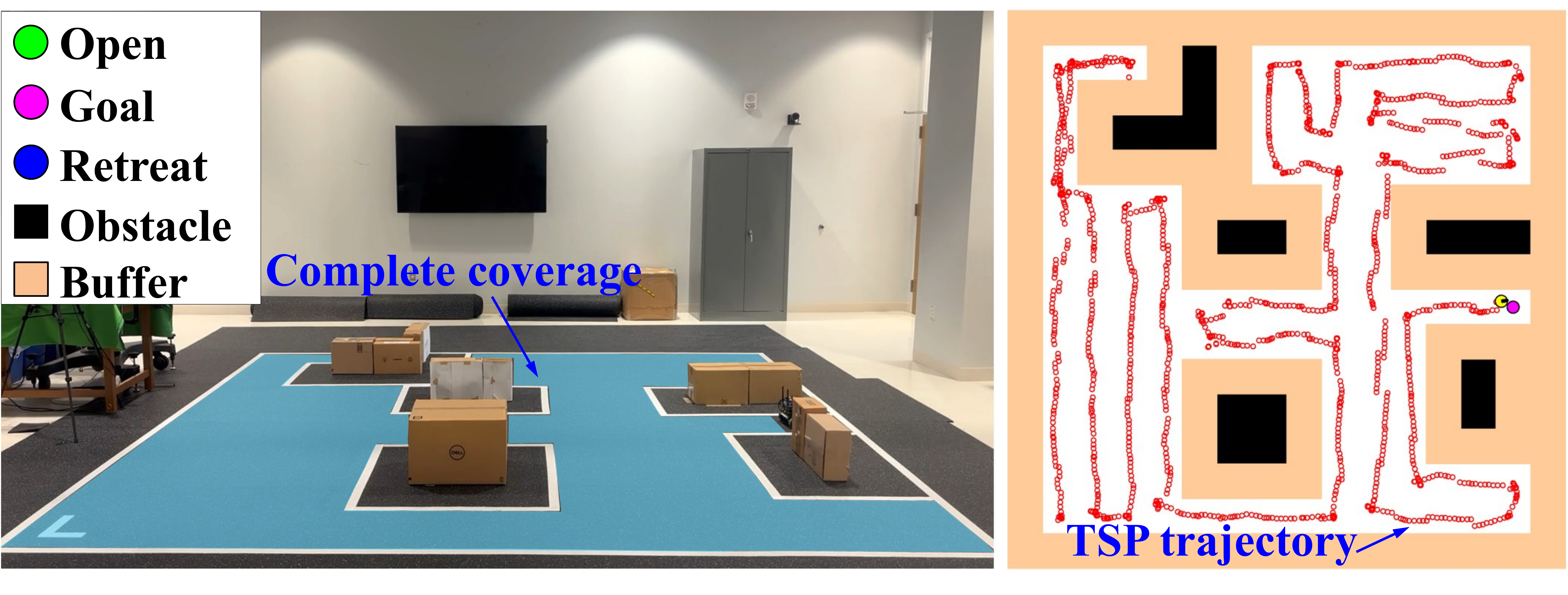}\label{fig:Cstar_experiment_part6}}\hspace{-12pt}\\
          \caption{Snapshots of a real experiment in a laboratory with complex obstacle layout. Video available in supplementary documents.}\label{fig:Cstar_experiment}
          \vspace{0pt}
\end{figure*}

For further validation, a total of twenty four diverse scenarios with complex obstacle layouts are randomly generated, as shown in Fig.~\ref{fig:scenario}. The performance metrics of all algorithms are recorded for these scenarios and plotted in Fig.~\ref{fig:performance_overall}. The results show a similar trend where C$^*$ provides a superior solution quality in all metrics as compared to the baseline algorithms. The average computation time recorded in these scenarios at each iteration including the sampling, RCG construction, and waypoint generation is of the order of $10^{-2}$ s. The average computation time to find a TSP trajectory for a coverage hole is $\sim 0.015$ s.

\vspace{6pt}
{$\bullet$ \textit{\textbf{Comparison with the Optimal TSP-Solution}}:} To evaluate the quality of C$^*$ paths, we computed the optimal coverage path with the shortest length for each of the above scenario by solving a TSP over the full environment map using the Concorde exact solver~\cite{David2005}. Figs.~\ref{fig:performance_scene1},~\ref{fig:performance_scene2} and ~\ref{fig:performance_overall} show the TSP path lengths and it is seen over all scenarios that C$^*$ achieves path lengths close to the optimal TSP solution, while maintaining real-time applicability in unknown environments.

\vspace{6pt}
\subsubsection{\textbf{Performance in the Presence of Uncertainties}}
The performance of C$^*$ is further evaluated in the presence of sensor uncertainties. These uncertainties can result in poor localization of the robot as well as misplacement of detected obstacles, which could impact the overall coverage performance. Therefore, to analyze the effects of sensor uncertainties, noise was injected into the measurements of range detector (laser sensor), the heading angle (compass), and the localization system of the robot. A laser sensor typically admits a measurement error of 1\% of its operation range and a modestly priced compass can provide heading information as accurate as $0.5$$^\circ$~\cite{paull2013auv}. The above errors were simulated with Additive White Gaussian Noise (AWGN) with standard deviations of $\sigma_{laser}=1.5$ cm and $\sigma_{compass} = 0.5^\circ$, respectively. The accuracy of the robot localization system depends on many factors, such as applied techniques, hardware platforms and environmental conditions. For example, the accuracy of a GPS system can range from $0.02$ m to $0.12$ m~\cite{paull2013auv}. Thus, the uncertainty due to localization system is studied using an AWGN with standard deviation ranging from $\sigma_{localization} =0.02$ m to $0.12$ m. 

\begin{figure}[!t]
    \centering
    \includegraphics[width=0.7\columnwidth]{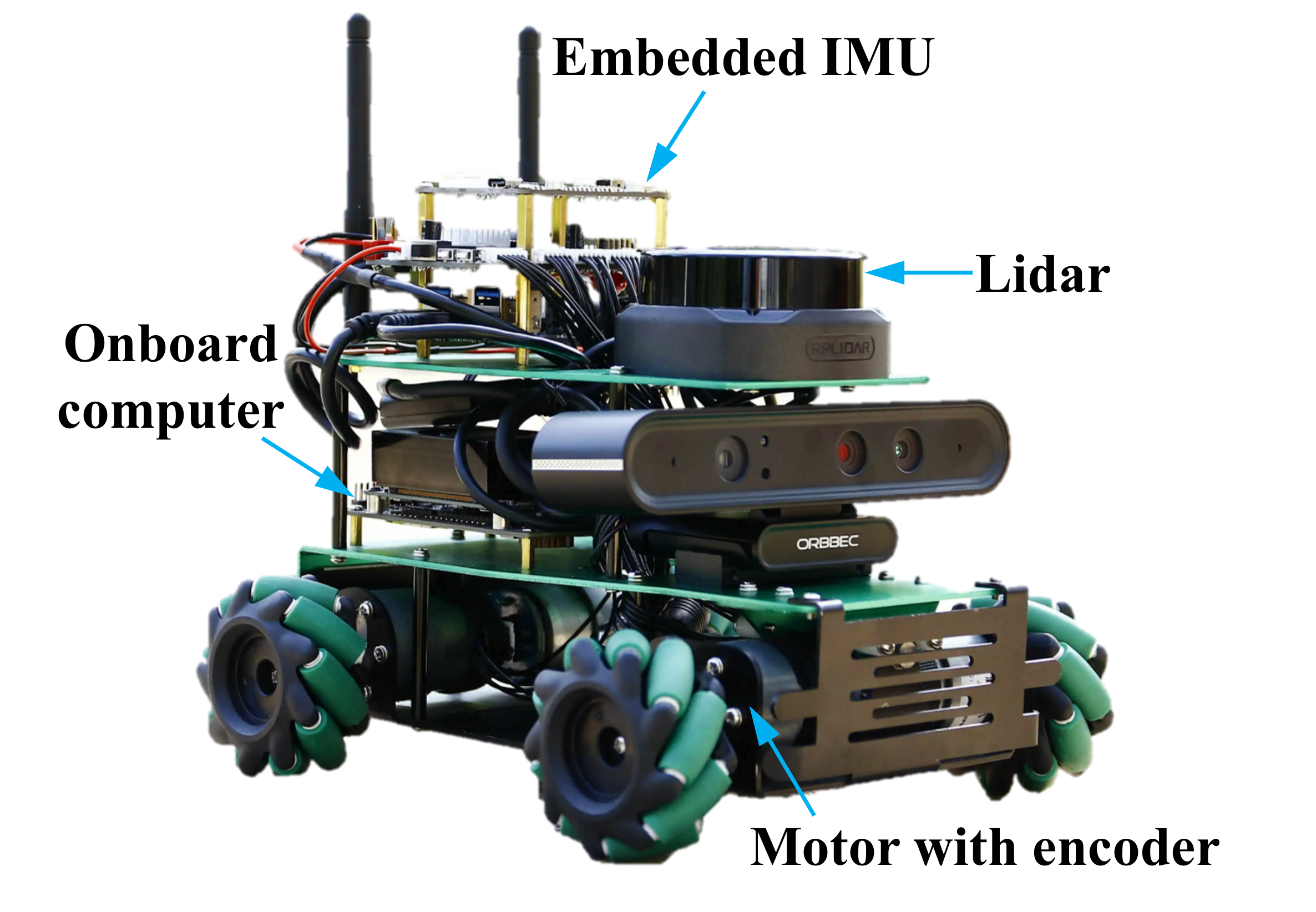}
    \caption{Robot ROSMASTER X3 integrated with sensors.} 
   \label{fig:robot} 
   \vspace{-9pt}
\end{figure}

The above uncertainties have a direct effect on the coverage ratio, i.e., the percentage of area covered with respect to the actual area of the coverage space. To calculate the coverage ratio, the search area is partitioned into a grid with a resolution of $1$ m, then the total number of obstacle-free cells covered by the robot are recorded and divided by the total number free cells. Fig.~\ref{fig:uncertainty} shows the box plot of the coverage ratio for different noise levels over ten Monte Carlo runs for the two scenarios presented above. The horizontal line in a box is the median result of the Monte Carlo simulations, while the box shows the middle $50^{th}$ percentile values of results.

\vspace{-6pt}
\subsection{Validation by Real Experiments}
The performance of C$^*$ is also validated by real experiments in a $7$ m $\times$ $7$ m lab space with several static obstacles. Fig~\ref{fig:robot} shows the robot, ROSMASTER X3~\cite{robot}, used for coverage of this space. The robot is equipped with 1) a RPLIDAR S2L~\cite{robot} with a range of $8$ m for obstacle detection, 2) MD520 motor with an encoder~\cite{robot} for detection of wheel rotation and linear displacement, and 3) MPU9250 IMU~\cite{robot} for detection of speed, acceleration, and orientation. The Gmapping algorithm~\cite{grisetti2007improved} is used for indoor localization. 
The robot carries Jetson Nano minicomputer that collects sensor measurements and runs C$^*$ for real-time control and navigation. Table~\ref{tab:sensor} lists the sensing and computing systems on the real robot.

Fig.~\ref{fig:Cstar_experiment} shows the various snapshots of an experiment where the robot covers the lab space by running C$^*$ onboard with adaptive coverage trajectory generation in real-time. Each snapshot consists of two images where i) the left image shows the actual robot covering the lab space and ii) the right image shows the discovered and unknown areas, RCG nodes statuses marked with different colors (\textit{Open} nodes as green, retreat nodes as blue and goal node as pink), and the coverage trajectory executed by the robot. The buffer zones marked by orange color are added around the obstacles for safety of the robot considering inertia and localization error. Fig.~\ref{fig:Cstar_experiment_part1} shows the start of the coverage process at the bottom left corner where the robot travels ahead of the starting point. Fig.~\ref{fig:Cstar_experiment_part2} shows the instant when the robot reaches a dead-end surrounded by obstacle and covered area in the local neighborhood. In this case, it escapes from the dead-end by selecting the nearest retreat node as the next waypoint. Fig.~\ref{fig:Cstar_experiment_part3} shows that the robot continues the back-and-forth coverage. Fig.~\ref{fig:Cstar_experiment_part4} shows the instant when the robot detects a coverage hole and starts the TSP-based optimal coverage. Fig.~\ref{fig:Cstar_experiment_part5} shows the TSP-based optimal coverage of another coverage hole.  
Finally, Fig.~\ref{fig:Cstar_experiment_part6} shows that the robot achieves complete coverage, thus validating the effectiveness of C$^*$ on real robots. The average computation time in the real experiment at each iteration is of the order of $10^{-2}$ s. The average computation time to find a TSP trajectory is $\sim0.008$ s.

\begin{table}[!t]{}
\scriptsize
\caption {Sensing and computing systems of the experimental robot. }\label{tab:sensor}\vspace{-3pt}
\centering
\setlength\tabcolsep{2.0pt}
\begin{tabular}{l c c c c c} 
 \toprule
\specialrule{0em}{1pt}{1pt}\vspace{3pt} 
&\tabincell{l}{\textbf{Onboard} \\ \textbf{computer}} 
&\tabincell{l}{\textbf{Wheel} \\ \textbf{encoder}}
& \textbf{IMU}
&\tabincell{l}{\textbf{Lidar}} 
& \textbf{Localization}\\ 
\toprule 
\specialrule{0em}{3pt}{-3pt}
\tabincell{l}{\textbf{Model}}
& \tabincell{l}{Jetson Nano}
& \tabincell{l}{MD520}  
& \tabincell{l}{MPU9250} 
& \tabincell{l}{RPLIDAR S2L}
& \tabincell{l}{Gmapping}\\
\specialrule{0em}{3pt}{3pt}
\tabincell{l}{\textbf{Specifications}}
& \tabincell{l}{$4$GB RAM \\ and $1.5$GHz \\ processor}
& \tabincell{l}{$-$} 
& \tabincell{l}{$-$}
& \tabincell{l}{Range $8$ m}
& \tabincell{l}{$-$}\\
\toprule
\end{tabular}
\end{table}

\begin{figure*}[t]
    \centering
    \subfloat[C$^*$ coverage in different cycles where each cycle consists of advance (magenta), coverage (red) and retreat  (blue) trajectory segments.]{
    \includegraphics[width=0.98\textwidth]{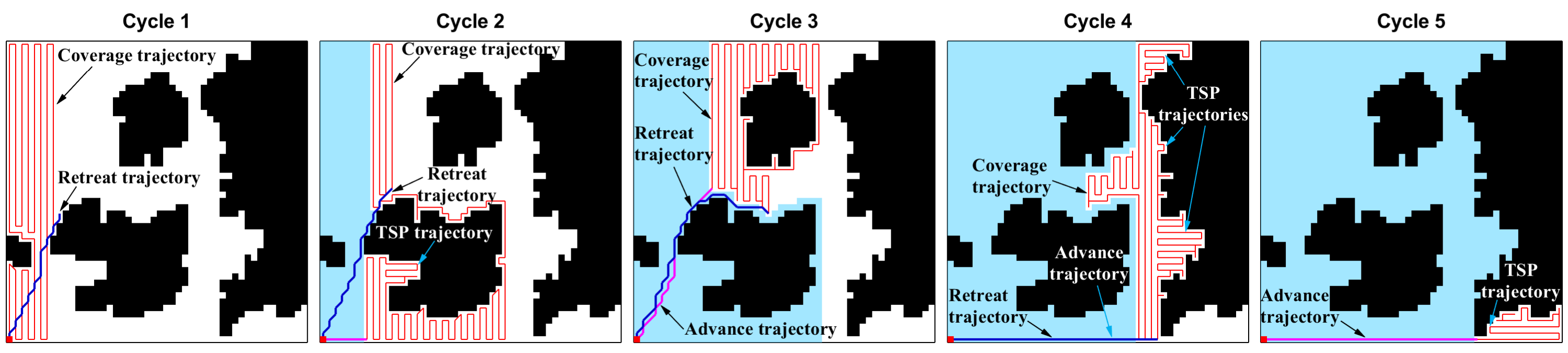}\label{fig:CstarPlus}\vspace{0pt}}\vspace{3pt}\\
    \centering
    \subfloat[$\varepsilon^*$+ coverage in different cycles where each cycle consists of advance (magenta), coverage (red) and retreat  (blue) trajectory segments.]{
    \includegraphics[width=0.98\textwidth]{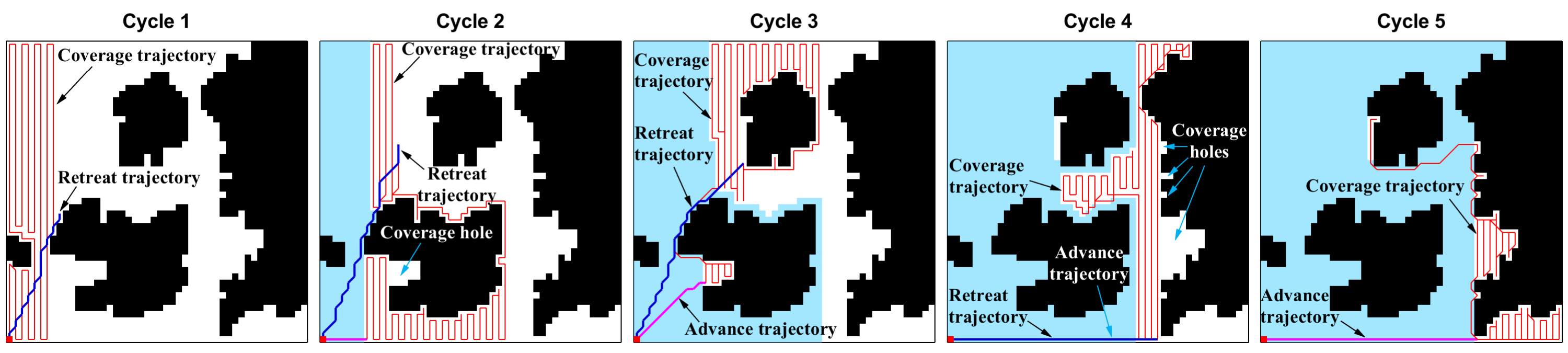}\label{fig:EstarPlus}\vspace{-0pt}}\vspace{3pt}\\
    \centering
    \subfloat[ECOCPP coverage in different cycles where each cycle consists of advance (magenta), coverage (red) and retreat  (blue) trajectory segments.]{
    \includegraphics[width=0.98\textwidth]{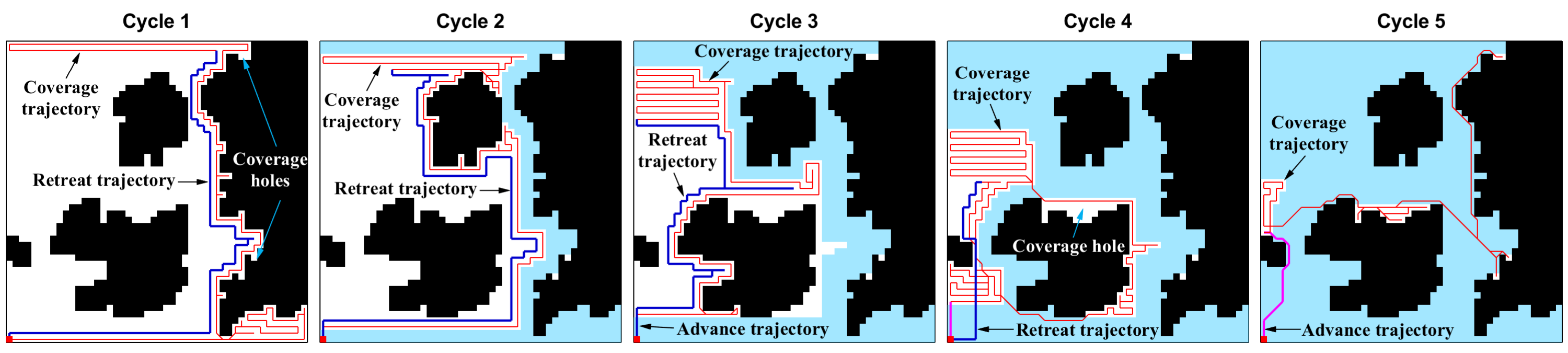}\label{fig:ECOCPP}\vspace{-0pt}}\vspace{6pt}\\
    \centering
    \subfloat[Comparison of performance metrics of different algorithms.]{
    \includegraphics[width=0.95\textwidth]{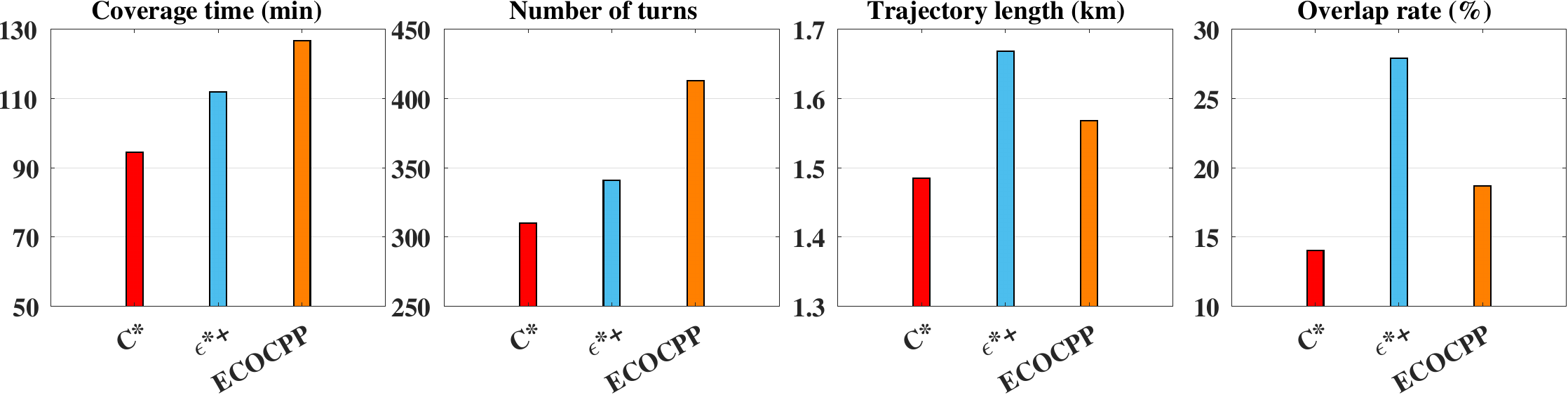}\label{fig:performance_energyConstraint}}\\      
  \caption{Performance comparison of C$^*$ with the baseline algorithms.}\label{fig:compare_energyConstraint}
      \vspace{-0pt}
\end{figure*}

\vspace{6pt}
\subsection{CPP Applications}
\label{Casestudy}

Finally, the performance of C$^*$ is evaluated on two different CPP applications using 1) energy-constrained robots and 2) multi-robot teams.

\begin{figure*}[t]
    \centering
    \subfloat[Coverage trajectories generated by different algorithms. Four robots are initialized at four corners.]{
    \includegraphics[width=0.90\textwidth]{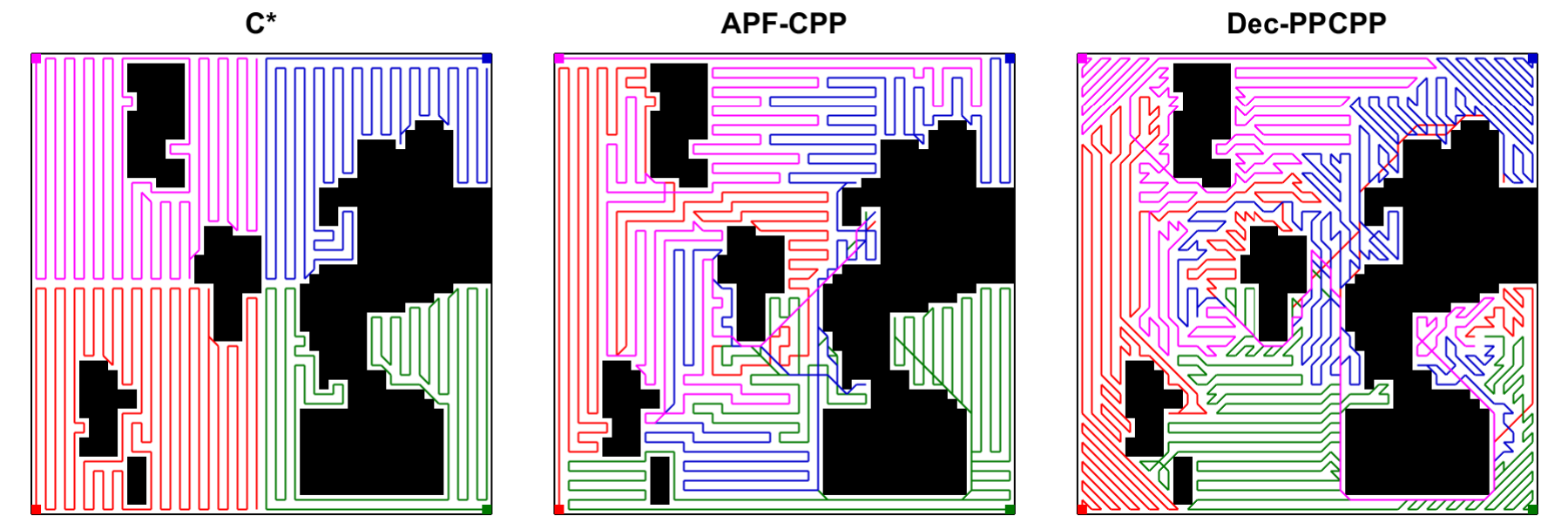}\label{fig:path_multirobot}}\vspace{12pt}\\
         \centering
    \subfloat[Comparison of performance metrics of different algorithms.]{
         \includegraphics[width=0.90\textwidth]{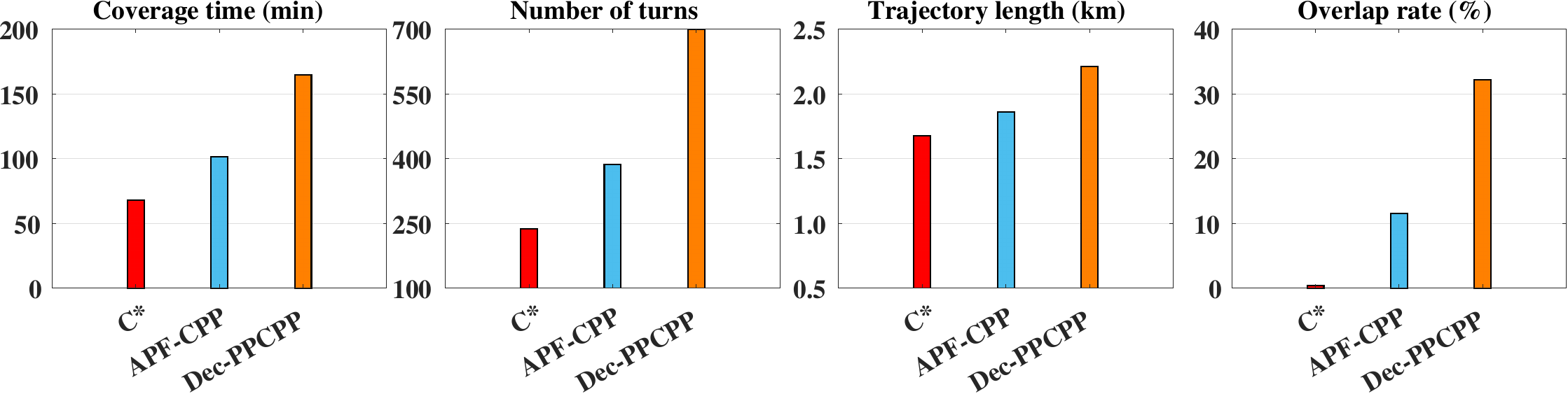}\label{fig:performance_multirobot}}\\
          \caption{Performance comparison of C$^*$ with the baseline algorithms.}\label{fig:comparison_multirobot}
          \vspace{-0pt}
\end{figure*}

\vspace{6pt}
\subsubsection{\textbf{Coverage using Energy-Constrained Robots}}
\label{subsubsection:energyconstrained}
In this application, the CPP problem considers robots with an energy constraint, i.e., limited battery capacity. Thus, the robot has to make multiple trips back to the charging station to recharge its battery. As such, the coverage is completed in multiple cycles 
where each cycle consists of three trajectory segments: 1) advance, 2) coverage, and 3) retreat. 
\begin{itemize}
\item The advance trajectory segment takes the robot from the charging station to the starting point of next round.  
\item The coverage trajectory segment performs coverage till the robot's battery depletes to the extent that enables it to go back to the charging station, and 
\item The retreat trajectory segment brings the robot back to the charging station along the shortest path.
\end{itemize}

During coverage the robot continuously monitors the remaining energy so that it has sufficient energy left for returning to the charging station. If the energy level is low then it returns back for recharging. For simplicity the energy consumption of the robot is modeled as proportional to the trajectory length~\cite{shnaps2016}. For C$^*$, the robot evaluates its energy content every time it reaches a node of RCG to make a decision to return back for recharging. Thus, while sitting at the current node $\hat{n}_{i}$, it computes the expected remaining energy upon reaching the next goal node $\hat{n}_{i+1}$. This is computed as the current remaining energy minus the expected energy needed to move to $\hat{n}_{i+1}$. Then it computes the retreat trajectory from $\hat{n}_{i+1}$ to the charging station using A$^*$~\cite{hart1968formal}. If the expected remaining energy at $\hat{n}_{i+1}$ is less than the expected energy consumption of the retreat trajectory, then it retreats back to the charging station from $\hat{n}_{i}$, otherwise it moves to $\hat{n}_{i+1}$. Once the robot is recharged at the charging station, the robot selects the nearest open node on RCG as the start node of next round and advances to this node along the shortest A$^*$ path.

Fig.~\ref{fig:compare_energyConstraint} shows the comparative evaluation results of C$^*$ with the baseline algorithms ($\varepsilon^*+$~\cite{shen2020} and energy constrained online path planning (ECOCPP) algorithm~\cite{dogru2022eco}) on a randomly generated $50$ m $\times$ $50$ m complex scenario. The battery charging station is located at the bottom left corner. The robot is considered to have the total energy capacity of $360$ units. Figs.~\ref{fig:CstarPlus}-\ref{fig:ECOCPP} show the trajectories of C$^*$, $\varepsilon^*+$ and ECOCPP, respectively, at different cycles with the advance, coverage and retreat segments. Fig~\ref{fig:performance_energyConstraint} shows the performance comparison of C$^*$ with the baseline algorithms in terms of different performance metrics. It is seen that C$^*$ provides clean coverage pattern from left to right, with minimum coverage time and overlappings, and leaving behind no coverage holes. Overall, C$^*$ achieves the best performance in all metrics.

\vspace{6pt}
\subsubsection{\textbf{Coverage using a Multi-Robot Team}}
\label{subsubsection:multirobot}
In this application, the CPP problem considers a multi-robot team to jointly cover the coverage space. It is desired that complete coverage is achieved with minimum overlappings of the multi-robot trajectories to avoid repetition. The multi-robot team consists of $4$ robots that are initialized at the corners of the space. For C$^*$, the area is partitioned into $4$ equal-sized subregions for coverage by each robot. Such partition is simple to create and helps to avoid mingling of the robot trajectories during coverage, thus preventing complex maneuvers. Note that the construction of an optimal partitioning scheme based on robot capabilities and obstacle distribution is beyond the scope of this paper. Further, the optimal reallocation strategies based on task completion, resource management and robot failures is also beyond the scope of this paper. 

The multi-robot CPP application presented in this paper focuses on an efficient hybrid approach for joint coverage. In this approach, each robot collects sensor data of the environment and sends it to a centralized server located at the base station. The server first creates a unified map of the environment with the input data from multiple robots. Then it performs progressive sampling on the sampling front of each robot. Note that the sampling is also done on the boundaries of individual partitioned subregions. Then it constructs a single RCG using these samples, while retaining the boundary samples during RCG pruning. The information sharing with the base station and construction of a single RCG is important because some obstacle-free regions could be hidden behind an obstacle and neither detectable by the robot nor reachable from within its partitioned subregion. After that, each robot uses this RCG to calculate its own goal node for covering its subregion. 

The performance of  C$^*$ is compared with baseline algorithms (Dec-PPCPP~\cite{hassan2020dec} and Artificial Potential Field CPP (APF-CPP) algorithm~\cite{wang2024apf}). Dec-PPCPP is extended from PPCPP~\cite{hassan2019ppcpp} into a decentralized multi-robot CPP by adding dynamic predator (other robots) avoidance reward in the reward function.  Wang et al.~\cite{wang2024apf} proposed the APF-CPP algorithm which leverages the concept of artificial
potential field for decision making of each robot. Fig.~\ref{fig:path_multirobot} shows the trajectories of C$^*$, APF-CPP and Dec-PPCPP. Fig~\ref{fig:performance_multirobot} shows the comparative evaluation results in terms of the different performance metrics. It is seen that C$^*$ provides clean coverage pattern, with minimum coverage time and overlappings, and leaving behind no coverage holes. Overall, C$^*$ outperforms the baseline algorithms in terms of all metrics.

\section{Conclusion and Future Work}
\label{sec:conclusions}
The paper presents a novel algorithm, called C$^*$, for adaptive CPP in unknown environments. C$^*$ is built upon the concepts of progressive sampling, incremental RCG construction and TSP-based coverage hole prevention. It is shown that C$^*$ is computationally efficient and guarantees complete coverage. The comparative evaluation with existing algorithms shows that C$^*$ significantly improves the coverage time, number of turns, trajectory length, and overlap rate. Finally, C$^*$ is validated by real experiments and also extended to energy-constrained and multi-robot coverage applications. 

Future research areas include advanced CPP problems that consider 1) decentralized multi-agent coverage for robustness to failures~\cite{Hare_POSER2021} and scalability and 2) time-risk optimization~\cite{Songgupta2019} of coverage trajectory for environments with dynamic obstacles
and spatio-temporally varying currents~\cite{mittal2020rapid}.

\vspace{0pt}
\bibliographystyle{ieeetr}
\bibliography{Cstar_arxiv}

\begin{figure*}[t]
    \centering
    \vspace{-10pt}
    \subfloat[\textbf{Scenario 3 (Forest):} Performance comparison of C$^*$ with the baseline algorithms.]{
    \includegraphics[width=1.0\textwidth]{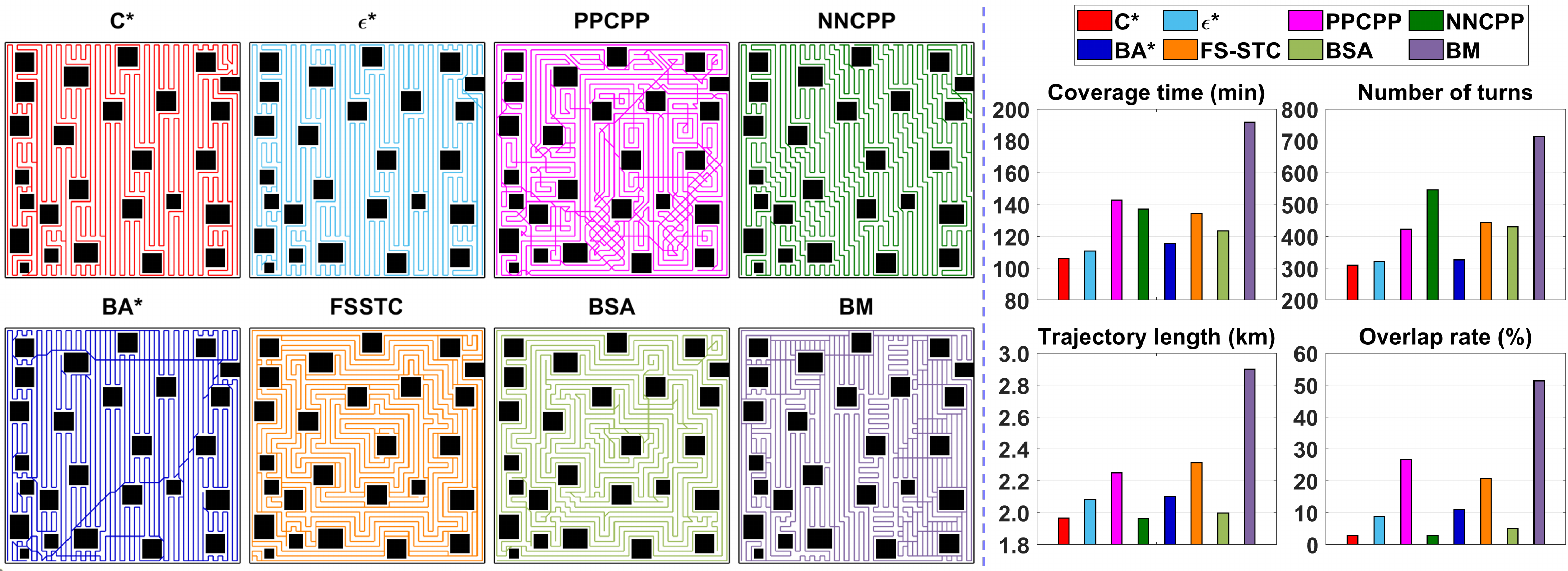}\label{fig:comparison_add_scene1}}\vspace{12pt}\\
    \centering
    \subfloat[\textbf{Scenario 4 (Office):} Performance comparison of C$^*$ with the baseline algorithms.]{
    \includegraphics[width=1.0\textwidth]{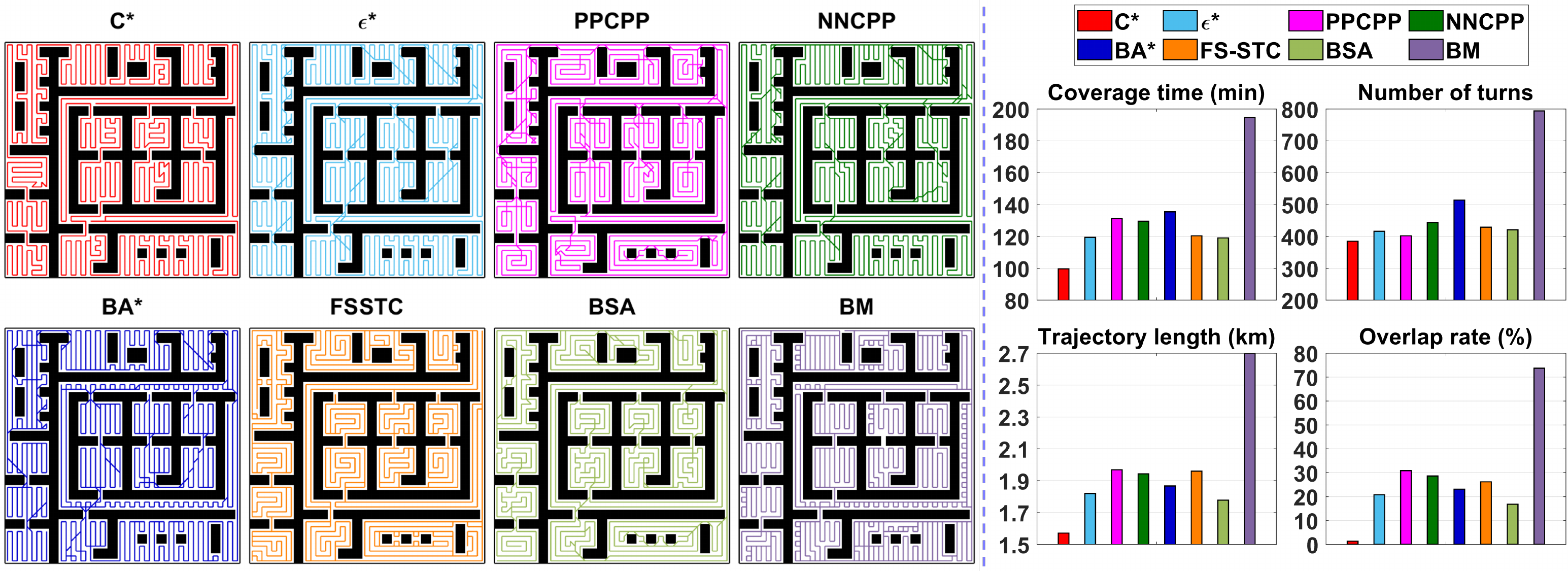}\label{fig:comparison_add_scene2}}\vspace{12pt}\\
    \centering
    \subfloat[\textbf{Scenario 5 (Warehouse):} Performance comparison of C$^*$ with the baseline algorithms.]{
    \includegraphics[width=1.0\textwidth]{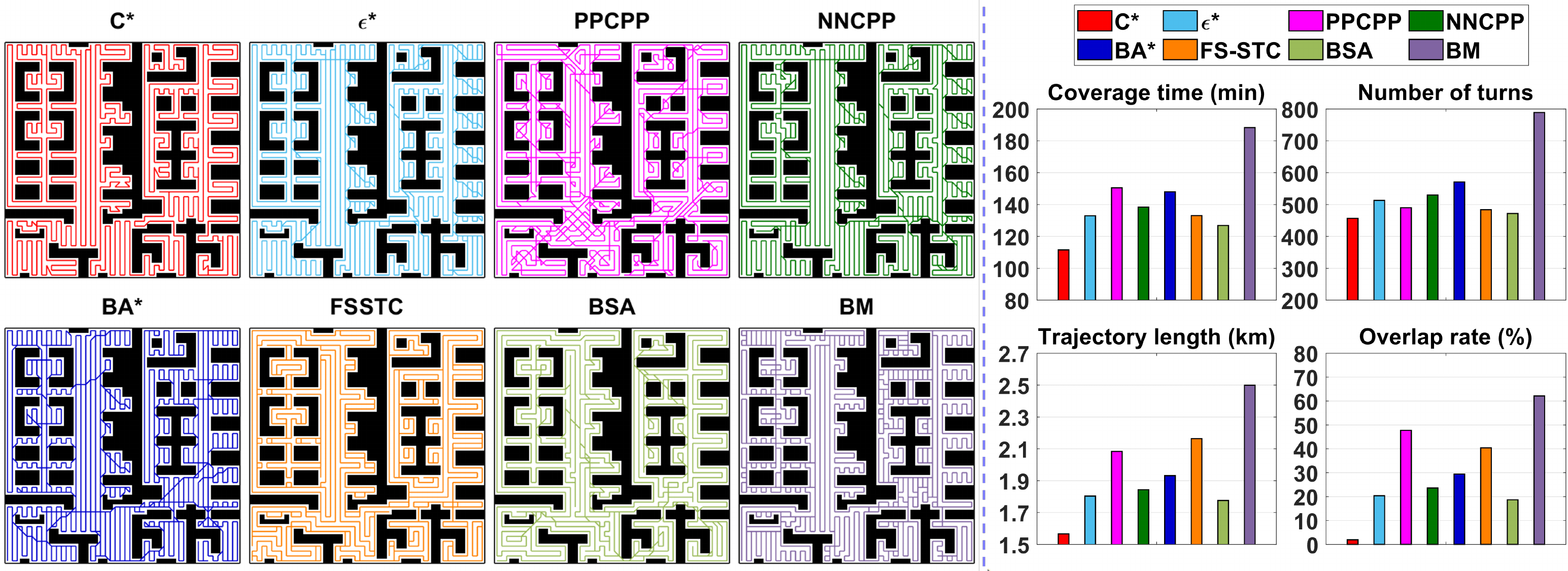}\label{fig:comparison_add_scene3}}\vspace{6pt}
      \vspace{-0pt}
\end{figure*}

\vspace{0pt}
\begin{appendix}
\subsection{Additional Results}\label{appendixA}
This section presents the detailed results for scenarios 3 (forest), 4 (office), 5 (warehouse), and 6 (maze) of Fig.~\ref{fig:scenario}. Fig.~\ref{fig:compare_addscenarios} shows the comparative evaluation results of the above scenarios, including the trajectory plots and performance metrics. 
The results demonstrate that C$^*$ consistently delivers superior performance in terms of  coverage time, number of turns, trajectory length, and overlap rate in all these scenarios.

\begin{figure*}[!t]
\ContinuedFloat
    \centering
    \vspace{-40pt}
    \subfloat[\textbf{Scenario 6 (Maze):} Performance comparison of C$^*$ with the baseline algorithms.]{
    \includegraphics[width=1\textwidth]{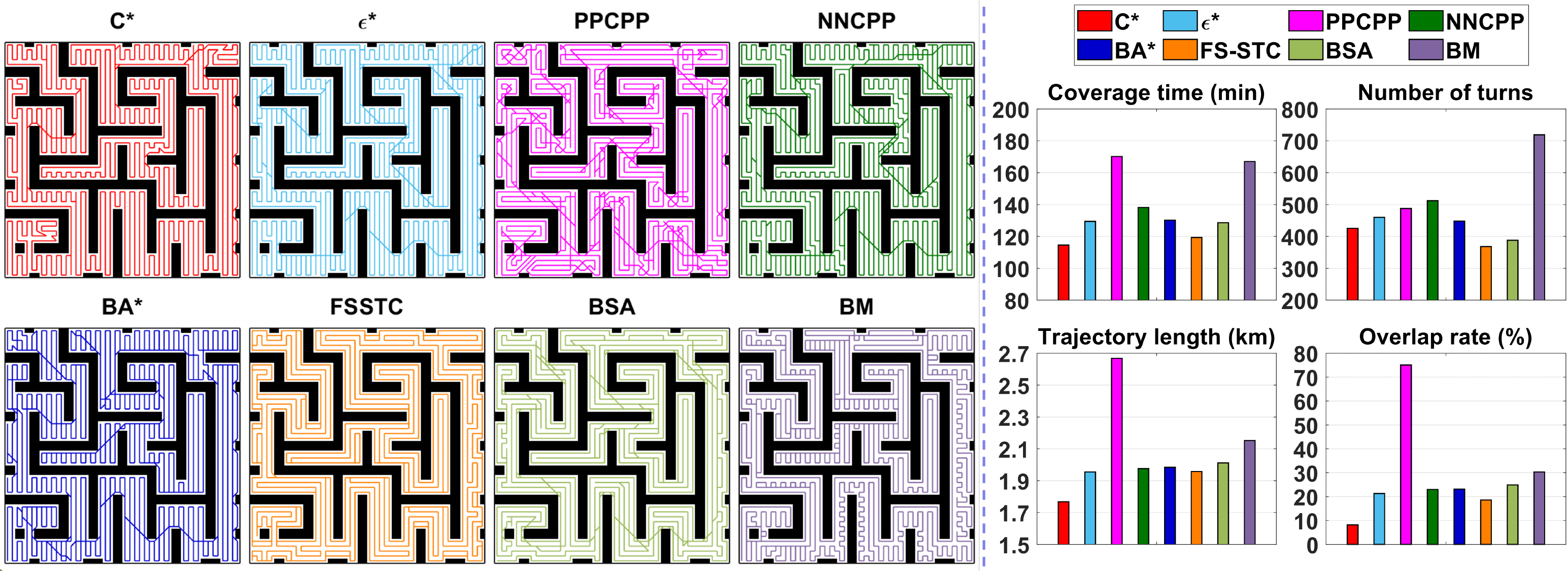}\label{fig:comparison_add_scene4}}\vspace{6pt}
      \caption{Performance comparison of C$^*$ with the baseline algorithms.}\label{fig:compare_addscenarios}\label{fig:comparison_add_scene4}
      \vspace{-0pt}
\end{figure*}

\subsection{Scalability Analysis}\label{appendixB}
This section presents the scalability analysis with respect to the map size. If the map size is increased while preserving its structure (i.e., obstacle distribution), then C$^*$ generates the same trajectory pattern with the same overlap rate. Thus, the trajectory length increases in proportion to the map size. As such, the coverage time increases linearly. To validate this, we conducted a simulation study of varying map sizes. Specifically, we generated three maps with identical obstacle layouts but different dimensions: $25 \,\mathrm{m} \times 25 \,\mathrm{m}$, $50 \,\mathrm{m} \times 50 \,\mathrm{m}$, and $100 \,\mathrm{m} \times 100 \,\mathrm{m}$, as shown in Fig.~\ref{fig:scale_scenario_path}. The C$^*$ trajectories for these three maps are shown in Fig.~\ref{fig:scale_scenario_path}, while the performance metrics are shown in Fig.~\ref{fig:performance_scale_scenario}. As expected, the coverage time and trajectory length increase linearly with the map size.

\begin{figure*}[!t]
    \centering
    \vspace{-100pt}
    \subfloat[Coverage paths generated by our algorithm in three maps with identical obstacle layouts but different physical dimensions. All maps are discretized into a $50 \times 50$ grid for coverage.]{
    \includegraphics[width=0.90\textwidth]{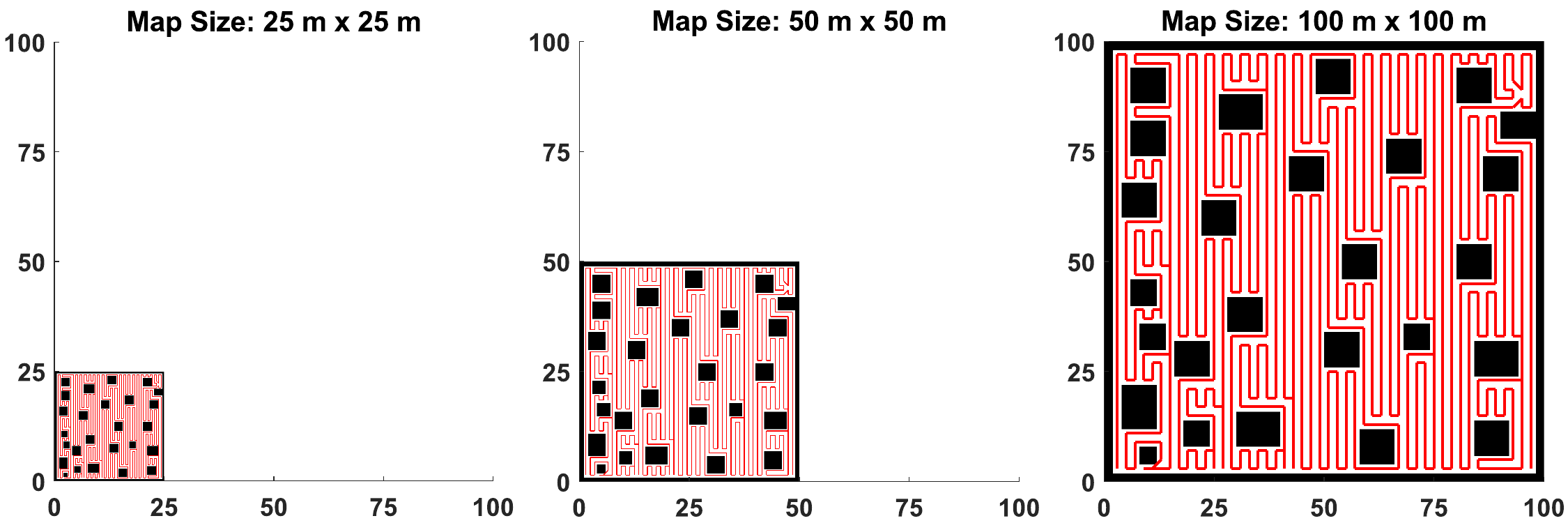}\label{fig:scale_scenario_path}}\vspace{12pt}\quad
    \centering
    \subfloat[Quantitative results showing coverage time, number of turns, trajectory length, and overlap rate.]{
    \includegraphics[width=0.90\textwidth]{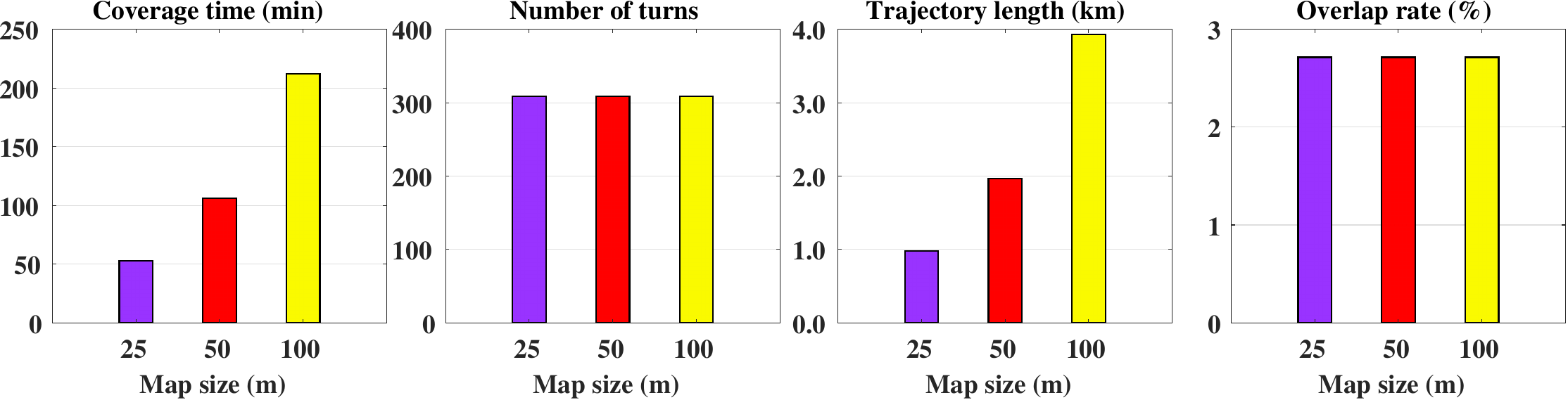}\label{fig:performance_scale_scenario}}\vspace{6pt}\\
    \caption{Scalability analysis of C$^*$ for different map sizes: (a) coverage paths in maps of $25 \times 25 \,\mathrm{m}$, $50 \times 50 \,\mathrm{m}$, and $100 \times 100 \,\mathrm{m}$ with identical obstacle layouts; (b) performance comparison in terms of coverage time, number of turns, trajectory length, and overlap rate.}
    \label{fig:scale_scenario_result}
\end{figure*}

\end{appendix}

\end{document}